\documentclass{article}

\usepackage[nonatbib, final]{neurips_2021}

\usepackage[utf8]{inputenc} % allow utf-8 input
\usepackage[T1]{fontenc}    % use 8-bit T1 fonts
\usepackage[hidelinks]{hyperref}       % hyperlinks
\usepackage{url}            % simple URL typesetting
\usepackage{booktabs}       % professional-quality tables
\usepackage{amsfonts}       % blackboard math symbols
\usepackage{nicefrac}       % compact symbols for 1/2, etc.
\usepackage{microtype}      % microtypography
\usepackage{xcolor}         % colors
\usepackage{amssymb}
\usepackage{amsmath}
\usepackage{amsthm}
\usepackage{graphicx}
\usepackage{caption}
\usepackage{subcaption}
\usepackage{float}
\usepackage{svg}
\usepackage{wrapfig}
\usepackage[ruled,vlined]{algorithm2e}
\usepackage{lscape}
\usepackage{rotating}
\usepackage{sidecap}
\usepackage[toc,page,header]{appendix}
\usepackage{minitoc}
\usepackage{bbold}
\usepackage[numbers]{natbib}
\usepackage{comment}
\usepackage{multirow}
\usepackage{multicol}
\usepackage{caption}

\newtheorem{condition}{Condition}
\newtheorem{definition}{Definition}
\newtheorem{theorem}{Theorem}

\newtheorem{lemma}[theorem]{Lemma}
\newtheorem{proposition}[theorem]{Proposition}
\newcommand{\ve}[1]{\mathbf{#1}}
\newcommand{\ves}[1]{\boldsymbol{#1}}
\newcommand{\Lagr}{\mathcal{L}}
\newcommand{\dd}{\text{d}}
\newcommand{\ddt}{\frac{\text{d}}{\text{d}t}}

\DeclareMathOperator*{\argmin}{arg\,min}

\newcommand{\myequation}{\begin{equation}}
\newcommand{\myendequation}{\end{equation}}
\newcommand{\appropto}{\mathrel{\vcenter{
  \offinterlineskip\halign{\hfil$##$\cr
    \propto\cr\noalign{\kern2pt}\sim\cr\noalign{\kern-2pt}}}}}
\let\[\myequation
\let\]\myendequation

\newcommand{\J}{J}
\newcommand{\Q}{Q}
\newcommand{\rank}{\text{rank}}
\newcommand{\vfb}{\ve{v}^{\mathrm{fb}}}
\newcommand{\vff}{\ve{v}^{\mathrm{ff}}}
\newcommand{\taufb}{\tau_{v^{\mathrm{fb}}}}
\newcommand{\sss}{\mathrm{ss}}
\newcommand{\tdu}{\Delta\ve{u}}

\newcommand{\vdl}{\ves{\delta}_L}
\newcommand{\vd}{\ves{\delta}}

\newcommand{\vv}{\ve{v}}
\newcommand{\vr}{\ve{r}}
\newcommand{\vt}{\ves{\theta}}
\newcommand{\uint}{\ve{u}^{\mathrm{int}}}

\usepackage{todonotes} % for comments

\title{Credit Assignment in Neural Networks through \\
Deep Feedback Control}

\author{%
  \textbf{Alexander Meulemans\thanks{Equal contribution}, Matilde Tristany Farinha$^*$, Javier García Ordóñez,}\\
  \textbf{Pau Vilimelis Aceituno, João Sacramento, Benjamin F. Grewe}\\
  \\
  Institute of Neuroinformatics, University of Zürich and ETH Zürich\\
  \texttt{ameulema@ethz.ch}  
}

\begin{document}

\doparttoc % Tell to minitoc to generate a toc for the parts
\faketableofcontents % Run a fake tableofcontents command for the partocs

\maketitle

\begin{abstract}
\vspace{-0.3cm}
The success of deep learning sparked interest in whether the brain learns by using similar techniques for assigning credit to each synaptic weight for its contribution to the network output.
However, the majority of current attempts at biologically-plausible learning methods are either non-local in time, require highly specific connectivity motifs, or have no clear link to any known mathematical optimization method. Here, we introduce Deep Feedback Control (DFC), a new learning method that uses a feedback controller to drive a deep neural network to match a desired output target and whose control signal can be used for credit assignment.
The resulting learning rule is fully local in space and time and approximates Gauss-Newton optimization for a wide range of feedback connectivity patterns. To further underline its biological plausibility, we relate DFC to a multi-compartment model of cortical pyramidal neurons with a local voltage-dependent synaptic plasticity rule, consistent with recent theories of dendritic processing. By combining dynamical system theory with mathematical optimization theory, we provide a strong theoretical foundation for DFC that we corroborate with detailed results on toy experiments and standard computer-vision benchmarks.
\end{abstract}

\section{Introduction}
\vspace{-0.3cm}
The error backpropagation (BP) algorithm \citep{rumelhart1986learning, werbos1982applications, linnainmaa1970representation} is currently the gold standard to perform credit assignment (CA)
%--i.e., determine how each synaptic weight influences the network output \citep{hinton1984boltzmann}--
in deep neural networks. Although deep learning was inspired by biological neural networks, an exact mapping of BP onto biology to explain learning in the brain leads to several inconsistencies with experimental results that are not yet fully addressed \citep{crick1989recent, grossberg1987competitive, lillicrap2020backpropagation}. First, BP requires an exact symmetry between the weights of the forward and feedback pathways \citep{grossberg1987competitive, lillicrap2020backpropagation}, also called the weight transport problem. Another issue of relevance is that, in biological networks, feedback also changes each neuron’s activation and thus its immediate output \citep{larkum2009synaptic, gilbert2013top}, which does not occur in BP.

\citet{lillicrap2016random} convincingly showed that the weight transport problem can be sidestepped in modest supervised learning problems by using random feedback connections. However, follow-up studies indicated that random feedback paths cannot provide precise CA in more complex problems \citep{bartunov2018assessing, launay2019principled, moskovitz2018feedback, crafton2019direct}, which can be mitigated by learning feedback weights that align with the forward pathway \citep{akrout2019deep, kunin2020two, lansdell2019learning, Guerguiev2020Spike-based,golkar2020biologically} or approximate its inverse \citep{bengio2014auto, lee2015difference, meulemans2020theoretical, bengio2020deriving}. However, this precise alignment imposes strict constraints on the feedback weights, whereas more flexible constraints could provide the freedom to use feedback also for other purposes besides learning, such as attention and prediction \citep{gilbert2013top}.

A complementary line of research proposes models of cortical microcircuits which propagate CA signals through the network using dynamic feedback \citep{sacramento2018dendritic, whittington2017approximation, guerguiev2017towards} or multiplexed neural codes \citep{payeur2020burst}, thereby directly influencing neural activations with feedback. However, these models introduce highly specific connectivity motifs and tightly coordinated plasticity mechanisms. Whether these constraints can be fulfilled by cortical networks is an interesting experimental question. Another line of work uses adaptive control theory \citep{slotine1991applied} to derive learning rules for non-hierarchical recurrent neural networks (RNNs) based on error feedback, which drives neural activity to track a reference output \citep{gilra2017predicting, deneve2017brain, alemi2017learning, bourdoukan2015enforcing}. These methods have so far only been used to train single-layer RNNs with fixed output and feedback weights, making it unclear whether they can be extended to deep neural networks.
Finally, two recent studies \citep{podlaski2020biological, kohan2018error} use error feedback in a dynamical setting to invert the forward pathway, thereby enabling errors to flow backward.
% Finally, \citet{podlaski2020biological} use feedback controllers to dynamically invert the feedforward pathway to propagate errors back in a deep neural network. 
These approaches rely on a learning rule that is non-local in time and it remains unclear whether they approximate any known optimization method.
%\cite{bengio2020deriving} takes a first step in this direction by relating under certain conditions such iterative inversion methods to Gauss-Newton optimization.
Addressing the latter, two recent studies take a first step by relating learned (non-dynamical) inverses of the forward pathway \citep{meulemans2020theoretical} and iterative inverses restricted to invertible networks \citep{bengio2020deriving} to approximate Gauss-Newton optimization.
% Most related to DFC is the dynamical inversion method of \citet{podlaski2020biological}, which dynamically inverts the output error using a separate controller for each hidden layer. The resulting weight update rule is non-local in time and it has not yet been related to any known optimization method. By contrast, DFC implements GN optimization using local learning rules when Conditions \ref{con:Q_GN} and \ref{con:local_stability} are met. Interestingly, recent work on Target Propagation (TP) showed that it approximates a mixture of GN optimization and gradient descent \citep{meulemans2020theoretical, bengio2020deriving}, suggesting a connection between DFC and TP.

% Third paragraph: Introduce the work of Sophie Deneve and Wulfram Gerstner (the spiking papers for learning through control). Focus on error feedback that is used for learning while also changing neuron activations. However, only can handle shallow reservoir networks + no strong link to optimization theory. Also introduce dynamical inversion (Podlaski) that uses feedback control to dynamically invert output errors in deep neural networks. However, non-local in time and no clear link to optimization. 

Inspired by the Dynamic Inversion method \citep{podlaski2020biological}, we introduce Deep Feedback Control (DFC), a new biologically-plausible CA method that addresses the above-mentioned limitations and extends the control theory approach to learning \citep{gilra2017predicting, deneve2017brain, alemi2017learning, bourdoukan2015enforcing} to deep neural networks.
% Inspired by control theory we introduce Deep Feedback Control (DFC) to address the above-mentioned limitations. 
DFC uses a feedback controller that drives a deep neural network to match a desired output target. For learning, DFC then simply uses the dynamic change in the neuron activations to update their synaptic weights, resulting in a learning rule fully local in space and time. We show that DFC approximates Gauss-Newton (GN) optimization and therefore provides a fundamentally different approach to CA compared to BP. Furthermore, DFC does not require precise alignment between forward and feedback weights, nor does it rely on highly specific connectivity motifs. 
Interestingly, the neuron model used by DFC can be closely connected to recent multi-compartment models of cortical pyramidal neurons. Finally, we provide detailed experimental results, corroborating our theoretical contributions and showing that DFC does principled CA on standard computer-vision benchmarks in a way that fundamentally differs from standard BP.

% Here, we introduce Deep Feedback Control (DFC), which embodies a necessary first step in extending the control theory approach for learning to deep neural networks while simultaneously having an established link to mathematical optimization theory. DFC uses a feedback controller to dynamically drive a deep neural network to match the desired output target, after which the changed neuron activations are used to update the synaptic weights. More specifically, we show that (i) DFC does principled credit assignment in deep neural networks by approximating Gauss-Newton optimization, (ii) the resulting synaptic learning rule is fully local in space and time, (iii) DFC is compatible with a flexible range of feedback connectivity and network architectures and (iv) we can learn the feedback weights with simple anti-Hebbian updates such that they remain compatible with the changing forward pathway during training. (v) Finally, we provide detailed experimental results, corroborating our theoretical contributions and showing that DFC does principled credit assignment on standard computer-vision benchmarks.

% Fourth paragraph: Introduce DFC which builds further upon the previous discussed works, and provides important new ingredients: (i) error feedback that changes activations and is used for learning accross hierarchies, (ii) learning rule fully local in time and space, (iii) Strong link with Gauss-Newton optimization, (iv) works for flexible range of feedback connectivity (v) fb learning rule that tailors the feedback control to a changing network

\section{The Deep Feedback Control method}
\vspace{-0.3cm}
Here, we introduce the core parts of DFC. In contrast to conventional feedforward neural network models, DFC makes use of a dynamical neuron model (Section \ref{section:dynamical_neuron_model}). We use a feedback controller to drive the neurons of the network to match a desired output target (Section \ref{section:feedback_controller}), while simultaneously updating the synaptic weights using the change in neuronal activities (Section \ref{section:plasticity_rules}). This combination of dynamical neurons and controller leads to a simple but powerful learning method, that is linked to GN optimization and offers a flexible range of feedback connectivity (see Section \ref{sec:learning_TPDI}).

\subsection{Neuron and network dynamics}\label{section:dynamical_neuron_model}
\vspace{-0.2cm}
The first main component of DFC is a dynamical multilayer network, in which every neuron integrates its forward and feedback inputs according to the following dynamics:
\begin{align}\label{eq:network_dynamics}
    \tau_v \ddt \ve{v}_i(t) &= -\ve{v}_i(t) + W_i\phi\big(\ve{v}_{i-1}(t)\big) + Q_i\ve{u}(t) \quad 1\leq i \leq L,
\end{align}
with $\ve{v}_i$ a vector containing the pre-nonlinearity activations of the neurons in layer $i$, $W_i$ the forward weight matrix, $\phi$ a smooth nonlinearity, $\ve{u}$ a feedback input, $Q_i$ the feedback weight matrix, and $\tau_v$ a time constant. See Fig. \ref{fig:DFC_schematics}B for a schematic representation of the network. To simplify notation, we define $\ve{r}_i=\phi(\ve{v}_i)$ as the post-nonlinearity activations of layer $i$. The input $\ve{r}_0$ remains fixed throughout the dynamics \eqref{eq:network_dynamics}. Note that in the absence of feedback, i.e., $\ve{u}=0$, the equilibrium state of the network dynamics \eqref{eq:network_dynamics} corresponds to a conventional multilayer feedforward network state, which we denote with superscript `-':
\begin{align} \label{eq:feedforward_equilibrium}
    \ve{r}_i^- = \phi(\ve{v}_i^-) = \phi(W_i \ve{r}_{i-1}^-), \quad 1\leq i \leq L, \quad \text{with } \ve{r}_0^- = \ve{r}_0.
\end{align}

\subsection{Feedback controller} \label{section:feedback_controller}
\vspace{-0.2cm}
The second core component of DFC is a feedback controller, which is only active during learning. Instead of a single backward pass for providing feedback, DFC uses a feedback controller to continuously drive the network to an output target $\ve{r}^*_L$ (see Fig. \ref{fig:DFC_schematics}D). Following the Target Propagation framework \citep{lee2015difference, meulemans2020theoretical, bengio2020deriving}, we define $\ve{r}^*_L$ as the feedforward output nudged towards lower loss:

\begin{align}\label{eq:output_target}
    \ve{r}^*_L \triangleq  \ve{r}^-_L - \lambda \frac{\partial \mathcal{L}(\ve{r}_L,\ve{y})}{\partial \ve{r}_L}\Big\rvert_{\ve{r}_L^{ }=\ve{r}_L^{-}} = \ve{r}_L^- + \ves{\delta}_L,
\end{align}

with $\mathcal{L}(\ve{r}_L,\ve{y})$ a supervised loss function defining the task, $\ve{y}$ the label of the training sample, $\lambda$ a stepsize, and $\ves{\delta}_L$ shorthand notation.
Note that \eqref{eq:output_target} only needs the easily obtained loss gradient w.r.t. the output, e.g., for an $L^2$ output loss, one obtains the convex combination $\ve{r}_L^* = (1-2\lambda)\ve{r}_L^- + 2\lambda\ve{y}$.

The feedback controller produces a feedback signal $\ve{u}(t)$ to drive the network output $\ve{r}_L(t)$ towards its target $\ve{r}_L^*$, using the control error $\ve{e}(t) \triangleq \ve{r}^*_L - \ve{r}_L(t) $. A standard approach in designing a feedback controller is the Proportional-Integral-Derivative (PID) framework \citep{franklin2015feedback}. While DFC is compatible with various controller types, such as a full PID controller or a pure proportional controller (see Appendix \ref{app:controller_type}), we use a PI controller for a combination of simplicity and good performance, resulting in the following controller dynamics (see also Fig. \ref{fig:DFC_schematics}A):
\begin{align} \label{eq:controller_dynamics}
    \ve{u}(t) = K_I \ve{u}^{\text{int}}(t) + K_P \ve{e}(t), \quad 
    \tau_u \ddt \ve{u}^{\text{int}}(t) = \ve{e}(t) - \alpha \ve{u}^{\text{int}}(t),
\end{align}
where a leakage term is added to constrain the magnitude of $\ve{u}^{\text{int}}$. For mathematical simplicity, we take the control matrices equal to $K_I = I$ and $K_P = k_p I$ with $k_p\geq0$ the proportional control constant.
This PI controller adds a leaky integration of the error $\ve{u}^{\text{int}}$ to a scaled version of the error $k_p \ve{e}$ which could be implemented by a dedicated neural microcircuit (for a discussion see App. \ref{app:controller_microcircuit}). Drawing inspiration from the Target Propagation framework \citep{bengio2014auto, lee2015difference, meulemans2020theoretical, bengio2020deriving} and the Dynamic Inversion framework \citep{podlaski2020biological}, one can think of the controller and network dynamics as performing a \textit{dynamic inversion} of the output target $\ve{r}_L^*$ towards the hidden layers, as the controller dynamically changes the activation of the hidden layers until the output target is reached. 

\begin{figure}[h!]
\centering
\includegraphics[width=0.95\linewidth]{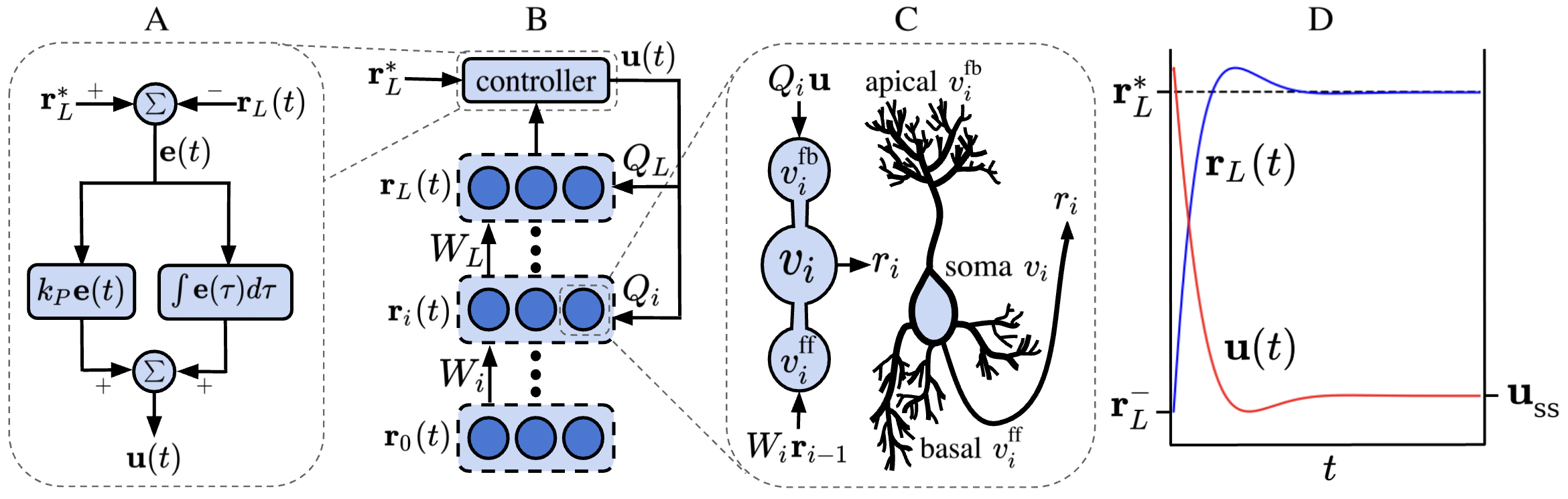}
\caption{(A) A block diagram of the controller, where we omitted the leakage term of the integral controller. (B) Schematic illustration of DFC. (C) Schematic illustration of the multi-compartment neuron used by DFC, compared to a cortical pyramidal neuron sketch (see also Discussion). (D) Illustration of the output $\ve{r}_L(t)$ and the controller dynamics $\ve{u}(t)$ in DFC.}
\label{fig:DFC_schematics}
\end{figure}

\subsection{Forward weight updates}\label{section:plasticity_rules}
\vspace{-0.2cm}
The update rule for the feedforward weights has the form:
% Due to the feedback control of the neural activations, the plasticity rule for the forward weights $W_i$ can be of a simple form where the neuron compares its controlled activation to its feedforward input. This rule is completely local in time and space. For biological realism, we assume that the weights are continuously plastic, resulting in 
\begin{align}\label{eq:W_dynamics}
    \tau_W \ddt W_i (t) = \big(\phi(\ve{v}_i(t)) - \phi(W_i\ve{r}_{i-1}(t))\big)\ve{r}_{i-1}(t)^T.
\end{align}
This learning rule simply compares the neuron's controlled activation to its current feedforward input and is thus local in space and time. Furthermore, it can be interpreted most naturally by compartmentalizing the neuron into the central compartment $\ve{v}_i$ from \eqref{eq:network_dynamics} and a feedforward compartment $\ve{v}_i^{\text{ff}} \triangleq W_i\ve{r}_{i-1}$ that integrates the feedforward input. Now, the forward weight dynamics \eqref{eq:W_dynamics} represents a delta rule using the difference between the actual firing rate of the neuron, $\phi(\ve{v}_i)$, and its estimated firing rate, $\phi(\ve{v}_i^{\text{ff}})$, based on the feedforward inputs. Note that we assume $\tau_W$ to be a large time constant, such that the network \eqref{eq:network_dynamics} and controller dynamics \eqref{eq:controller_dynamics} are not influenced by the weight dynamics, i.e., the weights are considered fixed in the timescale of the controller and network dynamics.

In Section \ref{sec:fb_learning}, we show how the feedback weights $Q_i$ can also be learned locally in time and space for supporting the stability of the network dynamics and the learning of $W_i$. This feedback learning rule needs a feedback compartment $\ve{v}_i^{\text{fb}} \triangleq Q_i\ve{u}$, leading to the \textit{three-compartment neuron} schematized in Fig. \ref{fig:DFC_schematics}C, inspired by recent multi-compartment models of the pyramidal neuron (see Discussion).
Now, that we introduced the DFC model, we will show that (i) the weight updates \eqref{eq:W_dynamics} can properly optimize a loss function (Section \ref{sec:learning_TPDI}), (ii) the resulting dynamical system is stable under certain conditions (Section \ref{sec:stability_TPDI}), and (iii) learning the feedback weights facilitates (i) and (ii) (Section \ref{sec:fb_learning}).

\section{Learning theory} \label{sec:learning_TPDI}
\vspace{-0.3cm}
To understand how DFC optimizes the feedforward mapping \eqref{eq:feedforward_equilibrium} on a given loss function, we link the weight updates \eqref{eq:W_dynamics} to mathematical optimization theory. We start by showing that DFC dynamically inverts the output error to the hidden layers (Section \ref{sec:dynamical_inversion}), which we link to GN optimization under flexible constraints on the feedback weights $Q_i$ and on layer activations (Section \ref{sec:gn_optimization}). In Section \ref{sec:minimum_norm_updates}, we relax some of these constraints, and show that DFC still does principled optimization by using \textit{minimum norm} (MN) updates for $W_i$. During this learning theory section, we assume stable dynamics, which we investigate in more detail in Section \ref{sec:stability_TPDI}. All theoretical results of this section are tailored towards a PI controller, and they can be easily extended to pure proportional or integral control (see App. \ref{app:controller_type}).
%We investigate the weight plasticity rule \eqref{eq:W_dynamics} when the network and controller are at steady state, as this provides a clear and intuitive picture of how DFC is able to optimize loss functions. We later comment on how the steady-state results relate to the continuous plasticity described in eq. \eqref{eq:W_dynamics}.

\subsection{DFC dynamically inverts the output error}\label{sec:dynamical_inversion}
\vspace{-0.2cm}
To understand how the weight update \eqref{eq:W_dynamics} can access error information, we start by investigating the steady state of the network dynamics \eqref{eq:network_dynamics} and the controller dynamics \eqref{eq:controller_dynamics}, assuming that all weights are fixed (hence, a separation of timescales). As the feedback controller controls all layers simultaneously, we introduce a compact notation for: concatenated neuron activations $\ve{v} \triangleq [\ve{v}^{T}_1, ..., \ve{v}^{T}_L]^T$, feedforward compartments $\ve{v}^{\text{ff}} \triangleq [\ve{v}^{\text{ff},T}_1, ..., \ve{v}^{\text{ff},T}_L]^T$, and feedback weights $Q \triangleq [Q_1^T ... Q_L^T]^T$. Lemma \ref{lemma:steady_state_solution_system} shows a first-order Taylor approximation of the steady-state solution (full proof in App. \ref{app:linearized_dyn}). 
\begin{lemma}\label{lemma:steady_state_solution_system}
    Assuming stable dynamics, a small target stepsize $\lambda$, and $W_i$ and $Q_i$ fixed, the steady-state solutions of the dynamical systems \eqref{eq:network_dynamics} and \eqref{eq:controller_dynamics} can be approximated by:
    \begin{align}\label{eq:l1_dynamical_inversion}
        \ve{u}_{\mathrm{ss}} = (JQ + \tilde{\alpha} I)^{-1}\ves{\delta}_L + \mathcal{O}(\lambda^2), \quad
        \ve{v}_{\mathrm{ss}} = \ve{v}^{\mathrm{ff}}_{\mathrm{ss}} +  Q (JQ + \tilde{\alpha} I)^{-1}\ves{\delta}_L + \mathcal{O}(\lambda^2),
    \end{align}
    with $J \triangleq \left.\left[\frac{\partial \ve{r}^-_L}{\partial \ve{v}_1},...,\frac{\partial \ve{r}^-_L}{\partial \ve{v}_L}\right]\right\rvert_{\ve{v}=\ve{v}^-}$ the Jacobian of the network output w.r.t. all $\ve{v}_i$, evaluated at the network equilibrium without feedback, $\ves{\delta}_L$ the output error as defined in \eqref{eq:output_target}, $\ve{v}^{\mathrm{ff}}_{i,ss} = W_i \phi(\ve{v}_{i-1,\mathrm{ss}})$, and $\tilde{\alpha} = \alpha/(1+\alpha k_p)$.
\end{lemma}
To get a better intuition of what this steady state represents, consider the scenario where we want to nudge the network activation $\ve{v}$ with $\Delta \ve{v}$, i.e., $\ve{v}_{\mathrm{ss}} = \ve{v}^{\text{ff}}_{\mathrm{ss}} + \Delta \ve{v}$, such that the steady-state network output equals its target $\ve{r}_L^*$. With linearized network dynamics, this results in solving the linear system $J \Delta \ve{v} = \ves{\delta}_L$. As $\Delta \ve{v}$ is of much higher dimension than $\ves{\delta}_L$, this is an underdetermined system with infinitely many solutions. Constraining the solution to the column space of $Q$ leads to the unique solution $\Delta \ve{v} = Q (JQ)^{-1}\ves{\delta}_L$, corresponding to the steady-state solution in Lemma \ref{lemma:steady_state_solution_system} minus a small damping constant $\tilde{\alpha}$.
Hence,
similar to \citet{podlaski2020biological}, through an interplay between the network and controller dynamics,
the controller \textit{dynamically inverts} the output error $\ves{\delta}_L$ to produce feedback that exactly drives the network output to its desired target.
%Next, we show that under flexible conditions on the column space defined by the feedback weights $Q$, this dynamical inversion results in GN optimization.

\subsection{DFC approximates Gauss-Newton optimization}\label{sec:gn_optimization}
\vspace{-0.2cm}
To understand the optimization characteristics of DFC, we show that under flexible conditions on $Q_i$ and the layer activations, DFC approximates GN optimization. We first briefly review GN optimization and introduce two conditions needed for the main theorem. 

\textbf{Gauss-Newton optimization} \citep{gauss1809theoria} is an approximate second-order optimization method used in nonlinear least-squares regression. The GN update for the model parameters $\ves{\theta}$ is computed as:
\begin{align} \label{eq:damped_gn_iteration}
    \Delta \ves{\theta} &= J^{\dagger}_{\theta}\ve{e}_L,
\end{align}
with $J_{\theta}$ the Jacobian of the model output w.r.t. $\ves{\theta}$ concatenated for all minibatch samples, $J^{\dagger}_{\theta}$ its Moore-Penrose pseudoinverse, and $\ve{e}_L$ the output errors.
%and $\lambda$ a Tikhonov damping constant \citep{tikhonov1943stability, levenberg1944method}. When $\lambda \rightarrow 0$, we get $\Delta \ves{\theta} = -J^{\dagger} \ve{e}_L$, with  of $J$.

\begin{condition}\label{con:magnitude_r}
Each layer of the network, except from the output layer, has the same activation norm:
\begin{align}
    \|\ve{r}_0\|_2 = \|\ve{r}_1\|_2 = ... \|\ve{r}_{L-1}\|_2 \triangleq \|\ve{r}\|_2.
\end{align}
\end{condition}
Note that the latter condition considers a statistic $\|\ve{r}_i\|_2$ of a whole layer and does not impose specific constraints on single neural firing rates. This condition can be interpreted as each layer, except the output layer, having the same `energy budget' for firing. 

\begin{condition}\label{con:Q_GN}
The column space of $Q$ is equal to the row space of $J$.
\end{condition}
This more abstract condition imposes a flexible constraint on the feedback weights $Q_i$, that generalizes common learning rules with direct feedback connections \citep{lansdell2019learning, meulemans2020theoretical}. For instance, besides $Q = J^T$ (BP; \citep{lansdell2019learning}) and $Q = J^{\dagger}$ \citep{meulemans2020theoretical}, many other instances of $Q$ which have not yet been explored in the literature fulfill Condition \ref{con:Q_GN} (see Fig. \ref{fig:pca_feedbackweights}), hence leading to principled optimization (see Theorem \ref{theorem:GN_TPDI}). With these conditions in place, we are ready to state the main theorem of this section (full proof in App. \ref{app:learning_theory}).

\begin{theorem}\label{theorem:GN_TPDI}
Assuming Conditions \ref{con:magnitude_r} and \ref{con:Q_GN} hold, $\J$ is full rank, the task loss $\mathcal{L}$ is a $L^2$ loss, and $\lambda,\alpha \rightarrow 0$ 
%in computing $\ve{r}_L^*$ defined by \eqref{eq:output_target}
, then the following steady-state (ss) updates for the forward weights,
\begin{align}\label{eq:update_W_ss}
    \Delta W_{i, \mathrm{ss}} = \eta (\ve{v}_{i,\mathrm{ss}} - \ve{v}_{i,\mathrm{ss}}^{\mathrm{ff}})\ve{r}_{i-1,\mathrm{ss}}^T\,\,,
\end{align}
with $\eta$ a stepsize parameter, align with the weight updates for $W_i$ for the feedforward network \eqref{eq:feedforward_equilibrium} prescribed by the GN optimization method with a minibatch size of 1.
\end{theorem}

\begin{wrapfigure}{r}{0.40\linewidth}
  \begin{center}
    \vspace*{-0.6cm}
    \includegraphics[width=1.0\linewidth]{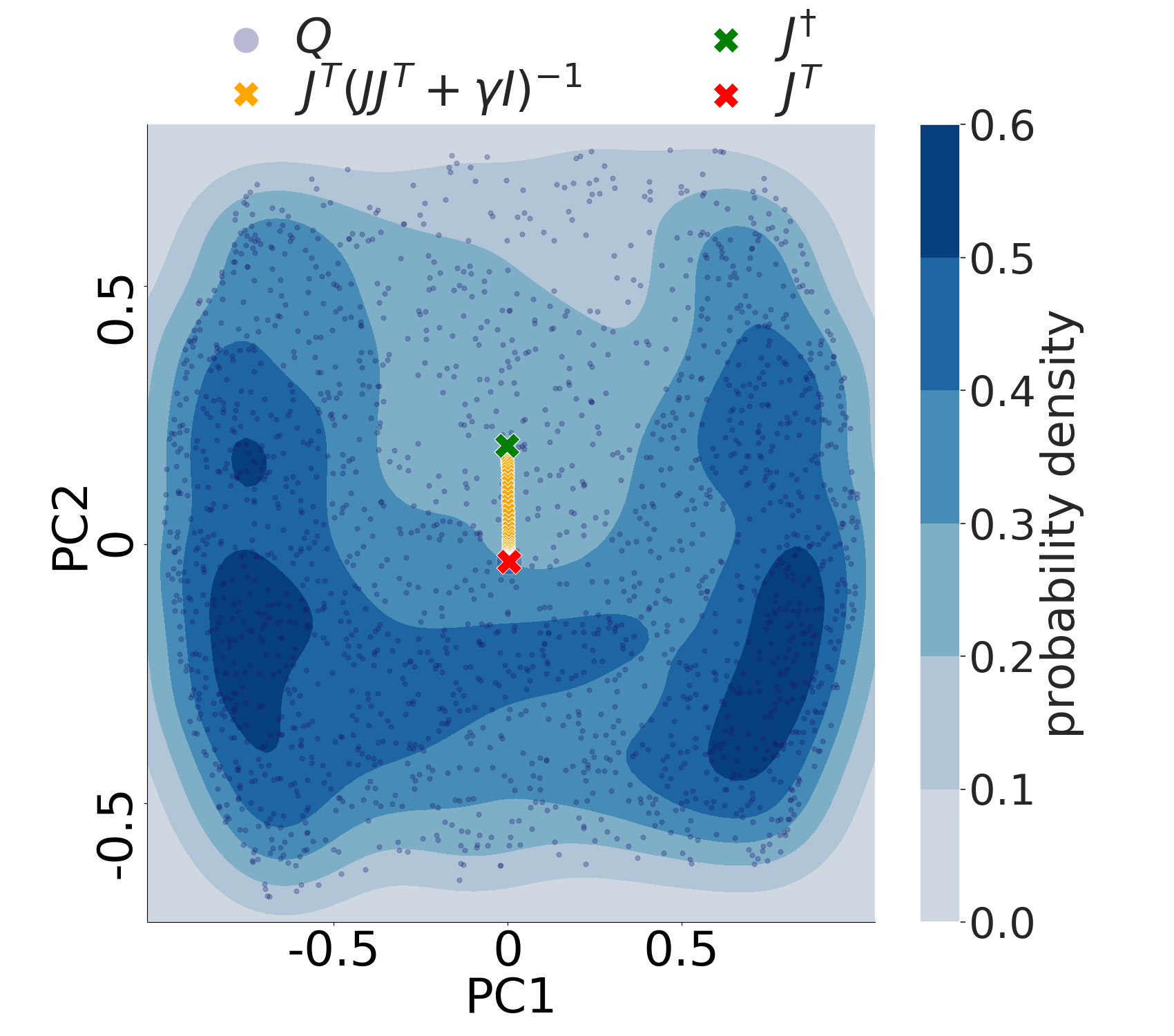}
  \end{center}
  \vspace{-0.2cm}
  \caption{Randomly generated feedback matrices $Q$ (blue) that satisfy Conditions \ref{con:Q_GN} and \ref{con:local_stability}, and have unity norm, visualized by a principal component analysis, with density contours added for visual clarity. $J^T$, $J^{\dagger}$, and $J^T(JJ^T + \gamma I)^{-1},\, \gamma \in [10^{-5}, 10^2]$, are added, highlighting that the optimal feedback configurations for DFC (blue) span a much wider space compared to conventional CA methods.}
    \label{fig:pca_feedbackweights}
  \vspace*{-0.4cm}
\end{wrapfigure}

In this theorem, we need Condition \ref{con:Q_GN} such that the dynamical inversion $Q (JQ)^{-1}$ \eqref{eq:l1_dynamical_inversion} equals the pseudoinverse of $\J$ and we need Condition \ref{con:magnitude_r} to extend this pseudoinverse to the Jacobian of the output w.r.t. the network weights, as in eq. \eqref{eq:damped_gn_iteration}.
Theorem \ref{theorem:GN_TPDI} links the DFC method to GN optimization, thereby showing that it does principled optimization, while being fundamentally different from BP. In contrast to recent work that connects target propagation to GN \citep{meulemans2020theoretical, bengio2020deriving}, we do not need to approximate the GN curvature matrix by a block-diagonal matrix but use the full curvature instead. Hence, one can use Theorem 2 in \citet{cai2019gram} to obtain convergence results for this setting of GN with a minibatch size of 1, in highly overparameterized networks. Strikingly, the feedback path of DFC does not need to align with the forward path or its inverse to provide optimally aligned weight updates with GN, as long as it satisfies the flexible Condition \ref{con:Q_GN} (see Fig. \ref{fig:pca_feedbackweights}).

The steady-state updates \eqref{eq:update_W_ss} used in Theorem \ref{theorem:GN_TPDI} differ from the actual updates \eqref{eq:W_dynamics} in two nuanced ways. First, the plasticity rule \eqref{eq:W_dynamics} uses a nonlinearity, $\phi$, of the compartment activations, whereas in Theorem \ref{theorem:GN_TPDI} this nonlinearity is not included. There are two reasons for this: (i) the use of $\phi$ in \eqref{eq:W_dynamics} can be linked to specific biophysical mechanisms in the pyramidal cell \citep{urbanczik2014learning} (see Discussion), and (ii) using $\phi$ makes sure that saturated neurons do not update their forward weights, which leads to better performance (see App. \ref{app:linear_undamped_lr}). Second, in Theorem \ref{theorem:GN_TPDI}, the weights are only updated at steady state, whereas in \eqref{eq:W_dynamics} they are continuously updated during the dynamics of the network and controller. Before settling rapidly, the dynamics oscillate around the steady-state value (see Fig. \ref{fig:DFC_schematics}D), and hence, the accumulated continuous updates \eqref{eq:W_dynamics} will be approximately equal to its steady-state equivalent, since the oscillations approximately cancel each other out and the steady state is quickly reached (see Section \ref{sec:toy_regression} and App. \ref{app:relation_ss_cont}). Theorem \ref{theorem:GN_TPDI} needs a $L^2$ loss function and Condition \ref{con:magnitude_r} and \ref{con:Q_GN} to hold for linking DFC with GN. In the following subsection, we relax these assumptions and show that DFC still does principled optimization.

% \begin{figure}[H]
%     \centering
%     \includegraphics[width=0.5\linewidth]{figures/FIgure_2.png}
%     \caption{Caption!}
%     \label{fig:1B}
% \end{figure}

% Note that Theorem \ref{theorem:GN_TPDI} considers a discrete update $\Delta W_i$ instead of a continuous update as defined in \eqref{eq:W_dynamics}. In what follows we show that $\Delta W_i$ is a good approximation of the continuous weight updates if the network settles quickly to steady state. Consider an input that is presented to the network from time $T_1$ until $T_2$ and that the network converges at $T_{ss}<T_2$. The change in weight prescribed by \eqref{eq:W_dynamics} is then equal to 
% \begin{align}
%     \int_{T_1}^{T_2}\dd W_i = \int_{T_1}^{T_{ss}}\dd W_i + \frac{T_2 - T_{ss}}{\tau_W} (\ve{v}_{i,\mathrm{ss}}^S - \ve{v}_{i,\mathrm{ss}}^B)\ve{r}_{i-1,\mathrm{ss}}^T
% \end{align}
% If $T_{ss}-T_1 \ll T_2 - T_{ss}$, then the second term in the right-hand-side (RHS) dominates and we have that $\int_{T_1}^{T_2}\dd W_i \approx \Delta W_i$ with $\eta = (T_2 - T_{ss})/\tau_W$. 

\subsection{DFC uses weighted minimum norm updates} \label{sec:minimum_norm_updates}
\vspace{-0.2cm}
GN optimization with a minibatch size of 1 is equivalent to MN updates \citep{meulemans2020theoretical}, i.e., it computes the smallest possible weight update such that the network exactly reaches the current output target after the update. These MN updates can be generalized to weighted MN updates for targets using arbitrary loss functions. The following theorem shows the connection between DFC and these weighted MN updates, while removing the need for Condition \ref{con:magnitude_r} and an $L^2$ loss (full proof in App. \ref{app:learning_theory}). 

\begin{theorem}\label{theorem:MN_DFC}
Assuming stable dynamics, Condition \ref{con:Q_GN} holds and $\lambda, \alpha \rightarrow 0$, the steady-state weight updates \eqref{eq:update_W_ss} are proportional to the weighted MN updates of $W_i$ for letting the feedforward output $\ve{r}^-_L$ reach $\ve{r}_L^*$, i.e., the solution to the following optimization problem:
\begin{align}
    \argmin_{\Delta W_i, i \in [1,..,L]}\quad \sum_{i=1}^{L}\|\ve{r}_{i-1}^{-(m)}\|_2^2\|\Delta W_i\|_F^2 \qquad \text{s.t.}\quad \ve{r}^{-(m+1)}_L = \ve{r}_L^{*(m)},
\end{align}
with $m$ the iteration and $\ve{r}^{-(m+1)}_L$ the network output without feedback after the weight update.
\end{theorem}

Theorem \ref{theorem:MN_DFC} shows that Condition \ref{con:Q_GN} enables the controller to drive the network towards its target $\ve{r}_L^*$ with MN activation changes, $\Delta \ve{v} = \ve{v} - \ve{v}^{\text{ff}}$, which combined with the steady-state weight update \eqref{eq:update_W_ss} result in weighted MN updates $\Delta W_i$ (see also App. \ref{app:intuitive_condition2}). When the feedback weights do not have the correct column space, the weight updates will not be MN. Nevertheless, the following proposition shows that the weight updates still follow a descent direction given arbitrary feedback weights. 
% This result also makes it possible to show convergence for linear networks with unconstrained feedback weights.

% This theorem adds a new intuition to Lemma \ref{lemma:steady_state_solution_system}. The weight space dimension is much larger compared to the output space, hence an infinite number of possible weight updates $\Delta W_i$ could push the output exactly towards its target $\ve{r}_L^*$. However, if the feedback weights have the correct column space (Condition \ref{con:Q_GN}), the controller combined with the network dynamics makes sure that the weight updates are minimum-norm. 

\begin{proposition}\label{prop:descent_direction}
Assuming stable dynamics and $\lambda, \alpha \rightarrow 0$, the steady-state weight updates \eqref{eq:update_W_ss} with a layer-specific learning rate $\eta_i = \eta/\|r_{i-1}\|_2^2$ lie within 90 degrees of the loss gradient direction. 
% Furthermore, if $\phi$ is linear, training with $\Delta W_{i,\mathrm{ss}}$ and $\eta_i$ converges to the global minimum of the loss function.
% the postsynaptic error signal at steady state $\ve{v}_{\mathrm{ss}} - \ve{v}^{\text{ff}}_{\mathrm{ss}}$ is always within 90 degrees of the loss gradient w.r.t. $\ve{v}$. Furthermore, if Condition \ref{con:magnitude_r} holds, $\phi$ is linear and $\eta \rightarrow 0$, the training with steady-state updates $\eqref{eq:update_W_ss}$ converges to the global minimum of the loss function. 
\end{proposition}
\vspace{-0.1cm}

\section{Stability of DFC}\label{sec:stability_TPDI}
\vspace{-0.3cm}
Until now, we assumed that the network dynamics are stable, which is necessary for DFC, as an unstable network will diverge, making learning impossible. In this section, we investigate the conditions on the feedback weights $Q_i$ necessary for stability. To gain intuition, we linearize the network around its feedforward values, assume a separation of timescales between the controller and the network ($\tau_u \gg \tau_v$), and only consider integrative control ($k_p=0$). This results in the following dynamics (see App. \ref{app:stability} for the derivation):
\begin{align}\label{eq:linearization_u}
    \tau_u \ddt \ve{u}(t) = - (JQ + \alpha I)\ve{u}(t) + \ve{\delta}_L.
\end{align}
Hence, in this simplified case, the local stability of the network around the equilibrium point depends on the eigenvalues of $JQ$, which is formalized in the following condition and proposition.

\begin{condition}\label{con:local_stability}
    Given the network Jacobian evaluated at the steady state, $J_{\mathrm{ss}} \triangleq \left.\left[\frac{\partial \ve{r}^-_L}{\partial \ve{v}_1},...,\frac{\partial \ve{r}^-_L}{\partial \ve{v}_L}\right]\right\rvert_{\ve{v}=\ve{v}_{\mathrm{ss}}}$, the real parts of the eigenvalues of $J_{\mathrm{ss}}Q$ are all greater than $-\alpha$.
\end{condition}
\begin{proposition}\label{prop:local_stability}
Assuming $\tau_u \gg \tau_v$ and $k_p=0$, the network and controller dynamics are locally asymptotically stable around its equilibrium iff Condition \ref{con:local_stability} holds.
\end{proposition}
This proposition follows directly from Lyapunov's Indirect Method \citep{doi:10.1080/00207179208934253}. When assuming the more general case where $\tau_v$ is not negligible and $k_p > 0$, the stability criteria quickly become less interpretable (see App. \ref{app:stability}). However, experimentally, we see that Condition \ref{con:local_stability} is a good proxy condition for guaranteeing stability in the general case where $\tau_v$ is not negligible and $k_p > 0$ (see Section \ref{sec:experiments} and App. \ref{app:stability}).

\section{Learning the feedback weights} \label{sec:fb_learning}
\vspace{-0.3cm}
Condition \ref{con:Q_GN} and \ref{con:local_stability} emphasize the importance of the feedback weights for enabling efficient learning and ensuring stability of the network dynamics, respectively. As the forward weights, and hence the network Jacobian, $J$, change during training, the set of feedback configurations that satisfy Conditions \ref{con:Q_GN} and \ref{con:local_stability} also change. This creates the need to adapt the feedback weights accordingly to ensure efficient learning and network stability. We solve this challenge by learning the feedback weights, such that they can adapt to the changing network during training. We separate forward and feedback weight training in alternating wake-sleep phases \citep{hinton1995wake}. Note that in practice, a fast alternation between the two phases is not required (see Section \ref{sec:experiments}). 
%Furthermore, we can merge both phases into a single phase, if the feedback synapses can high-pass filter their input $\ve{u}$ (see App. \ref{app:fb_learning_simultaneously}). 

Inspired by the Weight Mirror method \citep{akrout2019deep}, we learn the feedback weights by inserting independent zero-mean noise $\ves{\epsilon}$ in the system dynamics:
\begin{align}\label{eq:network_dynamic}
    \tau_v \ddt \ve{v}_i(t) = -\ve{v}_i(t) + W_i\phi\big(\ve{v}_{i-1}(t)\big) + Q_i\ve{u}(t) + \sigma \ves{\epsilon}_i.
\end{align}
The noise fluctuations propagated to the output carry information from the network Jacobian, $J$. To let $\ve{e}$, and hence $\ve{u}$, incorporate this noise information, we set the output target $\ve{r}^*_L$ to the average network output $\ve{r}_L^-$.
As the network is continuously perturbed by noise, the controller will try to counteract the noise and regulate the network towards the output target $\ve{r}_L^-$. The feedback weights can then be trained with a simple anti-Hebbian plasticity rule with weight decay, which is local in space and time:
\begin{align}\label{eq:Q_dynamics}
    \tau_Q \ddt Q_i(t) = -\ve{v}^{\text{fb}}_i(t)\ve{u}(t)^T - \beta Q_i,
\end{align}
where $\beta$ is the scale factor of the weight decay term and where we assume that a subset of the noise input $\ves{\epsilon}_i$ enters through the feedback compartment, i.e., $\ve{v}^{\text{fb}}_i = Q_i \ve{u} + \sigma_{\text{fb}} \ves{\epsilon}_i^{\text{fb}}$. The correlation between the noise in $\ve{v}^{\text{fb}}_i$ and noise fluctuations in $\ve{u}$ provides the teaching signal for $Q_i$. Theorem \ref{theorem:fb_learning_simplified} shows under simplifying assumptions that the feedback learning rule \eqref{eq:Q_dynamics} drives $Q_i$ to satisfy Condition \ref{con:Q_GN} and \ref{con:local_stability} (see App. \ref{app:fb_learning} for the full theorem and its proof).

\begin{theorem}[Short version]\label{theorem:fb_learning_simplified}
Assume a separation of timescales $\tau_v \ll \tau_u \ll \tau_Q$, $\alpha$ big, $k_p=0$, $\ve{r}_L^* = \ve{r}_L^-$, and Condition \ref{con:local_stability} holds. Then, for a fixed input sample and $\sigma \rightarrow 0$, the first moment of $Q$ converges approximately to:
\begin{align}\label{eq:t6_Q_ss}
\lim_{\sigma \rightarrow 0}\mathbb{E}[Q_{\mathrm{ss}}] \appropto J^{T} (JJ^T + \gamma I)^{-1},
\end{align}
for some $\gamma > 0$. Furthermore, $\mathbb{E}[Q_{\mathrm{ss}}]$ satisfies Conditions \ref{con:Q_GN} and \ref{con:local_stability}, even if $\alpha = 0$ in the latter.
\end{theorem}
Theorem \ref{theorem:fb_learning_simplified} shows that under simplifying assumptions, $Q$ converges towards a damped pseudoinverse of $J$, which satisfies Conditions \ref{con:Q_GN} and \ref{con:local_stability}. Empirically, we see that this also approximately holds for more general settings where $\tau_v$ is not negligible, $k_p > 0$, and small $\alpha$ (see Section \ref{sec:experiments} and App. \ref{app:fb_learning}). 

The above theorem leaves two questions unanswered. First, it assumes that Condition \ref{con:local_stability} holds, however, the task of the feedback weight training is to make unstable network dynamics stable, resulting in a chicken-and-egg problem. The solution we use is to take $\alpha$ big enough to make the network stable during early training, after which the feedback weights align according to \eqref{eq:t6_Q_ss} and $\alpha$ can be decreased. 
%This scenario could correspond to special conditions during development of the neural circuits.
Second, Theorem \ref{theorem:fb_learning_simplified} considers training the feedback weights to convergence over one fixed input sample. However, in reality many different input samples will be considered during learning. When the network is linear, $J$ is the same for each input sample and eq. \eqref{eq:t6_Q_ss} holds exactly. However, for nonlinear networks, $J$ will be different for each sample, causing the feedback weights to align with an average of $J^{T} (JJ^T + \gamma I)^{-1}$ over many samples. 
%Finally, note that $J^{T} (JJ^T + \gamma I)^{-1}$ is only one possible configuration for the feedback weights that satisfies Condition \ref{con:Q_GN} and \ref{con:local_stability} (see Fig. \ref{fig:pca_feedbackweights}). Future work could capitalize on the flexibility of Condition \ref{con:Q_GN} by designing feedback learning rules that aim to satisfy this condition, without pushing $Q$ to one specific configuration. The increased freedom of $Q$ could then be used for purposes other than learning, such as attention and prediction \cite{gilbert2013top}. 

% Future work could design other feedback weight learning rules leading to more flexible feedback weight configurations compatible with our DFC framework.

\section{Experiments} \label{sec:experiments}
\vspace{-0.3cm}
We evaluate DFC in detail on toy experiments to showcase that our theoretical results translate to practice (Section \ref{sec:toy_regression}) and on a modest range of computer vision benchmarks -- MNIST classification and autoencoding \citep{lecun1998mnist}, and Fashion MNIST classification \citep{xiao2017fashion} -- to show that DFC can do precise CA in more challenging settings (Section \ref{sec:mnist}). Alongside DFC, we test two variants: (i) \textbf{DFC-SS} which only updates its feedforward weights $W_i$ after the steady state (SS) of \eqref{eq:network_dynamics} and \eqref{eq:controller_dynamics} is reached; and (ii) \textbf{DFC-SSA} which analytically computes the linearized steady state of \eqref{eq:network_dynamics} and \eqref{eq:controller_dynamics} according to Lemma \ref{lemma:steady_state_solution_system}. To investigate whether learning the feedback weights is crucial for DFC, we compare for all three settings: (i) learning the feedback weights $Q_i$ according to \eqref{eq:Q_dynamics}; and (ii) fixing the feedback weights to the initialization $Q_i=\prod_{k=i+1}^L W_k^T$, which approximately satisfies Condition \ref{con:Q_GN} and \ref{con:local_stability} at the beginning of training (see App. \ref{app:experiments}), denoted with suffix \textbf{(fixed)}. For the former, we pre-train the feedback weights according to \eqref{eq:Q_dynamics} to ensure stability. During training, we iterate between 1 epoch of forward weight training and $X$ epochs of feedback weight training (if applicable), where $X\in[1,2,3]$ is a hyperparameter.
We compare all variants to Direct Feedback Alignment (\textbf{DFA}) \citep{nokland2016direct} as a control for direct feedback connectivity. DFC is simulated with the Euler-Maruyama method, which is the equivalent of forward Euler for stochastic differential equations \citep{sarkka2019applied}. We initialize the network to its feedforward activations \eqref{eq:feedforward_equilibrium} for each datasample and, for computational efficiency, we buffer the weight updates \eqref{eq:W_dynamics} and \eqref{eq:Q_dynamics} and apply them once at the end of the simulation for the considered datasample. App. \ref{app:simulation_and_algorithms} and \ref{app:experiments} provide further details on the implementation of all experiments.\footnote{PyTorch implementation of all methods is available at \url{https://github.com/meulemansalex/deep_feedback_control}.}

\subsection{Empirical verification of the theory} \label{sec:toy_regression}
\vspace{-0.2cm}
Figure \ref{fig:toy_experiment_tanh} visualizes the theoretical results of Theorems \ref{theorem:GN_TPDI} and \ref{theorem:MN_DFC} and Conditions \ref{con:magnitude_r}, \ref{con:Q_GN} and \ref{con:local_stability}, in an empirical setting of nonlinear student teacher regression, where a randomly initialized teacher network generates synthetic training data for a student network. We see that Condition \ref{con:Q_GN} is approximately satisfied for all DFC variants that learn their feedback weights (Fig. \ref{fig:toy_experiment_tanh}A), leading to close alignment with the ideal weighted MN updates of Theorem \ref{theorem:MN_DFC} (Fig. \ref{fig:toy_experiment_tanh}B). For nonlinear networks and linear direct feedback, it is in general not possible to perfectly satisfy Condition \ref{con:Q_GN} as the network Jacobian $J$ varies for each datasample, while $Q_i$ remains the same. However, the results indicate that feedback learning finds a configuration for $Q_i$ that approximately satisfies Condition \ref{con:Q_GN} for all datasamples. When the feedback weights are fixed, Condition \ref{con:Q_GN} is approximately satisfied in the beginning of training due to a good initialization. However, as the network changes during training, Condition \ref{con:Q_GN} degrades modestly, which results in worse alignment compared to DFC with trained feedback weights (Fig. \ref{fig:toy_experiment_tanh}B). 

For having GN updates, both Conditions \ref{con:magnitude_r} and \ref{con:Q_GN} need to be satisfied. Although we do not enforce Condition \ref{con:magnitude_r} during training, we see in Fig. \ref{fig:toy_experiment_tanh}C that it is crudely satisfied, which can be explained by the saturating properties of the $\tanh$ nonlinearity. This is reflected in the alignment with the ideal GN updates in Fig. \ref{fig:toy_experiment_tanh}D that follows the same trend as the alignment with the MN updates. Fig. \ref{fig:toy_experiment_tanh}E shows that all DFC variants remain stable throughout training, even when the feedback weights are fixed. In App. \ref{app:stability}, we indicate that Condition \ref{con:local_stability} is a good proxy for the stability shown in Fig. \ref{fig:toy_experiment_tanh}E. Finally, we see in Fig. \ref{fig:toy_experiment_tanh}F that the weight updates of DFC and DFC-SS align well with the analytical steady-state solution of Lemma \ref{lemma:steady_state_solution_system}, confirming that our learning theory of Section \ref{sec:learning_TPDI} applies to the continuous weight updates \eqref{eq:W_dynamics} of DFC.

\begin{figure}[h!]
    \vspace*{-0.2cm}
    \centering
    \includegraphics[width=0.97\linewidth]{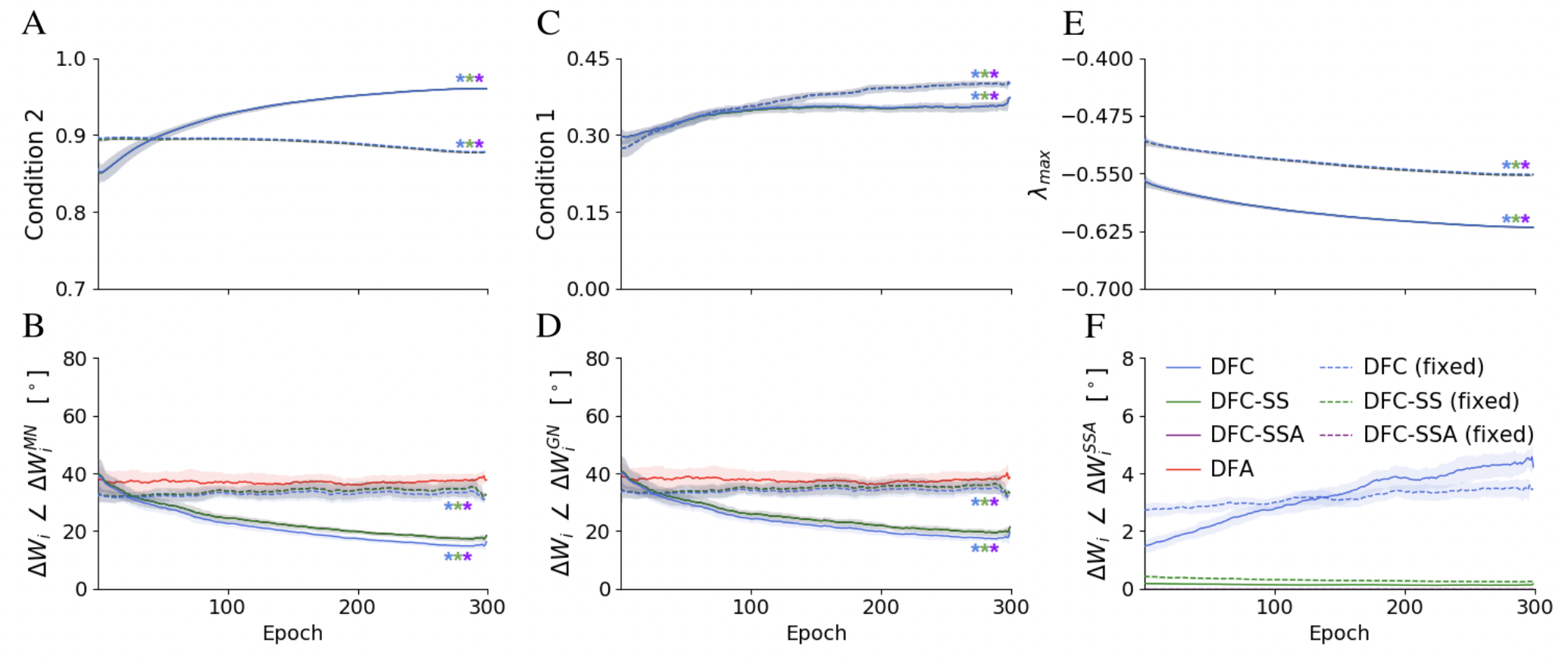}
    \caption[Results]{Results for nonlinear student-teacher regression task with layer sizes (15-10-10-5), $\tanh$ nonlinearities, a linear output layer, $k_p=1.5$, $\lambda=0.05$, and $\alpha=0.0015$. (\textbf{A}) Ratio between the norms of $Q$ projected into the row space of $J$, and $Q$, with a value of 1 indicating perfect compliance of Condition \ref{con:Q_GN}. (\textbf{B},\textbf{D},\textbf{F}) Angle in degrees between the concatenated parameter updates of the whole network and: (\textbf{B}) the ideal weighted MN parameter updates (Theorem \ref{theorem:MN_DFC}); (\textbf{D}) the ideal GN parameter updates (Theorem \ref{theorem:GN_TPDI}); and (\textbf{F}) the DFC-SSA parameter updates (see App. \ref{app:alignment_measures} for all definitions). (\textbf{C}) The standard deviation of the layer norms $\|\ve{r}_i\|_2$, divided by the average layer norm, with a value of zero indicating perfect compliance to Condition \ref{con:magnitude_r}. (\textbf{E}) The maximum real part of the eigenvalues of the total system dynamics matrix evaluated at equilibrium (see App. \ref{app:alignment_measures}), with negative real parts indicating local stability. For all measures, a window-average is plotted together with the window-std (shade). Stars indicate overlapping plots.}
    \label{fig:toy_experiment_tanh}
    \vspace*{-0.4cm}
\end{figure}

% In Table \ref{tab:lambda_alpha_robustness}, we show that the performance of DFC is robust to a wide range of values for $\lambda$ and $\alpha$. 
In Fig. 4, we show that the alignment with MN updates remains robust for $\lambda \in [10^{-3}:10^{-1}]$ and $\alpha \in [10^{-4}:10^{-1}]$, highlighting that our theory explains the behavior of DFC robustly when the limit of $\lambda$ and $\alpha$ to zero does not hold. When we clamp the output target to the label ($\lambda = 0.5$), the alignment with the MN updates decreases as expected (see Fig. 4), because the linearization of Lemma \ref{lemma:steady_state_solution_system} becomes less accurate and the strong feedback changes the neural activations more significantly, thereby changing the pre-synaptic factor of the update rules (c.f. eq. \ref{eq:update_W_ss}). However, performance results on MNIST, provided in Table \ref{tab:lambda_alpha_robustness}, show that the performance of DFC remains robust for a wide range of $\lambda$s and $\alpha$s, including $\lambda = 0.5$, suggesting that DFC can also provide principled CA in this setting of strong feedback, which motivates future work to design a complementary theory for DFC focused on this extreme case.

\begin{figure}[h!]
\vspace*{-0.2cm}
    \centering
    \begin{minipage}{0.7\textwidth}
    \centering
    \includegraphics[width=0.9\linewidth]{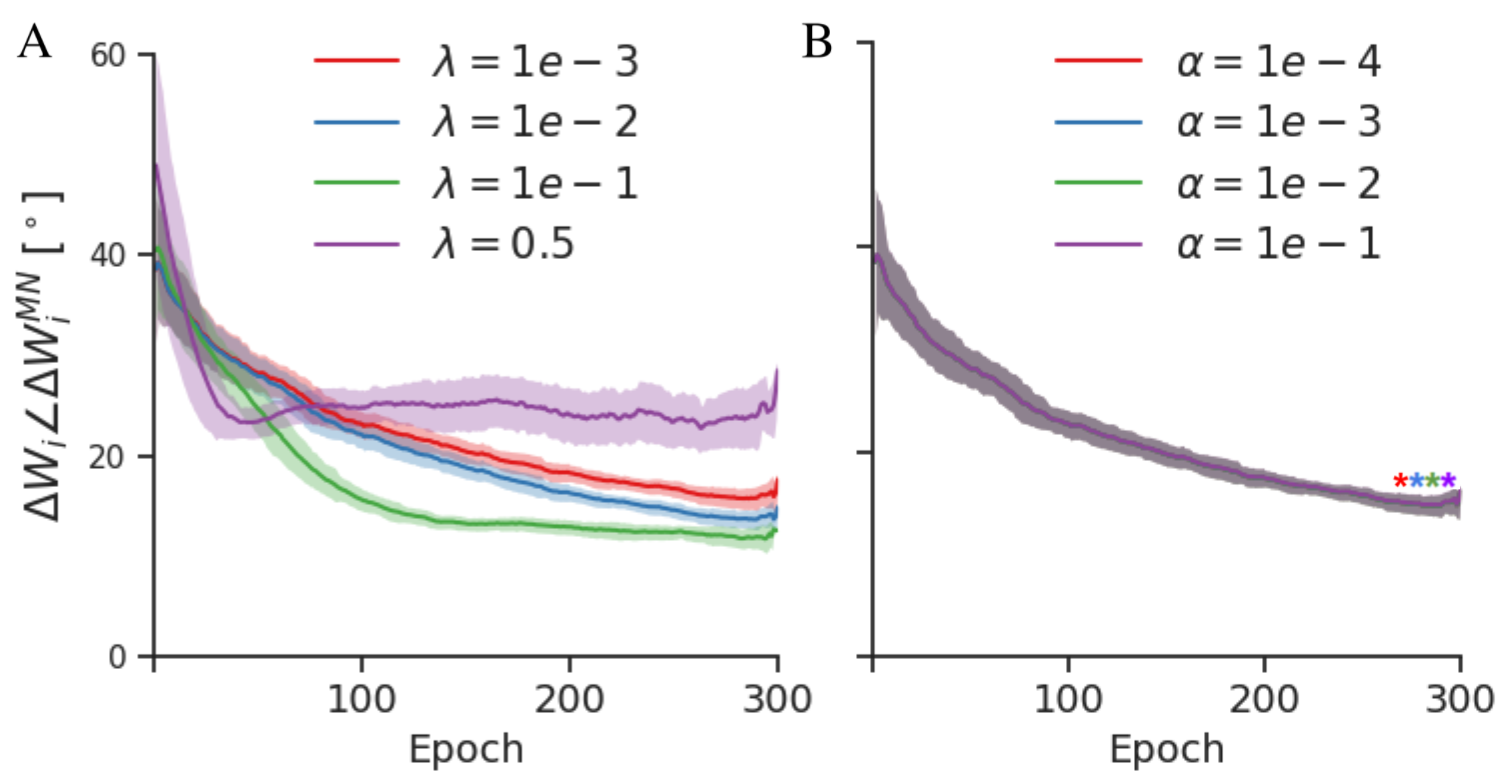}
    \label{fig:gnt_lambda_alpha_robustness_angle_alignment}
    \end{minipage}
    \begin{minipage}{0.26\textwidth}
    \begin{flushleft}
    \vspace{-0.35cm}
        Figure 4: Comparison of the alignment between the DFC weight updates and the MN updates for variable values of $\lambda$ (A) and $\alpha$ (B), when performing the nonlinear student-teacher regression task described in Fig. \ref{fig:toy_experiment_tanh}. Stars indicate overlapping plots.
    \end{flushleft}
    \end{minipage}
    \vspace*{-0.2cm}
\end{figure}

\subsection{Performance of DFC on computer vision benchmarks} \label{sec:mnist}
\vspace{-0.2cm}
The classification results on MNIST and Fashion-MNIST (Table \ref{tab:test_results}) show that the performances of DFC and its variants, but also its controls, lie close to the performance of BP, indicating that they perform proper CA in these tasks. To see significant differences between the methods, we consider the more challenging task of training an autoencoder on MNIST, where it is known that DFA fails to provide precise CA \citep{lillicrap2016random,lansdell2019learning, podlaski2020biological}. The results in Table \ref{tab:test_results} show that the DFC variants with trained feedback weights clearly outperform DFA and have close performance to BP. The low performance of the DFC variants with fixed feedback weights show the importance of learning the feedback weights continuously during training to satisfy Condition \ref{con:Q_GN}. Finally, to disentangle optimization performance from implicit regularization mechanisms, which both influence the test performance, we investigate the performance of all methods in minimizing the training loss of MNIST.\footnote{We used separate hyperparameter configurations, selected for minimizing the training loss.} The results in Table \ref{tab:test_results} show improved performance of the DFC method with trained feedback weights compared to BP and controls, suggesting that the approximate MN updates of DFC can faster descend the loss landscape for this simple dataset.
\vspace{-0.2cm}

\begin{table}[h]
\centering
\vspace*{-0.2cm}
\caption[Test errors]{Test errors (classification) and test loss (autoencoder) corresponding to the epoch with the best validation result (for 5000 validation samples) over a training of 100 epochs (classification) or 25 epochs (autoencoder). Training loss after 100 epochs (MNIST train loss). We use the Adam optimizer \citep{kingma2014adam}. Architectures: 3x256 fully connected (FC) tanh hidden layers and softmax output (classification), 256-32-256 FC hidden layers for autoencoder MNIST with tanh-linear-tanh nonlinearities, and a linear output. Mean $\pm$ std (5 random seeds). Best results (except BP) are displayed in bold.}
\label{tab:test_results}
\begin{tabular}{*5c}
\toprule
{}   & MNIST & Fashion-MNIST  & MNIST-autoencoder  & MNIST (train loss)   \\
\midrule
BP & $2.08^{\pm 0.15}\%$ & $10.60^{\pm 0.34}\%$ & $9.42^{\pm 0.09} \cdot 10^{-2}$ &  $1.53^{\pm 0.19}\cdot 10^{-7}$  \\
\midrule
DFC   &  $2.25^{\pm 0.094}\%$ & $11.17^{\pm 0.27}\%$  & $11.28^{\pm 0.18} \cdot 10^{-2}$ & $7.61^{\pm 0.65}\cdot 10^{-8}$ \\
DFC-SSA  &  $\mathbf{2.18^{\pm 0.16}\%}$ & $11.28^{\pm 0.27}\%$ & $11.27^{\pm 0.09} \cdot 10^{-2}$  & $4.89^{\pm 1.26}\cdot 10^{-8}$ \\
DFC-SS  &  $2.29^{\pm 0.097}\%$ & $\mathbf{11.15^{\pm 0.32}}\%$  & $\mathbf{11.21^{\pm 0.04} \cdot 10^{-2}}$ & $\mathbf{4.80^{\pm 0.70} \cdot 10^{-8}}$ \\
DFC (fixed)   &  $2.47^{\pm 0.12}\%$ & $11.62^{\pm 0.30}\%$ & $33.37^{\pm0.60} \cdot 10^{-2}$ & $1.30^{\pm 0.15} \cdot 10^{-6}$ \\
DFC-SSA (fixed)   &  $2.46^{\pm 0.11}\%$ & $11.44^{\pm 0.14}\%$  & $31.90^{\pm 0.77} \cdot 10^{-2}$ & $1.73^{\pm 0.39} \cdot 10^{-6}$\\
DFC-SS (fixed) 
&  $2.39^{\pm 0.22}\%$ & $11.55^{\pm 0.42}\%$ & $32.31^{\pm 0.37} \cdot 10^{-2}$  & $1.67^{\pm 0.70} \cdot 10^{-6}$ \\
DFA   &  $2.69^{\pm 0.11}\%$ & $11.38^{\pm 0.25}\%$ & $29.95^{\pm 0.36} \cdot 10^{-2}$ & $7.09^{\pm 1.11} \cdot 10^{-7}$ \\
% MN   &  $2.11^{\pm 0.046}\%$ & $11.39^{\pm 0.39}\%$ & - & $12.01^{\pm -} \cdot 10^{-2}$ \\
\bottomrule
\end{tabular}
\vspace*{-0.2cm}
\end{table}

\begin{table}[H]
\vspace*{-0.2cm}
\centering
\caption{Test errors on MNIST with variable $\lambda$ values and fixed $\alpha = 0.0015$ (left), and with variable $\alpha$ values and fixed $\lambda=0.08$ (right). Same experimental setting as in Table \ref{tab:test_results}.} \label{tab:lambda_alpha_robustness}
\begin{tabular}{*3c|*3c}
\toprule
$\lambda$ & DFC-SS & DFC & $\alpha$ & DFC-SS & DFC \\
\midrule
$1e^{-3}$ & $2.26^{\pm 0.11}\%$ & $2.29^{\pm 0.04}\%$ & $1e^{-4}$ & $2.31^{\pm 0.12}\%$ & $2.28^{\pm 0.06}\%$ \\
$1e^{-2}$ & $2.25^{\pm 0.05}\%$ & $2.31^{\pm 0.04}\%$ & $1e^{-3}$ & $2.28^{\pm 0.15}\%$ & $2.31^{\pm 0.11}\%$\\
$1e^{-1}$ & $2.27^{\pm 0.07}\%$ & $2.30^{\pm 0.06}\%$ & $1e^{-2}$ & $2.26^{\pm 0.05}\%$ & $2.32^{\pm 0.12}\%$ \\
$0.5$ & $2.31^{\pm 0.15}\%$ & $2.34^{\pm 0.15}\%$ & $1e^{-1}$ & $2.28^{\pm 0.11}\%$ & $2.34^{\pm 0.16}\%$ \\
\bottomrule
\end{tabular}
\vspace*{-0.2cm}
\end{table}

% \subsection{(Reinforcement learning)}
% Optional experiment to show that our framework also works with other types of 'errors' like policy gradients etc. 
\section{Discussion}
\vspace{-0.3cm}
We introduced DFC as an alternative biologically-plausible learning method for deep neural networks. DFC uses error feedback to drive the network activations to a desired output target. This process generates a neuron-specific learning signal which can be used to learn both forward and feedback weights locally in time and space. In contrast to other recent methods that learn the feedback weights and aim to approximate BP \citep{akrout2019deep, kunin2020two, lansdell2019learning, Guerguiev2020Spike-based, payeur2020burst}, we show that DFC approximates GN optimization, making it fundamentally different from BP approximations. 

DFC is optimal -- i.e., Conditions \ref{con:Q_GN} and \ref{con:local_stability} are satisfied -- for a wide range of feedback connectivity strengths. Thus, we prove that principled learning can be achieved with local rules and without symmetric feedforward and feedback connectivity by leveraging the network dynamics. This finding has interesting implications for experimental neuroscientific research looking for precise patterns of symmetric connectivity in the brain. Moreover, from a computational standpoint, the flexibility that stems from Conditions \ref{con:Q_GN} and \ref{con:local_stability} might be relevant for other mechanisms besides learning, such as attention and prediction \citep{gilbert2013top}.

To present DFC in its simplest form, we used direct feedback mappings from the output controller to all hidden layers. Although numerous anatomical studies of the mammalian neocortex reported the occurrence of such direct feedback connections \citep{ungerleider2008cortical, rockland1994direct}, it is unlikely that all feedback pathways are direct. We note that DFC is also compatible with other feedback mappings, such as layerwise connections or separate feedback pathways with multiple layers of neurons (see App. \ref{app:fb_pathways}).

Interestingly, the three-compartment neuron is closely linked to recent multi-compartment models of the cortical pyramidal neuron \citep{sacramento2018dendritic, guerguiev2017towards, payeur2020burst, richards2019dendritic}. In the terminology of these models, our central, feedforward, and feedback compartments, correspond to the somatic, basal dendritic, and apical dendritic compartments of pyramidal neurons, respectively (see Fig. \ref{fig:DFC_schematics}C). In line with DFC, experimental observations \citep{larkum2013cellular,spruston2008pyramidal} suggest that feedforward connections converge onto the basal compartment and feedback connections onto the apical compartment. Moreover, our plasticity rule for the forward weights \eqref{eq:W_dynamics} belongs to a class of dendritic predictive plasticity rules for which a biological implementation based on backpropagating action potentials has been put forward \citep{urbanczik2014learning}. 

% \paragraph{Biological plausibility of DFC}
% \begin{itemize}
%     \item flexible architecture
%     \item direct feedback mappings? Also point to other possible feedback architectures in the appendix
%     \item robustness to noise, noise as a feature for fb learning
%     \item merge 2 phases into one
% \end{itemize}

\textbf{Limitations and future work.}
In practice, the forward weight updates are not exactly equal to GN or MN updates (Theorems \ref{theorem:GN_TPDI} and \ref{theorem:MN_DFC}), due to (i) the nonlinearity $\phi$ in the weight update rule \ref{eq:W_dynamics}, (ii) non-infinitesimal values for $\alpha$ and $\lambda$, (iii) limited training iterations for the feedback weights, and (iv) the limited capacity of linear feedback mappings to satisfy Condition \ref{con:Q_GN} for each datasample. Figs. \ref{fig:toy_experiment_tanh} and 4, and Table \ref{tab:lambda_alpha_robustness} show that DFC approximates the theory well in practice and has robust performance, however, future work can improve the results further by investigating new feedback architectures (see App. \ref{app:fb_pathways}). We note that, even though GN optimization has desirable approximate second-order optimization properties, it is presently unclear whether these second-order characteristics translate to our setting with a minibatch size of 1. 
Currently, our proposed feedback learning rule \eqref{eq:Q_dynamics} aims to approximate one specific configuration and hence does not capitalize on the increased flexibility of DFC and Condition \ref{con:Q_GN}. Therefore, an interesting future direction is to design more flexible feedback learning rules that aim to satisfy Conditions \ref{con:Q_GN} and \ref{con:local_stability} without targeting one specific configuration.
%Finally, note that $J^{T} (JJ^T + \gamma I)^{-1}$ is only one possible configuration for the feedback weights that satisfies Condition \ref{con:Q_GN} and \ref{con:local_stability} (see Fig. \ref{fig:pca_feedbackweights}). Future work could capitalize on the flexibility of Condition \ref{con:Q_GN} by designing feedback learning rules that aim to satisfy this condition, without pushing $Q$ to one specific configuration. The increased freedom of $Q$ could then be used for purposes other than learning, such as attention and prediction \cite{gilbert2013top}. 
Furthermore, DFC needs two separate phases for training the forward weights and feedback weights. Interestingly, if the feedback plasticity rule \eqref{eq:Q_dynamics} uses a high-passed filtered version of the presynaptic input $\ve{u}$, both phases can be merged into one, with plasticity always on for both forward and feedback weights (see App. \ref{app:fb_learning_simultaneously}).
Finally, as DFC is dynamical in nature, it is costly to simulate on commonly used hardware for deep learning, prohibiting us from testing DFC on large-scale problems such as those considered by \citet{bartunov2018assessing}. 
A promising alternative is to implement DFC on analog hardware, where the dynamics of DFC can correspond to real physical processes on a chip. This would not only make DFC resource-efficient, but also position DFC as an interesting training method for analog implementations of deep neural networks, commonly used in Edge AI and other applications where low energy consumption is key \cite{xiao2020analog,misra2010artificial}. 
%Here, DFC could alleviate persistent problems such as device mismatch  \citep{bult2000analog,kinget2005device}. 

To conclude, we show that DFC can provide principled CA in deep neural networks by actively using error feedback to drive neural activations. The flexible requirements for feedback mappings combined with the strong link between DFC and GN, underline that it is possible to do principled CA in neural networks without adhering to the symmetric layer-wise feedback structure imposed by BP.

\clearpage

\begin{ack}
This work was supported by the Swiss National Science Foundation (B.F.G. CRSII5-173721 and 315230\_189251), ETH project funding (B.F.G. ETH-20 19-01), the Human Frontiers Science Program (RGY0072/2019) and funding from the Swiss Data Science Center (B.F.G, C17-18, J. v. O. P18-03). João Sacramento was supported by an Ambizione grant (PZ00P3\_186027) from the Swiss National Science Foundation. Pau Vilimelis Aceituno was supported by an ETH Zürich Postdoc fellowship. Javier García Ordóñez received support from La Caixa Foundation through the Postgraduate Studies in Europe scholarship. We would like to thank Anh Duong Vo and Nicolas Zucchet for feedback, William Podlaski, Jean-Pascal Pfister and Aditya Gilra for insightful discussions, and Simone Surace for his detailed feedback on Appendix \ref{app:fb_learning}.1.
\end{ack}

% \section*{References}
% \bibliography{main.bib}
\bibliography{main.bib}

\begin{thebibliography}{76}
\providecommand{\natexlab}[1]{#1}
\providecommand{\url}[1]{\texttt{#1}}
\expandafter\ifx\csname urlstyle\endcsname\relax
  \providecommand{\doi}[1]{doi: #1}\else
  \providecommand{\doi}{doi: \begingroup \urlstyle{rm}\Url}\fi

\bibitem[Rumelhart et~al.(1986)Rumelhart, Hinton, and
  Williams]{rumelhart1986learning}
David~E Rumelhart, Geoffrey~E Hinton, and Ronald~J Williams.
\newblock Learning representations by back-propagating errors.
\newblock \emph{Nature}, 323\penalty0 (6088):\penalty0 533, 1986.

\bibitem[Werbos(1982)]{werbos1982applications}
Paul~J Werbos.
\newblock Applications of advances in nonlinear sensitivity analysis.
\newblock In \emph{System modeling and optimization}, pages 762--770. Springer,
  1982.

\bibitem[Linnainmaa(1970)]{linnainmaa1970representation}
Seppo Linnainmaa.
\newblock The representation of the cumulative rounding error of an algorithm
  as a taylor expansion of the local rounding errors.
\newblock \emph{Master's Thesis (in Finnish), Univ. Helsinki}, pages 6--7,
  1970.

\bibitem[Crick(1989)]{crick1989recent}
Francis Crick.
\newblock The recent excitement about neural networks.
\newblock \emph{Nature}, 337\penalty0 (6203):\penalty0 129--132, 1989.

\bibitem[Grossberg(1987)]{grossberg1987competitive}
Stephen Grossberg.
\newblock Competitive learning: From interactive activation to adaptive
  resonance.
\newblock \emph{Cognitive Science}, 11\penalty0 (1):\penalty0 23--63, 1987.

\bibitem[Lillicrap et~al.(2020)Lillicrap, Santoro, Marris, Akerman, and
  Hinton]{lillicrap2020backpropagation}
Timothy~P Lillicrap, Adam Santoro, Luke Marris, Colin~J Akerman, and Geoffrey
  Hinton.
\newblock Backpropagation and the brain.
\newblock \emph{Nature Reviews Neuroscience}, pages 1--12, 2020.

\bibitem[Larkum et~al.(2009)Larkum, Nevian, Sandler, Polsky, and
  Schiller]{larkum2009synaptic}
Matthew~E Larkum, Thomas Nevian, Maya Sandler, Alon Polsky, and Jackie
  Schiller.
\newblock Synaptic integration in tuft dendrites of layer 5 pyramidal neurons:
  a new unifying principle.
\newblock \emph{Science}, 325\penalty0 (5941):\penalty0 756--760, 2009.

\bibitem[Gilbert and Li(2013)]{gilbert2013top}
Charles~D Gilbert and Wu~Li.
\newblock Top-down influences on visual processing.
\newblock \emph{Nature Reviews Neuroscience}, 14\penalty0 (5):\penalty0
  350--363, 2013.

\bibitem[Lillicrap et~al.(2016)Lillicrap, Cownden, Tweed, and
  Akerman]{lillicrap2016random}
Timothy~P Lillicrap, Daniel Cownden, Douglas~B Tweed, and Colin~J Akerman.
\newblock Random synaptic feedback weights support error backpropagation for
  deep learning.
\newblock \emph{Nature Communications}, 7:\penalty0 13276, 2016.

\bibitem[Bartunov et~al.(2018)Bartunov, Santoro, Richards, Marris, Hinton, and
  Lillicrap]{bartunov2018assessing}
Sergey Bartunov, Adam Santoro, Blake Richards, Luke Marris, Geoffrey~E Hinton,
  and Timothy Lillicrap.
\newblock Assessing the scalability of biologically-motivated deep learning
  algorithms and architectures.
\newblock In \emph{Advances in Neural Information Processing Systems 31}, pages
  9368--9378, 2018.

\bibitem[Launay et~al.(2019)Launay, Poli, and Krzakala]{launay2019principled}
Julien Launay, Iacopo Poli, and Florent Krzakala.
\newblock Principled training of neural networks with direct feedback
  alignment.
\newblock \emph{arXiv preprint arXiv:1906.04554}, 2019.

\bibitem[Moskovitz et~al.(2018)Moskovitz, Litwin-Kumar, and
  Abbott]{moskovitz2018feedback}
Theodore~H Moskovitz, Ashok Litwin-Kumar, and LF~Abbott.
\newblock Feedback alignment in deep convolutional networks.
\newblock \emph{arXiv preprint arXiv:1812.06488}, 2018.

\bibitem[Crafton et~al.(2019)Crafton, Parihar, Gebhardt, and
  Raychowdhury]{crafton2019direct}
Brian~Alexander Crafton, Abhinav Parihar, Evan Gebhardt, and Arijit
  Raychowdhury.
\newblock Direct feedback alignment with sparse connections for local learning.
\newblock \emph{Frontiers in Neuroscience}, 13:\penalty0 525, 2019.

\bibitem[Akrout et~al.(2019)Akrout, Wilson, Humphreys, Lillicrap, and
  Tweed]{akrout2019deep}
Mohamed Akrout, Collin Wilson, Peter Humphreys, Timothy Lillicrap, and
  Douglas~B Tweed.
\newblock Deep learning without weight transport.
\newblock In \emph{Advances in Neural Information Processing Systems 32}, pages
  974--982, 2019.

\bibitem[Kunin et~al.(2020)Kunin, Nayebi, Sagastuy-Brena, Ganguli, Bloom, and
  Yamins]{kunin2020two}
Daniel Kunin, Aran Nayebi, Javier Sagastuy-Brena, Surya Ganguli, Jonathan
  Bloom, and Daniel Yamins.
\newblock Two routes to scalable credit assignment without weight symmetry.
\newblock In \emph{International Conference on Machine Learning}, pages
  5511--5521. PMLR, 2020.

\bibitem[Lansdell et~al.(2020)Lansdell, Prakash, and
  Kording]{lansdell2019learning}
Benjamin~James Lansdell, Prashanth Prakash, and Konrad~Paul Kording.
\newblock Learning to solve the credit assignment problem.
\newblock In \emph{International Conference on Learning Representations}, 2020.

\bibitem[Guerguiev et~al.(2020)Guerguiev, Kording, and
  Richards]{Guerguiev2020Spike-based}
Jordan Guerguiev, Konrad Kording, and Blake Richards.
\newblock Spike-based causal inference for weight alignment.
\newblock In \emph{International Conference on Learning Representations}, 2020.

\bibitem[Golkar et~al.(2020)Golkar, Lipshutz, Bahroun, Sengupta, and
  Chklovskii]{golkar2020biologically}
Siavash Golkar, David Lipshutz, Yanis Bahroun, Anirvan~M. Sengupta, and
  Dmitri~B. Chklovskii.
\newblock A biologically plausible neural network for local supervision in
  cortical microcircuits, 2020.

\bibitem[Bengio(2014)]{bengio2014auto}
Yoshua Bengio.
\newblock How auto-encoders could provide credit assignment in deep networks
  via target propagation.
\newblock \emph{arXiv preprint arXiv:1407.7906}, 2014.

\bibitem[Lee et~al.(2015)Lee, Zhang, Fischer, and Bengio]{lee2015difference}
Dong-Hyun Lee, Saizheng Zhang, Asja Fischer, and Yoshua Bengio.
\newblock Difference target propagation.
\newblock In \emph{Joint european conference on machine learning and knowledge
  discovery in databases}, pages 498--515. Springer, 2015.

\bibitem[Meulemans et~al.(2020)Meulemans, Carzaniga, Suykens, Sacramento, and
  Grewe]{meulemans2020theoretical}
Alexander Meulemans, Francesco Carzaniga, Johan Suykens, Jo\~{a}o Sacramento,
  and Benjamin~F. Grewe.
\newblock A theoretical framework for target propagation.
\newblock \emph{Advances in Neural Information Processing Systems},
  33:\penalty0 20024--20036, 2020.

\bibitem[Bengio(2020)]{bengio2020deriving}
Yoshua Bengio.
\newblock Deriving differential target propagation from iterating approximate
  inverses.
\newblock \emph{arXiv preprint arXiv:2007.15139}, 2020.

\bibitem[Sacramento et~al.(2018)Sacramento, Costa, Bengio, and
  Senn]{sacramento2018dendritic}
Jo{\~a}o Sacramento, Rui~Ponte Costa, Yoshua Bengio, and Walter Senn.
\newblock Dendritic cortical microcircuits approximate the backpropagation
  algorithm.
\newblock In \emph{Advances in Neural Information Processing Systems 31}, pages
  8721--8732, 2018.

\bibitem[Whittington and Bogacz(2017)]{whittington2017approximation}
James~CR Whittington and Rafal Bogacz.
\newblock An approximation of the error backpropagation algorithm in a
  predictive coding network with local hebbian synaptic plasticity.
\newblock \emph{Neural computation}, 29\penalty0 (5):\penalty0 1229--1262,
  2017.

\bibitem[Guerguiev et~al.(2017)Guerguiev, Lillicrap, and
  Richards]{guerguiev2017towards}
Jordan Guerguiev, Timothy~P Lillicrap, and Blake~A Richards.
\newblock Towards deep learning with segregated dendrites.
\newblock \emph{ELife}, 6:\penalty0 e22901, 2017.

\bibitem[Payeur et~al.(2021)Payeur, Guerguiev, Zenke, Richards, and
  Naud]{payeur2020burst}
Alexandre Payeur, Jordan Guerguiev, Friedemann Zenke, Blake Richards, and
  Richard Naud.
\newblock Burst-dependent synaptic plasticity can coordinate learning in
  hierarchical circuits.
\newblock \emph{Nature neuroscience}, 24\penalty0 (5):\penalty0 1546, 2021.

\bibitem[Slotine et~al.(1991)Slotine, Li, et~al.]{slotine1991applied}
Jean-Jacques~E Slotine, Weiping Li, et~al.
\newblock \emph{Applied nonlinear control}, volume 199.
\newblock Prentice hall Englewood Cliffs, NJ, 1991.

\bibitem[Gilra and Gerstner(2017)]{gilra2017predicting}
Aditya Gilra and Wulfram Gerstner.
\newblock Predicting non-linear dynamics by stable local learning in a
  recurrent spiking neural network.
\newblock \emph{Elife}, 6:\penalty0 e28295, 2017.

\bibitem[Den{\`e}ve et~al.(2017)Den{\`e}ve, Alemi, and
  Bourdoukan]{deneve2017brain}
Sophie Den{\`e}ve, Alireza Alemi, and Ralph Bourdoukan.
\newblock The brain as an efficient and robust adaptive learner.
\newblock \emph{Neuron}, 94\penalty0 (5):\penalty0 969--977, 2017.

\bibitem[Alemi et~al.(2018)Alemi, Machens, Den{\`e}ve, and
  Slotine]{alemi2017learning}
Alireza Alemi, Christian Machens, Sophie Den{\`e}ve, and Jean-Jacques Slotine.
\newblock Learning arbitrary dynamics in efficient, balanced spiking networks
  using local plasticity rules.
\newblock \emph{AAAI Conference on Artificial Intelligence (AAAI)}, 2018.

\bibitem[Bourdoukan and Deneve(2015)]{bourdoukan2015enforcing}
Ralph Bourdoukan and Sophie Deneve.
\newblock Enforcing balance allows local supervised learning in spiking
  recurrent networks.
\newblock \emph{Advances in Neural Information Processing Systems},
  28:\penalty0 982--990, 2015.

\bibitem[Podlaski and Machens(2020)]{podlaski2020biological}
William~F Podlaski and Christian~K Machens.
\newblock Biological credit assignment through dynamic inversion of feedforward
  networks.
\newblock \emph{Advances in Neural Information Processing Systems 33}, 2020.

\bibitem[Kohan et~al.(2018)Kohan, Rietman, and Siegelmann]{kohan2018error}
Adam~A Kohan, Edward~A Rietman, and Hava~T Siegelmann.
\newblock Error forward-propagation: Reusing feedforward connections to
  propagate errors in deep learning.
\newblock \emph{arXiv preprint arXiv:1808.03357}, 2018.

\bibitem[Franklin et~al.(2015)Franklin, Powell, and
  Emami-Naeini]{franklin2015feedback}
Gene~F Franklin, J~David Powell, and Abbas Emami-Naeini.
\newblock \emph{Feedback control of dynamic systems}.
\newblock Pearson London, 2015.

\bibitem[Gauss(1809)]{gauss1809theoria}
Carl~Friedrich Gauss.
\newblock \emph{Theoria motus corporum coelestium in sectionibus conicis solem
  ambientium}, volume~7.
\newblock Perthes et Besser, 1809.

\bibitem[Cai et~al.(2019)Cai, Gao, Hou, Chen, Wang, He, Zhang, and
  Wang]{cai2019gram}
Tianle Cai, Ruiqi Gao, Jikai Hou, Siyu Chen, Dong Wang, Di~He, Zhihua Zhang,
  and Liwei Wang.
\newblock A gram-gauss-newton method learning overparameterized deep neural
  networks for regression problems.
\newblock \emph{arXiv preprint arXiv:1905.11675}, 2019.

\bibitem[Urbanczik and Senn(2014)]{urbanczik2014learning}
Robert Urbanczik and Walter Senn.
\newblock Learning by the dendritic prediction of somatic spiking.
\newblock \emph{Neuron}, 81\penalty0 (3):\penalty0 521--528, 2014.

\bibitem[Lyapunov(1992)]{doi:10.1080/00207179208934253}
A.~M. Lyapunov.
\newblock The general problem of the stability of motion.
\newblock \emph{International Journal of Control}, 55\penalty0 (3):\penalty0
  531--534, 1992.
\newblock \doi{10.1080/00207179208934253}.

\bibitem[Hinton et~al.(1995)Hinton, Dayan, Frey, and Neal]{hinton1995wake}
Geoffrey~E Hinton, Peter Dayan, Brendan~J Frey, and Radford~M Neal.
\newblock The" wake-sleep" algorithm for unsupervised neural networks.
\newblock \emph{Science}, 268\penalty0 (5214):\penalty0 1158--1161, 1995.

\bibitem[LeCun(1998)]{lecun1998mnist}
Yann LeCun.
\newblock The mnist database of handwritten digits.
\newblock \emph{http://yann. lecun. com/exdb/mnist/}, 1998.

\bibitem[Xiao et~al.(2017)Xiao, Rasul, and Vollgraf]{xiao2017fashion}
Han Xiao, Kashif Rasul, and Roland Vollgraf.
\newblock Fashion-mnist: a novel image dataset for benchmarking machine
  learning algorithms.
\newblock \emph{arXiv preprint arXiv:1708.07747}, 2017.

\bibitem[N{\o}kland(2016)]{nokland2016direct}
Arild N{\o}kland.
\newblock Direct feedback alignment provides learning in deep neural networks.
\newblock In \emph{Advances in neural information processing systems}, pages
  1037--1045, 2016.

\bibitem[S{\"a}rkk{\"a} and Solin(2019)]{sarkka2019applied}
Simo S{\"a}rkk{\"a} and Arno Solin.
\newblock \emph{Applied stochastic differential equations}, volume~10.
\newblock Cambridge University Press, 2019.

\bibitem[Kingma and Ba(2014)]{kingma2014adam}
Diederik~P Kingma and Jimmy Ba.
\newblock Adam: A method for stochastic optimization.
\newblock \emph{3rd International Conference on Learning Representations,
  {ICLR} 2015, San Diego, CA, USA, May 7-9, 2015, Conference Track
  Proceedings}, 2014.

\bibitem[Ungerleider et~al.(2008)Ungerleider, Galkin, Desimone, and
  Gattass]{ungerleider2008cortical}
Leslie~G Ungerleider, Thelma~W Galkin, Robert Desimone, and Ricardo Gattass.
\newblock Cortical connections of area v4 in the macaque.
\newblock \emph{Cerebral Cortex}, 18\penalty0 (3):\penalty0 477--499, 2008.

\bibitem[Rockland and Van~Hoesen(1994)]{rockland1994direct}
Kathleen~S Rockland and Gary~W Van~Hoesen.
\newblock Direct temporal-occipital feedback connections to striate cortex (v1)
  in the macaque monkey.
\newblock \emph{Cerebral cortex}, 4\penalty0 (3):\penalty0 300--313, 1994.

\bibitem[Richards and Lillicrap(2019)]{richards2019dendritic}
Blake~A Richards and Timothy~P Lillicrap.
\newblock Dendritic solutions to the credit assignment problem.
\newblock \emph{Current opinion in neurobiology}, 54:\penalty0 28--36, 2019.

\bibitem[Larkum(2013)]{larkum2013cellular}
Matthew Larkum.
\newblock A cellular mechanism for cortical associations: an organizing
  principle for the cerebral cortex.
\newblock \emph{Trends in neurosciences}, 36\penalty0 (3):\penalty0 141--151,
  2013.

\bibitem[Spruston(2008)]{spruston2008pyramidal}
Nelson Spruston.
\newblock Pyramidal neurons: dendritic structure and synaptic integration.
\newblock \emph{Nature Reviews Neuroscience}, 9\penalty0 (3):\penalty0
  206--221, 2008.

\bibitem[Xiao et~al.(2020)Xiao, Bennett, Feinberg, Agarwal, and
  Marinella]{xiao2020analog}
T~Patrick Xiao, Christopher~H Bennett, Ben Feinberg, Sapan Agarwal, and
  Matthew~J Marinella.
\newblock Analog architectures for neural network acceleration based on
  non-volatile memory.
\newblock \emph{Applied Physics Reviews}, 7\penalty0 (3):\penalty0 031301,
  2020.

\bibitem[Misra and Saha(2010)]{misra2010artificial}
Janardan Misra and Indranil Saha.
\newblock Artificial neural networks in hardware: A survey of two decades of
  progress.
\newblock \emph{Neurocomputing}, 74\penalty0 (1-3):\penalty0 239--255, 2010.

\bibitem[Moore(1920)]{moore1920reciprocal}
Eliakim~H Moore.
\newblock On the reciprocal of the general algebraic matrix.
\newblock \emph{Bull. Am. Math. Soc.}, 26:\penalty0 394--395, 1920.

\bibitem[Penrose(1955)]{penrose1955generalized}
Roger Penrose.
\newblock A generalized inverse for matrices.
\newblock In \emph{Mathematical proceedings of the Cambridge philosophical
  society}, volume~51, pages 406--413. Cambridge University Press, 1955.

\bibitem[Levenberg(1944)]{levenberg1944method}
Kenneth Levenberg.
\newblock A method for the solution of certain non-linear problems in least
  squares.
\newblock \emph{Quarterly of applied mathematics}, 2\penalty0 (2):\penalty0
  164--168, 1944.

\bibitem[Campbell and Meyer(2009)]{campbell2009generalized}
Stephen~L Campbell and Carl~D Meyer.
\newblock \emph{Generalized inverses of linear transformations}.
\newblock SIAM, 2009.

\bibitem[Schraudolph(2002)]{schraudolph2002fast}
Nicol~N Schraudolph.
\newblock Fast curvature matrix-vector products for second-order gradient
  descent.
\newblock \emph{Neural computation}, 14\penalty0 (7):\penalty0 1723--1738,
  2002.

\bibitem[Zhang et~al.(2019)Zhang, Martens, and Grosse]{zhang2019fast}
Guodong Zhang, James Martens, and Roger~B Grosse.
\newblock Fast convergence of natural gradient descent for over-parameterized
  neural networks.
\newblock In \emph{Advances in Neural Information Processing Systems 32}, pages
  8080--8091, 2019.

\bibitem[Seung(1996)]{seung1996brain}
H~Sebastian Seung.
\newblock How the brain keeps the eyes still.
\newblock \emph{Proceedings of the National Academy of Sciences}, 93\penalty0
  (23):\penalty0 13339--13344, 1996.

\bibitem[Koulakov et~al.(2002)Koulakov, Raghavachari, Kepecs, and
  Lisman]{koulakov2002model}
Alexei~A Koulakov, Sridhar Raghavachari, Adam Kepecs, and John~E Lisman.
\newblock Model for a robust neural integrator.
\newblock \emph{Nature neuroscience}, 5\penalty0 (8):\penalty0 775--782, 2002.

\bibitem[Goldman et~al.(2003)Goldman, Levine, Major, Tank, and
  Seung]{goldman2003robust}
Mark~S Goldman, Joseph~H Levine, Guy Major, David~W Tank, and HS~Seung.
\newblock Robust persistent neural activity in a model integrator with multiple
  hysteretic dendrites per neuron.
\newblock \emph{Cerebral cortex}, 13\penalty0 (11):\penalty0 1185--1195, 2003.

\bibitem[Goldman et~al.(2010)Goldman, Compte, and Wang]{goldman2010neural}
Mark~S Goldman, A~Compte, and Xiao-Jing Wang.
\newblock Neural integrator models.
\newblock \emph{Encyclopedia of neuroscience}, pages 165--178, 2010.

\bibitem[Lim and Goldman(2013)]{lim2013balanced}
Sukbin Lim and Mark~S Goldman.
\newblock Balanced cortical microcircuitry for maintaining information in
  working memory.
\newblock \emph{Nature neuroscience}, 16\penalty0 (9):\penalty0 1306--1314,
  2013.

\bibitem[Bejarano et~al.(2018)Bejarano, Ibarg{\"u}en-Mondrag{\'o}n, and
  G{\'o}mez-Hern{\'a}ndez]{bejarano2018stability}
D~Bejarano, Eduardo Ibarg{\"u}en-Mondrag{\'o}n, and Enith~Amanda
  G{\'o}mez-Hern{\'a}ndez.
\newblock A stability test for non linear systems of ordinary differential
  equations based on the gershgorin circles.
\newblock \emph{Contemporary Engineering Sciences}, 11\penalty0 (91):\penalty0
  4541--4548, 2018.

\bibitem[Martens and Grosse(2015)]{martens2015optimizing}
James Martens and Roger Grosse.
\newblock Optimizing neural networks with kronecker-factored approximate
  curvature.
\newblock In \emph{Proceedings of the 32nd International Conference on Machine
  Learning}, pages 2408--2417, 2015.

\bibitem[Botev et~al.(2017)Botev, Ritter, and Barber]{botev2017practical}
Aleksandar Botev, Hippolyt Ritter, and David Barber.
\newblock Practical gauss-newton optimisation for deep learning.
\newblock In \emph{Proceedings of the 34th International Conference on Machine
  Learning}, pages 557--565. JMLR. org, 2017.

\bibitem[Glorot and Bengio(2010)]{glorot2010understanding}
Xavier Glorot and Yoshua Bengio.
\newblock Understanding the difficulty of training deep feedforward neural
  networks.
\newblock In \emph{Proceedings of the thirteenth international conference on
  artificial intelligence and statistics}, pages 249--256. JMLR Workshop and
  Conference Proceedings, 2010.

\bibitem[Paszke et~al.(2017)Paszke, Gross, Chintala, Chanan, Yang, DeVito, Lin,
  Desmaison, Antiga, and Lerer]{paszke2017automatic}
Adam Paszke, Sam Gross, Soumith Chintala, Gregory Chanan, Edward Yang, Zachary
  DeVito, Zeming Lin, Alban Desmaison, Luca Antiga, and Adam Lerer.
\newblock Automatic differentiation in pytorch.
\newblock 2017.

\bibitem[Bergstra et~al.(2011)Bergstra, Bardenet, Bengio, and
  K{\'e}gl]{bergstra2011algorithms}
James~S Bergstra, R{\'e}mi Bardenet, Yoshua Bengio, and Bal{\'a}zs K{\'e}gl.
\newblock Algorithms for hyper-parameter optimization.
\newblock In \emph{Advances in neural information processing systems}, pages
  2546--2554, 2011.

\bibitem[Bergstra et~al.(2013)Bergstra, Yamins, and Cox]{bergstra2013hyperopt}
James Bergstra, Dan Yamins, and David~D Cox.
\newblock Hyperopt: A python library for optimizing the hyperparameters of
  machine learning algorithms.
\newblock In \emph{Proceedings of the 12th Python in science conference}, pages
  13--20. Citeseer, 2013.

\bibitem[Liaw et~al.(2018)Liaw, Liang, Nishihara, Moritz, Gonzalez, and
  Stoica]{liaw2018tune}
Richard Liaw, Eric Liang, Robert Nishihara, Philipp Moritz, Joseph~E Gonzalez,
  and Ion Stoica.
\newblock Tune: A research platform for distributed model selection and
  training.
\newblock \emph{arXiv preprint arXiv:1807.05118}, 2018.

\bibitem[Paszke et~al.(2019)Paszke, Gross, Massa, Lerer, Bradbury, Chanan,
  Killeen, Lin, Gimelshein, Antiga, Desmaison, Kopf, Yang, DeVito, Raison,
  Tejani, Chilamkurthy, Steiner, Fang, Bai, and Chintala]{pytorch}
Adam Paszke, Sam Gross, Francisco Massa, Adam Lerer, James Bradbury, Gregory
  Chanan, Trevor Killeen, Zeming Lin, Natalia Gimelshein, Luca Antiga, Alban
  Desmaison, Andreas Kopf, Edward Yang, Zachary DeVito, Martin Raison, Alykhan
  Tejani, Sasank Chilamkurthy, Benoit Steiner, Lu~Fang, Junjie Bai, and Soumith
  Chintala.
\newblock Pytorch: An imperative style, high-performance deep learning library.
\newblock In \emph{Advances in Neural Information Processing Systems 32}, pages
  8024--8035. Curran Associates, Inc., 2019.

\bibitem[Silver(2010)]{silver2010neuronal}
R~Angus Silver.
\newblock Neuronal arithmetic.
\newblock \emph{Nature Reviews Neuroscience}, 11\penalty0 (7):\penalty0
  474--489, 2010.

\bibitem[Ferguson and Cardin(2020)]{ferguson2020mechanisms}
Katie~A Ferguson and Jessica~A Cardin.
\newblock Mechanisms underlying gain modulation in the cortex.
\newblock \emph{Nature Reviews Neuroscience}, 21\penalty0 (2):\penalty0 80--92,
  2020.

\bibitem[Larkum et~al.(2004)Larkum, Senn, and L{\"u}scher]{larkum2004top}
Matthew~E Larkum, Walter Senn, and Hans-R L{\"u}scher.
\newblock Top-down dendritic input increases the gain of layer 5 pyramidal
  neurons.
\newblock \emph{Cerebral cortex}, 14\penalty0 (10):\penalty0 1059--1070, 2004.

\bibitem[Naud and Sprekeler(2017)]{naud2017burst}
Richard Naud and Henning Sprekeler.
\newblock Burst ensemble multiplexing: A neural code connecting dendritic
  spikes with microcircuits.
\newblock \emph{bioRxiv}, page 143636, 2017.

\bibitem[Bengio et~al.(2015)Bengio, Lee, Bornschein, Mesnard, and
  Lin]{bengio2015towards}
Yoshua Bengio, Dong-Hyun Lee, Jorg Bornschein, Thomas Mesnard, and Zhouhan Lin.
\newblock Towards biologically plausible deep learning.
\newblock \emph{arXiv preprint arXiv:1502.04156}, 2015.

\end{thebibliography}

%%%%%%%%%%%%%%%%%%%%%%%%%%%%%%
%%% Supplementary material %%%
%%%%%%%%%%%%%%%%%%%%%%%%%%%%%%

\newpage
\setcounter{page}{1}
\setcounter{figure}{0} \renewcommand{\thefigure}{S\arabic{figure}}
\setcounter{theorem}{0} \renewcommand{\thetheorem}{S\arabic{theorem}}
\setcounter{definition}{0} \renewcommand{\thedefinition}{S\arabic{definition}}
\setcounter{condition}{0} \renewcommand{\thecondition}{S\arabic{condition}}
\setcounter{table}{0} \renewcommand{\thetable}{S\arabic{table}}

\appendix

\section*{\Large{Supplementary Material}}
\textbf{Alexander Meulemans$^*$, Matilde Tristany Farinha$^*$, Javier García Ordóñez,}\\
  \textbf{Pau Vilimelis Aceituno, João Sacramento, Benjamin F. Grewe}\\
  \\
  Institute of Neuroinformatics, University of Zürich and ETH Zürich\\
  \texttt{ameulema@ethz.ch}  
\vspace{-1.6cm}
% \tableofcontents
\addcontentsline{toc}{section}{} 
\part{} % Start the appendix part
\parttoc % Insert the appendix TOC

\section{Proofs and extra information for Section \ref{sec:learning_TPDI}: Learning theory} \label{app:learning_theory}

\subsection{Linearized dynamics and fixed points}\label{app:linearized_dyn}
In this section, we linearize the network dynamics around the feedforward voltage levels $\ve{v}_i^-$ (i.e., the equilibrium of the network when no feedback is present) and study the equilibrium points resulting from the feedback input from the controller. 

\paragraph{Notation.} First, we introduce some shorthand notations:
\begin{align}\label{eq:notation_aux}
    \ve{v} &\triangleq [\ve{v}_1^T ... \ve{v}_L^T]^T \\
    Q &\triangleq [Q_1^T ... Q_L^T]^T\\
    f_i(\ve{v}_{i-1}) &\triangleq W_i \phi(\ve{v}_{i-1})\\
    \Delta^-\ve{v}_i &\triangleq \ve{v}_i - \ve{v}_i^- \\
    \Delta \ve{v}_i &\triangleq \ve{v}_i - W_i\phi(\ve{v}_{i-1}) = \ve{v}_i - \ve{v}_i^{\mathrm{ff}} \\
    \Delta \ve{v} &\triangleq [\Delta\ve{v}_1^T ... \Delta\ve{v}_L^T]^T\\
    J_{i,k} &\triangleq \frac{\partial \ve{v}_i}{\partial \ve{v}_k}\bigg\rvert_{\ve{v}_k=\ve{v}_k^-} = \frac{\partial f_i(..f_{k+1}(\ve{v}_k)..)}{\partial \ve{v}_k}\bigg\rvert_{\ve{v}_k=\ve{v}_k^-}\\
    J_i &\triangleq \frac{\partial \ve{r}_L}{\partial \ve{v}_i}\bigg\rvert_{\ve{v}_i=\ve{v}_i^-}\\
    J &\triangleq [J_1 ... J_L] = \frac{\partial \ve{r}_L}{\partial \Delta \ve{v}}\bigg\rvert_{\Delta \ve{v}=0} \label{eq:app_full_jacobian_def}\\
    \boldsymbol{\delta}_L &\triangleq -\lambda\frac{\partial \Lagr}{\partial \ve{r}_L}\bigg\rvert_{\ve{r}_L=\ve{r}_L^{-}} \\
    \ve{r}_L^* &\triangleq \ve{r}_L^- + \boldsymbol{\delta}_L
\end{align}

To investigate the steady state of the network and controller dynamics, we start by proving Lemma \ref{lemma:steady_state_solution_system}, which we restate here for convenience. 

\begin{lemma}\label{lemma_app:steady_state_solution_system}
Assuming stable dynamics, a small target stepsize $\lambda$, and $W_i$ and $Q_i$ fixed, the steady-state solutions of the dynamical systems \eqref{eq:network_dynamics} and \eqref{eq:controller_dynamics} can be approximated by 
    \begin{align}
        \ve{u}_{\mathrm{ss}} = (JQ + \tilde{\alpha} I)^{-1}\ves{\delta}_L + \mathcal{O}(\lambda^2), \quad
        \ve{v}_{\mathrm{ss}} = \ve{v}^{\mathrm{ff}}_{\mathrm{ss}} +  Q (JQ + \tilde{\alpha} I)^{-1}\ves{\delta}_L + \mathcal{O}(\lambda^2),
    \end{align}
    with $J \triangleq \frac{\partial \ve{r}^-_L}{\partial \ve{v}}\big\rvert_{\ve{v}=\ve{v}^-}$ the Jacobian of the network output w.r.t. $\ve{v}$, evaluated at the network equilibrium without feedback, $\ves{\delta}_L$ the output error as defined in \eqref{eq:output_target}, $\ve{v}^{\mathrm{ff}}_{i,ss} = W_i \phi(\ve{v}_{i-1,\mathrm{ss}})$, and $\tilde{\alpha} = \alpha/(1+\alpha k_p)$.
\end{lemma}

\begin{proof}
The proof is ordered as follows: first, we linearize the network dynamics around the feedforward equilibrium of \eqref{eq:feedforward_equilibrium}. Then, we solve the algebraic set of linear equilibrium equations.

With the introduced shorthand notation, we can combine \eqref{eq:network_dynamics} for $i=1, ..., L$ into a single dynamical equation for the whole network:
\begin{align}\label{eq:whole_network_dynamics}
    \tau \frac{\dd \ve{v}}{\dd t} = -\Delta \ve{v} + Q\ve{u}.
\end{align}
% \paragraph{Linearized dynamics.} We assume that $\lambda$ is small and hence $\ve{r}_L^*$ is close to the feedforward output $\ve{r}_L^-$. As the controller now only needs to do small changes to the network activations such that the output $\ve{r}_L$ reaches its target $\ve{r}_L^*$, $\ve{v}_i$ will be close to $\ve{v}_i^-$ during the feedback phase and we can linearize the dynamics around the feedforward activations $\ve{v}_i^-$. Linearizing \eqref{eq:network_dynamics} leads to 
% \begin{align}\label{eq:linearized_network_dynamics}
%     \tau \frac{\dd \ve{v}_i}{\dd t} &=  W_i\phi(\ve{v}_{i-1}) -\ve{v}_i + Q_i\ve{u} \\ 
%     &=  W_i\phi(\ve{v}_{i-1}) -\ve{v}_i^- + \ve{v}_i^- -\ve{v}_i + Q_i\ve{u} \\
%     &=  W_i\phi(\ve{v}_{i-1}) -W_i\phi(\ve{v}_{i-1}^-) + \ve{v}_i^- -\ve{v}_i + Q_i\ve{u} \\
%     &=  J_{i,i-1}(\ve{v}_{i-1} -\ve{v}_{i-1}^-) + \ve{v}_i^- -\ve{v}_i + Q_i\ve{u} + \mathcal{O}(\lambda^2) \\
%     &= -(\Delta^-\ve{v}_i - J_{i,i-1}\Delta^-\ve{v}_{i-1})  + Q_i\ve{u} + \mathcal{O}(\lambda^2) \\
% \end{align}
% Note that the controller dynamics \eqref{eq:controller_dynamics} are already linear. 

By linearizing the dynamics, we can derive the control error $\ve{e}(t) \triangleq \ve{r}_L^* - \ve{r}_L(t)$ as an affine transformation of $\Delta \ve{v}$. First, note that
\begin{align}
    \Delta^-\ve{v}_i &= \ve{v}_i - W_i\phi(\ve{v}_{i-1}) + W_i\phi(\ve{v}_{i-1}) - \ve{v}_i^- \\
    &= \Delta \ve{v}_i + W_i\phi(\ve{v}_{i-1}^- + \Delta^-\ve{v}_{i-1}) - \ve{v}_i^- \\
    &= \Delta \ve{v}_i + J_{i,i-1}\Delta^-\ve{v}_{i-1} + \mathcal{O}(\lambda^2).
\end{align}
By recursion, we have that 
\begin{align}
    \Delta^-\ve{v}_i &= \Delta \ve{v}_i + \sum_{k=1}^{i-1}J_{i,k}\Delta\ve{v}_k + \mathcal{O}(\lambda^2),
\end{align}
with $\Delta \ve{v}_1 = \Delta^- \ve{v}_1 = \ve{v}_1 - \ve{v}_1^-$ because the input to the network is not influenced by the controller, i.e., $\ve{v}_0 = \ve{v}_0^-$. 

The control error given by
\begin{align}\label{eq_app:linearized_control_error}
    \ve{e} &\triangleq \ve{r}_L^* - \ve{r}_L = \ve{r}_L^* - \ve{r}_L^- + \ve{r}_L^- - \ve{r}_L \\
    &= \boldsymbol{\delta}_L - J_L\Delta^-\ve{v}_L + \mathcal{O}(\lambda^2)\\
    &= \boldsymbol{\delta}_L - J_L\big(\Delta \ve{v}_L + \sum_{k=1}^{L-1}J_{L,k}\Delta\ve{v}_k\big) + \mathcal{O}(\lambda^2) \\
    &= \boldsymbol{\delta}_L - \sum_{k=1}^LJ_k\Delta \ve{v}_k + \mathcal{O}(\lambda^2)\\
    &= \boldsymbol{\delta}_L - J \Delta \ve{v} + \mathcal{O}(\lambda^2).
\end{align}

The controller dynamics are given by 
\begin{align} \label{eq_app:controller_dynamics}
    \ve{u}(t) = \ve{u}^{\text{int}}(t) + k_p \ve{e}(t) \\
    \tau_u \ddt \ve{u}^{\text{int}}(t) = \ve{e}(t) - \alpha \ve{u}^{\text{int}}(t)
\end{align}
By differentiating \eqref{eq_app:controller_dynamics} and using $\ve{u}^{\text{int}}=\ve{u} - k_p \ve{e}$ we get the following controller dynamics for $\ve{u}$:
\begin{align} \label{eq_app:total_controller_dynamics}
    \tau_u \ddt \ve{u}(t) = (1+\alpha k_p)\ve{e}(t) + k_p \tau_u \ddt \ve{e}(t) - \alpha \ve{u}(t).
\end{align}

% The linearized control dynamics are given by
% \begin{align} \label{eq:linearized_controller_dynamics}
%     \frac{\dd \ve{u}(t)}{\dd t} &= \ve{e}(t) - \alpha \ve{u}(t) = \boldsymbol{\delta}_L - J \Delta \ve{v}(t) - \alpha \ve{u}(t) + \mathcal{O}(\lambda^2)
% \end{align}
The system of equations \eqref{eq:whole_network_dynamics} and \eqref{eq_app:total_controller_dynamics} can be solved in steady state as follows. From \eqref{eq:whole_network_dynamics} at steady state, we have
\begin{align}
    \Delta \ve{v}_{\mathrm{ss}} = Q\ve{u}_{\mathrm{ss}}.
\end{align}
Substituting $\Delta \ve{v}_{\mathrm{ss}}$ into the steady state of \eqref{eq_app:total_controller_dynamics} while using the linearized control error \eqref{eq_app:linearized_control_error} gives
\begin{align}
    \ve{u}_{\mathrm{ss}} &= (JQ + \tilde{\alpha} I)^{-1}\boldsymbol{\delta}_L + \mathcal{O}(\lambda^2)\\
    \Delta \ve{v}_{\mathrm{ss}} &= Q(JQ + \tilde{\alpha} I)^{-1}\boldsymbol{\delta}_L + \mathcal{O}(\lambda^2) \label{eq:ss_v},
\end{align}
with $\tilde{\alpha} \triangleq \frac{\alpha}{1 + \alpha k_p}$. Using $\ve{v} = \ve{v}^{\text{ff}} + \Delta \ve{v}$ concludes the proof.
\end{proof}
% Note that $J$ has more columns than rows by definition, so $JQ$ will have in general full rank and is hence invertible. 
In the next section, we will investigate how this steady-state solution can result in useful weight updates (plasticity) for the forward weights $W_i$. 

\subsection{DFC approximates Gauss-Newton optimization} \label{app:DFC_GN}
In this subsection, we will investigate the steady state \eqref{eq:ss_v} for $\alpha \rightarrow 0$ and link the resulting weight updates to Gauss-Newton (GN) optimization. Before proving Theorem \ref{theorem:GN_TPDI}, we need to introduce and prove some lemmas. First, we need to show that $\lim_{\alpha \rightarrow 0} \Delta \ve{v}_{\mathrm{ss}} = J^{\dagger}\boldsymbol{\delta}_L$ under Condition \ref{con:Q_GN}, with $J^{\dagger}$ the Moore-Penrose pseudoinverse of $J$ \citep{moore1920reciprocal, penrose1955generalized}. This is done in the following Lemma.
\begin{lemma} \label{lemma:penrose_conditions_Q}
Assuming $J$ has full rank,
\begin{align} \label{eq_app:lemma_s3}
    \lim_{\alpha \rightarrow 0} Q(JQ + \alpha I)^{-1} = J^{\dagger}
\end{align}
iff Condition \ref{con:Q_GN} holds, i.e., $\text{Col}(Q) = \text{Row}(J)$. 
\end{lemma}
\begin{proof}
We begin by stating the Moore-Penrose conditions \citep{penrose1955generalized}:
\begin{condition}\label{con:moore_penrose}
$B=A^{\dagger}$ iff
\begin{enumerate}
    \item $ABA = A$
    \item $BAB=B$
    \item $AB = (AB)^T$
    \item $BA = (BA)^T$
\end{enumerate}
\end{condition}
In this proof, we need to consider 2 general cases: (i) $\J$ has full rank and $Q$ does not and (ii) $Q$ and $\J$ have both full rank. As $\J^T$ and $\Q$ have much more rows than columns, they will almost always be of full rank, however, we consider both cases for completeness. 

In case (i), where $Q$ has lower rank than $J$, we have that $\rank \big(\Q (\J\Q + \alpha I)^{-1}\big) \leq \rank (\Q)$. As $\rank (\J^{\dagger}) = \rank (\J^T)$, $\Q (\J\Q + \alpha I)^{-1}$ can never be the pseudoinverse of $\J$, thereby proving that a necessary condition for \eqref{eq_app:lemma_s3} is that $\rank (Q) \geq \rank (\J)$ (note that this condition is satisfied by Condition \ref{con:Q_GN}). Now, that we showed that it is a necessary condition that $\Q$ is full rank (as $\J$ is full rank by assumption of the lemma) for eq. \eqref{eq_app:lemma_s3} to hold, we proceed with the second case.

In case (ii), where $Q$ and $\J^T$ have both full rank, we need to prove under which conditions on $Q$, $S\triangleq \lim_{\alpha \rightarrow 0} Q(JQ + \alpha I)^{-1}$ is equal to $J^{\dagger}$. As $Q$ and $\J^T$ have both full rank, $\J \Q$ is of full rank and we have 
\begin{align}
    \J S = \lim_{\alpha \rightarrow 0} JQ(JQ + \alpha I)^{-1} = I.
\end{align}

% From Lemma \ref{lemma:v_ss} we have that 
% \begin{align}
%     S = U_{Q} \begin{bmatrix}
% \tilde{J}_1^{-1} \\
% 0
% \end{bmatrix} V_{Q}^T
% \end{align}
% Recall that $J = V_{Q} [\tilde{J}_1 \tilde{J}_2] U_{Q}^T$, hence S is a right-inverse of $J$ for arbitrary $Q$ as shown below.
% \begin{align}
%     JS = V_{Q} [\tilde{J}_1 \tilde{J}_2] U_{Q}^T S = U_{Q} \begin{bmatrix} \tilde{J}_1^{-1} \\ 0 \end{bmatrix} V_{Q}^T = V_{Q} \tilde{J}_1 \tilde{J}_1^{-1} V_{Q}^T = I
% \end{align}
Hence, conditions \ref{con:moore_penrose}.1, \ref{con:moore_penrose}.2 and \ref{con:moore_penrose}.3 are trivially satisfied:
\begin{enumerate}
    \item $JSJ = I J = J$
    \item $SJS = SI = S$
    \item $JS = I = I^T = (JS)^T$
\end{enumerate}
Condition \ref{con:moore_penrose}.4 will only be satisfied under certain constraints on $Q$. We first assume Condition \ref{con:Q_GN} holds to show its sufficiency after which we continue to show its necessity. 

Consider $U_{\J}$ as an orthogonal basis of the column space of $\J^T$. Then, we can write 
\begin{align} \label{eq_app:l3_J_orthogonal_decomp}
    \J = M_{\J} U_{\J}^T
\end{align}
for some full rank square matrix $M_{\J}$. As we assume Condition \ref{con:Q_GN} holds, we can similarly write $\Q$ as
\begin{align}
    \Q = U_{\J} M_{\Q}
\end{align}
for some full rank square matrix $M_{\Q}$. Condition \ref{con:moore_penrose}.4 can now be written as
\begin{align}
    S\J &= \Q (\J \Q)^{-1} \J \\
    &= U_{\J} M_{\Q} ( M_{\J} M_{\Q})^{-1} M_{\J} U_{\J}^T \\
    &= U_{\J}U_{\J}^T \\
    &= (S\J)^T,
\end{align}
showing that S is indeed the pseudoinverse of $\J$ if Condition \ref{con:Q_GN} holds, proving its sufficiency. 

For showing the necessity of Condition \ref{con:Q_GN}, we use a proof by contradiction. We now assume that Condition \ref{con:Q_GN} does not hold and hence the column space of $\Q$ is not equal to that of $\J$. Similar as before, consider $U_{\Q}$ and orthogonal basis of the column space of $\Q$. Furthermore, consider the square orthogonal matrix $\bar{U}_{\J} \triangleq [U_{\J} \tilde{U}_{\J}]$ with $U_{\J}$ as defined in \eqref{eq_app:l3_J_orthogonal_decomp} and $\tilde{U}_{\J}$ orthogonal on $U_{\J}$. We can now decompose $\Q$ into a part inside the column space of $\J^T$ and outside of that column space:
\begin{align}\label{eq_app:l3_Q_decomp}
    Q &= U_{\Q} M_{\Q} \\
    &= \bar{U}_{\J} \bar{U}_{\J}^T U_{\Q} M_{\Q} \\
    &= U_{\J} P_{\Q} + \tilde{U}_{\J} \tilde{P}_{\Q},
\end{align}
with $M_{\Q}$ a square full rank matrix, $P_{\Q} \triangleq U_{\J}^T U_{\Q} M_{\Q}$, and $\tilde{P}_{\Q} \triangleq \tilde{U}_{\J}^T U_{\Q} M_{\Q}$. The first part of \eqref{eq_app:l3_Q_decomp} represents the part of $\Q$ inside the column space of $\J^T$ and the second part represents the part of $\Q$ outside of this column space. For clarity, we assume that $P_{\Q}$ is full rank\footnote{
If $P_Q$ is not of full rank, $JQ$ is not of full rank and hence $\lim_{\alpha \rightarrow 0} (JQ + \alpha I)$ also not. Consequently, $\lim_{\alpha \rightarrow 0} Q(JQ + \alpha I)^{-1}$ will project $Q$ onto something of lower rank, making it impossible for $S$ to approximate $J^{\dagger}$, thereby showing that it is necessary that $P_{\Q}$ is full rank.
}, which is true in all but degenerate cases. Note that $\tilde{P}_{\Q}$ is different from zero, as we assume Condition \ref{con:Q_GN} does not hold in this proof by contradiction. Using this decomposition of $\Q$, we can write $S\J$ used in Condition \ref{con:moore_penrose}.4 as
\begin{align}
    S\J &= \Q (\J \Q)^{-1}\J \\
    &= (U_{\J}P_{\Q} + \tilde{U}_{\J} \tilde{P}_{\Q}) (M_{\J}P_{\Q})^{-1} (M_{\J} U_{\J}^T) \\
    &= U_{\J}U_{\J}^T + \tilde{U}_{\J}\tilde{P}_{\Q} P_{\Q}^{-1}U_{\J}^T.
\end{align}
The first part of the last equation is always symmetric, hence Condition \ref{con:moore_penrose}.4 boils down to the second part being symmetric:
\begin{align}
    \tilde{U}_{\J}\tilde{P}_{\Q} P_{\Q}^{-1}U_{\J}^T &= U_{\J}P_{\Q}^{-T} \tilde{P}_{\Q}^T\tilde{U}_{\J}^T \\
    \Rightarrow \quad \tilde{U}_{\J}^T\tilde{U}_{\J}\tilde{P}_{\Q} P_{\Q}^{-1}U_{\J}^T &= \tilde{U}_{\J}^TU_{\J}P_{\Q}^{-T} \tilde{P}_{\Q}^T\tilde{U}_{\J}^T \\
    \Rightarrow \quad \tilde{P}_{\Q} P_{\Q}^{-1}U_{\J}^T &=0 \\
    \Rightarrow \quad U_{\J}  P_{\Q}^{-T} \tilde{P}_{\Q}^T &=0.
\end{align}
As $U_{\J}$ has a zero-dimensional null space and $P_{\Q}$ is full rank, \ref{con:moore_penrose}.4 can only hold when $\tilde{P}_{\Q}=0$. This contradicts with our initial assumption in this proof by contradiction, stating that Condition \ref{con:Q_GN} does not hold and consequently $\Q$ has components outside of the column space of $\J$, thereby proving that Condition \ref{con:Q_GN} is necessary. 

% \begin{align}
%     SJ &= U_{Q} \begin{bmatrix} I & \tilde{J}_1^{-1}\tilde{J}_2\\ 0 & 0 \end{bmatrix} U_{Q}^T \\
%     (SJ)^T &= U_{Q} \begin{bmatrix} I & 0\\ \tilde{J}_1^{-1}\tilde{J}_2 & 0 \end{bmatrix} U_{Q}^T \label{eq_app:l3_Q_decomp}\\
% \end{align}
% As we assumed that $\tilde{J}^{-1}$ is full rank, Condition \ref{con:moore_penrose}.4 is satisfied iff $\tilde{J}_2 = 0$, thereby concluding the proof. 
\end{proof}

Theorem \ref{theorem:GN_TPDI} states that the updates for $W_i$ in DFC at steady-state align with the updates $W_i$ prescribed by the GN optimization method for a feedforward neural network. We first formalize a \textit{feedforward fully connected neural network}. 

\begin{definition}\label{def:feedforward_nn}
A feedforward fully connected neural network with $L$ layers, input dimension $n_0$, output dimension $n_L$ and hidden layer dimensions $n_i$, $0<i<L$ is defined by the following sequence of mappings: 
\begin{align}
    \ve{r}_i &= \phi(W_i\ve{r}_{i-1}), \quad 0<i<L \\
    \ve{r}_L &= \phi_L(W_i\ve{r}_{L-1}),
\end{align}
with $\phi$ and $\phi_L$ activation functions, $\ve{r}_0$ the input of the network, and $\ve{r}_L$ the output of the network.
\end{definition}
The Lemma below shows that the network dynamics \eqref{eq:network_dynamics} at steady-state are equal to a feedforward neural network corresponding to Definition \ref{def:feedforward_nn} in the absence of feedback. 
\begin{lemma}\label{lemma:equivalence_ff_network}
In the absence of feedback ($\ve{u}(t)=0$), the system dynamics \eqref{eq:network_dynamics} at steady-state are equivalent to a feedforward neural network defined by Definition \ref{def:feedforward_nn}. 
\end{lemma}

\begin{proof}
The proof is trivial upon noting that $Q\ve{u} =0$ without feedback and computing the steady-state of \eqref{eq:network_dynamics} using $\ve{r}_i \triangleq \phi(\ve{v}_i)$.
\end{proof}
Following the notation of eq. \eqref{eq:feedforward_equilibrium}, we denote with $\ve{r}_i^-$ the firing rates of the network in steady-state when feedback is absent, hence corresponding to the activations of a conventional feedforward neural network. The following Lemma investigates what the GN parameter updates are for a feedforward neural network. Later, we then show that the updates at equilibrium of DFC approximate these GN updates. For clarity, we assume that the network has only weights and no biases in all the following theorems and proofs, however, all proofs can be easily extended to comprise both weights and biases. First, we need to introduce some new notation for vectorized matrices.
% \paragraph{Notation for vectorized matrices}
\begin{align}
    \vec{W}_i &\triangleq \text{vec}(W_i)\\
    \bar{W} &\triangleq [\vec{W}_1^T ... \vec{W}_L^T]^T,
\end{align}
where $\text{vec}(W_i)$ denotes the concatenation of the columns of $W_i$ in a column vector.
\begin{lemma}\label{lemma:GN_update_feedforward}
Assuming an $L^2$ task loss and Condition \ref{con:magnitude_r} holds, the Gauss-Newton parameter updates for the weights of a feedforward network defined by Definition \ref{def:feedforward_nn} for a minibatch size of 1 is given by
\begin{align}\label{eq:GN_step}
    \Delta \bar{W}^{GN} = \frac{1}{2\lambda\|\ve{r}\|^2_2} R J^{\dagger}\boldsymbol{\delta}_L,
\end{align}
with $R$ defined in eq. \eqref{eq:R_matrix}.
\end{lemma}
\begin{proof}
Consider the Jacobian of the output w.r.t. the network weights $W$ (in vectorized form as defined above), evaluated at the feedforward activation:
\begin{align}
    J_{\bar{W}} \triangleq \frac{\partial \ve{r}_L}{\partial \bar{W}}\bigg\rvert_{\ve{r}_L=\ve{r}_L^{-}}.
\end{align}
For a minibatch size of 1, the GN update for the parameters $\bar{W}$, assuming an $L^2$ output loss, is given by \citep{gauss1809theoria, levenberg1944method}
\begin{align}\label{eq:GN_step_original}
    \Delta \bar{W}^{GN} = J_{\bar{W}}^{\dagger}(\ve{r}_L^{\text{true}} - \ve{r}^-_L) = \frac{1}{2\lambda} J_{\bar{W}}^{\dagger} \boldsymbol{\delta}_L,
\end{align}
with $\ve{r}_L^{\text{true}}$ the true supervised output (e.g., the class label). The remainder of this proof will manipulate expression \eqref{eq:GN_step_original} in order to reach \eqref{eq:GN_step}. Using $J_{\vec{W}_i} \triangleq \frac{\partial \ve{r}_L}{\partial \vec{W}_i}\big\rvert_{\ve{r}_L=\ve{r}_L^{-}}$, $J_{\bar{W}}$ can be restructured as:
\begin{align}
    J_{\bar{W}} = [J_{\vec{W}_1} ... J_{\vec{W}_L}].
\end{align}
Moreover, $J_{\vec{W}_i} = J_i \frac{\partial \ve{v}_i}{\partial \vec{W}_i}\big\rvert_{\ve{v}_i=\ve{v}_i^{-}}$. Using Kronecker products, this becomes\footnote{The Kronecker product leads to the following equality: $\text{vec}(ABC) = (C^T\otimes A)\text{vec}(B)$. Applied to our situation, this leads to the following equality: $\ve{v}_i = W_i\ve{r}_{i-1} = (\ve{r}_{i-1}^T \otimes I)\vec{W}_i$}
\begin{align}
    J_{\vec{W}_i} = J_i \big((\ve{r}_{i-1}^-)^T \otimes I\big).
\end{align}
Using the structure of $J_{\bar{W}}$, this leads to 
\begin{align}\label{eq:R_matrix}
    J_{\bar{W}} &= J R^T \\
    R^T &\triangleq \begin{bmatrix} 
        (\ve{r}_{0}^-)^T \otimes I & 0 & \hdots & 0 \\
        0 & (\ve{r}_{1}^-)^T \otimes I & \hdots & 0 \\
        \vdots & 0 & \ddots & 0 \\
        0 & \hdots &0 & (\ve{r}_{L-1}^-)^T \otimes I
        \end{bmatrix}
\end{align}
with the dimensions of $I$ such that the equality $J_{\bar{W}} = J R^T$ holds. What remains to be proven is that $J_{\bar{W}}^{\dagger} = \frac{1}{\|\ve{r}\|_2^2}R J^{\dagger}$, assuming that Condition \ref{con:magnitude_r} holds and knowing that $J_{\bar{W}} = J R^T$. To prove this, we need to know under which conditions $(J R^T)^{\dagger} = (R^T)^{\dagger} J^{\dagger}$. The following condition specifies when a pseudoinverse of a matrix product can be factorized \citep{campbell2009generalized}.
\begin{condition}\label{con:pseudo_inverse_factorization}
The Moore-Penrose pseudoinverse of a matrix product $(AB)^{\dagger}$ can be factorized as $(AB)^{\dagger} = B^{\dagger}A^{\dagger}$ if one of the following conditions hold:
\begin{enumerate}
    \item $A$ has orthonormal columns
    \item $B$ has orthonormal rows
    \item $B = A^T$ 
    \item $A$ has all columns linearly independent and $B$ has all rows linearly independent
\end{enumerate}
\end{condition}
In our case, J has more columns than rows, hence conditions \ref{con:pseudo_inverse_factorization}.1 and \ref{con:pseudo_inverse_factorization}.4 can never be satisfied.
Furthermore, condition \ref{con:pseudo_inverse_factorization}.3 does not hold, which leaves us with condition \ref{con:pseudo_inverse_factorization}.2. To investigate whether $R^T$ has orthonormal rows, we compute $R^TR$:
\begin{align}
    R^TR = \begin{bmatrix} \|\ve{r}^-_0\|_2^2 I  & \hdots & 0 \\ \vdots & \ddots & \vdots\\ 0 & \hdots & \|\ve{r}^-_{L-1}\|_2^2 I \end{bmatrix}
\end{align}
If Condition \ref{con:magnitude_r} holds, we have $\|\ve{r}^-_0\|_2^2 = \ldots = \|\ve{r}^-_{L-1}\|_2^2 \triangleq \|\ve{r}\|_2^2$ such that:
\begin{align}
    R^T R = \|\ve{r}\|_2^2 I.
\end{align}
Hence, $\frac{1}{\|\ve{r}\|_2} R^T$ has orthonormal rows iff Condition \ref{con:magnitude_r} holds. From now on, we assume that Condition \ref{con:magnitude_r} holds. Next, we will compute $(R^T)^{\dagger}$. Consider $R^T = U \Sigma V^T$, the singular value decomposition (SVD) of $R^T$. Its pseudoinverse is given by $(R^T)^{\dagger} = V \Sigma^{\dagger} U^T$. As the SVD is unique and $\frac{1}{\|\ve{r}\|_2}R^T$ has orthonormal rows, we can construct the SVD manually:
\begin{align}
    R^T = \underbrace{I}_{=U} \underbrace{\begin{bmatrix} \|\ve{r}\|_2I & 0 \end{bmatrix}}_{=\Sigma} \underbrace{\begin{bmatrix} \frac{1}{\|\ve{r}\|_2} R^T \\ \tilde{V}^T \end{bmatrix}}_{=V^T},
\end{align}
with $\tilde{V}^T$ being a basis orthonormal to $\frac{1}{\|\ve{r}\|_2} R^T$. Hence, we have that 
\begin{align}
    (R^T)^{\dagger} = V \Sigma^{\dagger} U^T = \frac{1}{\|\ve{r}\|_2^2} R.
\end{align}
Putting everything together and assuming that Condition \ref{con:magnitude_r} holds, we have that 
\begin{align}
    \Delta \bar{W}^{GN} = \frac{1}{2\lambda} J_{\bar{W}}^{\dagger} \boldsymbol{\delta}_L = \frac{1}{2\lambda\|\ve{r}\|_2^2}R J^{\dagger} \boldsymbol{\delta}_L,
\end{align}
thereby concluding the proof.
\end{proof}

Now, we are ready to prove Theorem \ref{theorem:GN_TPDI}.

\begin{theorem}[Theorem \ref{theorem:GN_TPDI} in main manuscript]\label{theorem:SM_GN_TPDI}
Assuming Conditions \ref{con:magnitude_r} and \ref{con:Q_GN} hold, $\J$ is full rank, the task loss $\mathcal{L}$ is a $L^2$ loss, and $\lambda,\alpha \rightarrow 0$ 
%in computing $\ve{r}_L^*$ defined by \eqref{eq:output_target}
, then the following steady-state (ss) updates for the forward weights
\begin{align}\label{eq:SM_update_W_ss}
    \Delta W_i = \eta (\ve{v}_{i,\mathrm{ss}} - \ve{v}_{i,\mathrm{ss}}^{\text{ff}})\ve{r}_{i-1,\mathrm{ss}}^T,
\end{align}
with $\eta$ a stepsize parameter, align with the weight updates for $W_i$ for the feedforward network \eqref{eq:feedforward_equilibrium} prescribed by the GN optimization method with a minibatch size of 1.

\end{theorem}
\begin{proof}
Lemma \ref{lemma:equivalence_ff_network} shows that the dynamical network $\eqref{eq:network_dynamics}$ at equilibrium in the absence of feedback is equivalent to a feedforward neural network. Lemma \ref{lemma:GN_update_feedforward} provides the GN update step for such a feedforward network, and hence also for our dynamical network. To prove Theorem \ref{theorem:GN_TPDI}, we have to show that $\lim_{\alpha, \lambda \rightarrow 0} \eta (\ve{v}_{i,\mathrm{ss}}^S - \ve{v}_{i,\mathrm{ss}}^B)\ve{r}_{i-1,\mathrm{ss}}^T$ is aligned with the GN update. First, we combine the updates $\Delta W_i$ into their concatenated vectorized form:
\begin{align}
    \Delta W_i &= \eta\Delta \ve{v}_{i,\mathrm{ss}} \ve{r}_{i-1,\mathrm{ss}}^T \\
    \Delta \vec{W}_i &= (\ve{r}_{i-1,\mathrm{ss}} \otimes I) \Delta \ve{v}_{i,\mathrm{ss}} \\
    \Delta \bar{W} &= \begin{bmatrix} \Delta \vec{W}_1 \\ \vdots \\ \Delta \vec{W}_L \end{bmatrix} = \eta R_{\mathrm{ss}} \Delta \ve{v}_{\mathrm{ss}} \label{eq:cor6_delta_w_bar}
\end{align}
with $R_{\mathrm{ss}}$ as defined in \eqref{eq:R_matrix}, but then with $\ve{r}_{i,\mathrm{ss}}$ instead of $\ve{r}_i^-$. From the linearized dynamics \eqref{eq:ss_v}, combined with Lemma \ref{lemma:penrose_conditions_Q} while assuming $\J$ is of full rank, we have that
\begin{align}\label{eq:cor6_lim_vss}
    \lim_{\alpha \rightarrow 0} \Delta \ve{v}_{\mathrm{ss}} = J^{\dagger}\boldsymbol{\delta}_L + \mathcal{O}(\lambda^2)
\end{align}
iff Condition \ref{con:Q_GN} holds. Taking $\eta = \frac{1}{2\lambda\|\ve{r}\|_2^2}$ and assuming an $L^2$ task loss, we have (using Lemma \ref{lemma:GN_update_feedforward}):
\begin{align}
    \lim_{\alpha \rightarrow 0} \Delta \bar{W} &= \frac{1}{2\lambda\|\ve{r}\|_2^2} R_{\mathrm{ss}} \Delta \ve{v}_{\mathrm{ss}} = \frac{1}{2\lambda\|\ve{r}\|_2^2} R_{\mathrm{ss}} J^{\dagger} 2\lambda(\ve{r}_L^{\text{true}} - \ve{r}^-_L) + \mathcal{O}(\lambda^2) \\
    \lim_{\alpha, \lambda \rightarrow 0} \Delta \bar{W} &= \frac{1}{\|\ve{r}\|_2^2} R  J^{\dagger} (\ve{r}_L^{\text{true}} - \ve{r}^-_L)\label{eq:limit_delta_W_bar}
\end{align}
where we used that $\lim_{\lambda \rightarrow 0} R_{\mathrm{ss}} = R$. By comparing $\lim_{\alpha, \lambda \rightarrow 0} \Delta \bar{W}$ to Lemma \ref{lemma:GN_update_feedforward}, we see that it is equal to the GN update for $\bar{W}$ for a minibatchsize of 1, iff Condition \ref{con:magnitude_r} and \ref{con:Q_GN} hold and for an appropriate learning rate $\eta = \frac{1}{2\lambda\|\ve{r}\|_2^2}$. As $\eta$ is a scalar, we have that for arbitrary $\eta$, $\lim_{\alpha, \lambda \rightarrow 0} \Delta \bar{W}$ is proportional to the Gauss-Newton parameter update, thereby concluding the proof.
\end{proof}
This theorem shows that for tasks with an $L^2$ loss and when Conditions \ref{con:magnitude_r} and \ref{con:Q_GN} hold, DFC approximates Gauss-Newton updates with a minibatch size of 1, which becomes an exact equivalence in the limit of $\alpha$ and $\lambda$ to zero. 

\subsection{DFC uses minimum norm updates}
To remove the need for Condition \ref{con:magnitude_r} and a L2 task loss,\footnote{The Gauss-Newton method can be generalized to other loss functions by using the Generalized Gauss-Newton method \citep{schraudolph2002fast}.} we show that the learning behavior of our network is mathematically sound under more relaxed conditions. Theorem \ref{theorem:MN_DFC} (restated below for convenience) shows that for arbitrary loss functions and without the need for Condition \ref{con:magnitude_r}, our synaptic plasticity rule can be interpreted as a weighted minimum norm (MN) parameter update for reaching the output target, assuming linearized dynamics (which becomes exact in the limit of $\lambda\rightarrow 0$).
\begin{theorem}\label{theorem:SM_MN_DFC}
Assuming stable dynamics, Condition \ref{con:Q_GN} holds and $\lambda, \alpha \rightarrow 0$, the steady-state weight updates \eqref{eq:update_W_ss} are proportional to the weighted MN updates of $W_i$ for letting the feedforward output $\ve{r}^-_L$ reach $\ve{r}_L^*$, i.e., the solution to the following optimization problem:
\begin{align}
    \argmin_{\Delta W_i, i \in [1,..,L]}\quad \sum_{i=1}^{L}\|\ve{r}_{i-1}^{-(m)}\|_2^2\|\Delta W_i\|_F^2 \qquad \text{s.t.}\quad \ve{r}^{-(m+1)}_L = \ve{r}_L^{*(m)},
\end{align}
with $m$ the iteration and $\ve{r}^{-(m+1)}_L$ the network output without feedback after the weight update.
\end{theorem}

\begin{proof}
Rewriting the optimization problem using
\begin{align}
    M = \begin{bmatrix}\label{eq_app:M_matrix}
    \|\ve{r}^-_0\|_2 I & \hdots & 0 \\
    & \ddots & \\
    0 & \hdots & \|\ve{r}^-_{L-1}\|_2 I \end{bmatrix}
\end{align}
and the concatenated vectorized weights $\bar{W}$, we get:
\begin{align}
    \argmin_{\Delta \bar{W}}\quad & \|M\Delta \bar{W}\|_2^2 \\
    \text{s.t.}\qquad &\ve{r}^{-(m+1)}_L = \ve{r}_L^{*(m)} \label{eq:cor6_constraints}
\end{align}
Linearizing the feedforward dynamics around the current parameter values $\bar{W}^{(m)}$ and using Lemma \ref{lemma:equivalence_ff_network}, we get:
\begin{align}\label{eq:cor6_linearized_output}
    \ve{r}_L^{-(m+1)} = \ve{r}_L^{-(m)} + J_{\bar{W}}\Delta\bar{W} + \mathcal{O}(\|\Delta \bar{W}\|_2^2).
\end{align}
We will now assume that $\mathcal{O}(\|\Delta \bar{W}\|_2^2)$ vanishes in the limit of $\lambda \rightarrow 0$, relative to the other terms in this Taylor expansion, and check this assumption at the end of the proof. 
Using \eqref{eq:cor6_linearized_output} to rewrite the constraints \eqref{eq:cor6_constraints}, we get:
\begin{align}
    \ve{r}^{-(m+1)}_L &= \ve{r}_L^{*(m)} \\
    \Leftrightarrow \quad J_{\bar{W}}\Delta\bar{W}  &= \boldsymbol{\delta}_L. %\\
    % J_{\bar{W}}\Delta\bar{W}  &= -\lambda \frac{\partial \Lagr}{\partial \ve{r}_L}\bigg\rvert_{\ve{r}_L=\ve{r}_L^{-}} \triangleq - \vdl
\end{align}
To solve the optimization problem, we construct its Lagrangian:
\begin{align}
    \mathbb{L} = \|M\Delta \bar{W}\|_2^2 + \ves{\mu}^T (J_{\bar{W}} \Delta \bar{W} - \vdl),
\end{align}
with $\ves{\mu}$ the Lagrange multipliers. As this is a convex optimization problem, the optimal solution can be found by solving the following set of equations:
\begin{align}
    \bigg(\frac{\partial \mathbb{L}}{\partial \ves{\mu}}\bigg)^T &= J_{\bar{W}} \Delta \bar{W}^* -\vdl = 0\\
    \bigg(\frac{\partial \mathbb{L}}{\partial\Delta \bar{W}}\bigg)^T &= 2 M^2\Delta \bar{W}^* + J_{\bar{W}}^T\ves{\mu}^* = 0 \quad \Rightarrow \Delta \bar{W}^* = -\frac{1}{2}M^{-2} J_{\bar{W}}^T\ves{\mu}^*\\
    \Rightarrow \quad \ves{\mu}^* &= -2(J_{\bar{W}}M^{-2}J_{\bar{W}}^T)^{-1}\vdl \\
    \Rightarrow \quad \Delta \bar{W}^* &=  M^{-2}J_{\bar{W}}^T\big(J_{\bar{W}} M^{-2} J_{\bar{W}}^T\big)^{-1}\vdl =   M^{-1} \big(J_{\bar{W}}M^{-1}\big)^{\dagger} \vdl,
\end{align}
assuming $J_{\bar{W}}M^{-2}J_{\bar{W}}^T$ is invertible, which is highly likely, as $J_{\bar{W}}$ is a skinny horizontal matrix and $M$ full rank. As $\mathcal{O}(\|\Delta \bar{W}\|_2) = \mathcal{O}(\lambda)$ and $\mathcal{O}(\|\Delta \bar{W}\|^2_2) = \mathcal{O}(\lambda^2)$, the Taylor expansion error $\mathcal{O}(\|\Delta \bar{W}\|^2_2)$ vanishes in the limit of $\lambda \rightarrow 0$, relative to the zeroth and first order terms, thereby confirming our assumption. 

Now, we proceed by factorizing $\big(J_{\bar{W}}M^{-1}\big)^{\dagger}$ into $\J^{\dagger}$ and some other term, similar as in Lemma \ref{lemma:GN_update_feedforward}. First, we note that $J_{\bar{W}} M^{-1} = \J R^T M^{-1}$, with $R^T$ defined in eq. \eqref{eq:R_matrix}. Furthermore, we have that $\big(R^TM^{-1}\big)\big(R^TM^{-1}\big)^T = I$, hence $R^TM^{-1}$ has orthonormal rows. Following Condition \ref{con:pseudo_inverse_factorization}, we can factorize $\big(J_{\bar{W}}M^{-1}\big)^{\dagger}$ as follows:
\begin{align} \label{eq_app:t7_W_mn_i}
    \big(J_{\bar{W}}M^{-1}\big)^{\dagger} &= \big(\J R^T M^{-1}\big)^{\dagger} = \big(R^TM^{-1}\big)^{\dagger} \J^{\dagger} = M^{-1} R \J^{\dagger} \\
   \Rightarrow \quad \Delta \bar{W}^* &= M^{-2} R \J^{\dagger} \vdl \\
   \Rightarrow \quad \Delta W_i^* &= \frac{1}{\|\ve{r}_{i-1}\|_2^2} \big[\J^{\dagger} \vdl\big]_i \ve{r}_{i-1}^T,
\end{align}
with $\big[\J^{\dagger} \vdl\big]_i$ the entries of the vector $\J^{\dagger} \vdl$ corresponding to $\ve{v}_i$.
We used $\big(R^TM^{-1}\big)^{\dagger} = M^{-1} R$, which has a similar derivation as the one used for $\big(R^T\big)^{\dagger}$ in Lemma \ref{lemma:GN_update_feedforward}.

We continue by showing that the weight update at equilibrium of DFC aligns with the MN solutions $\Delta W^*_i$. Adapting \eqref{eq:limit_delta_W_bar} from Theorem \ref{theorem:GN_TPDI} to arbitrary loss functions, assuming \ref{con:Q_GN} holds, and taking a layer-specific learning rate $\eta_i = \frac{1}{\|\ve{r}_{i-1}\|_2^2}$, we get that
\begin{align}\label{eq_app:t7_W_dfc_i}
    \lim_{\alpha, \lambda \rightarrow 0} \Delta W_i &= \frac{1}{\|\ve{r}_{i-1}\|_2^2}\big[\J^{\dagger} \vdl\big]_i \ve{r}_{i-1}^T,
\end{align}
for which we used the same notation as in eq. \eqref{eq_app:t7_W_mn_i} to divide the vector $\J^{\dagger} \vdl$ in layerwise components. As the DFC update \eqref{eq_app:t7_W_dfc_i} is equal to the MN solution \eqref{eq_app:t7_W_mn_i}, we can conclude the proof. Note that because we used layer-specific learning rates $\eta_i = \frac{1}{\|\ve{r}_{i-1}\|_2^2}$ only the layerwise updates $\Delta W_i$ and $\Delta W_i^*$ align, not their concatenated versions $\Delta \bar{W}$ and $\Delta \bar{W}^*$.
\end{proof}

Finally, we will remove Condition \ref{con:Q_GN} and show in Proposition \ref{prop:descent_direction} (here repeated in Proposition \ref{prop_app:descent_direction} for convenience) that the weight updates still follow a descent direction for arbitrary feedback weights. Before proving Proposition \ref{prop:descent_direction}, we need to introduce and prove the following Lemma.

\begin{lemma}\label{lemma:v_ss}
Assuming $\tilde{J}_1$ is full rank,
\begin{align}
    \lim_{\alpha \rightarrow 0} Q(JQ + \alpha I)^{-1} = U_{Q} \begin{bmatrix}
\tilde{J}_1^{-1} \\
0
\end{bmatrix} V_{Q}^T,
\end{align}
with $U_{Q}$, $V_{Q}$ the left and right singular vectors of $\Q$ and $\tilde{J}_1$ as defined as follows: consider $\tilde{J} = V_{Q}^T J U_{Q}$, the linear transformation of $\J$ by the singular vectors of $Q$ which can be written in blockmatrix form $\tilde{J} = [\tilde{J}_1 \tilde{J}_2]$ with $\tilde{J}_1$ a square matrix.
\end{lemma}

\begin{proof}
For the proof, we use the singular value decomposition (SVD) of $Q$ and use it to rewrite $Q(JQ + \tilde{\alpha} I)^{-1}$. The SVD is given by $Q = U_{Q} \Sigma_{Q} V_{Q}^T$, with $V_{Q}$ and $U_{Q}$ square orthogonal matrices and $\Sigma_{Q}$ a rectangular diagonal matrix: 
\begin{align}
    \Sigma_{Q} = \begin{bmatrix} \Sigma_{Q}^D \\ 0 \end{bmatrix},
\end{align}
with $\Sigma_{Q}^D$ a square diagonal matrix, containing the singular values of $Q$. Now, let us define $\tilde{J}$ as 
\begin{align}
    \tilde{J} \triangleq V_{Q}^T J U_{Q},
\end{align}
such that $J = V_{Q} \tilde{J} U_{Q}^T$. $\tilde{J}$ can be structured into $\tilde{J} = [\tilde{J}_1 \tilde{J}_2]$ with $\tilde{J}_1$ a square matrix. Now, we can rewrite $Q(JQ + \alpha I)^{-1}$ as 
\begin{align}
    Q(JQ + \alpha I)^{-1} &= U_{Q} \Sigma_{Q} V_{Q}^T\big(V_{Q} \tilde{J} U_{Q}^TU_{Q} \Sigma_{Q} V_{Q}^T + \alpha I\big)^{-1} \\
    &= U_{Q} \Sigma_{Q} \big(\tilde{J}\Sigma_{Q} + \alpha I)^{-1}V_{Q}^T \\
    &=  U_{Q} \begin{bmatrix} \Sigma_{Q}^D \\ 0 \end{bmatrix} \big(\tilde{J}_1\Sigma_{Q}^D + \alpha I \big)^{-1}V_{Q}^T.
\end{align}
Assuming $\tilde{J}_1$ and $\Sigma_{Q}^D$ to be invertible (i.e., no zero singular values), this leads to: 
\begin{align}
    \lim_{\alpha \rightarrow 0} Q(JQ + \alpha I)^{-1} =U_{Q} \begin{bmatrix}
\tilde{J}_1^{-1} \\
0
\end{bmatrix} V_{Q}^T,
\end{align}
thereby concluding the proof.
\end{proof}
This lemma shows clearly that $\lim_{\alpha \rightarrow 0} Q(JQ + \alpha I)^{-1}$ is a generalized inverse of the forward Jacobian $J$, constrained by the column space of $Q$, which is represented by $U_{Q}$.

\begin{proposition}\label{prop_app:descent_direction}
Assuming stable network dynamics and $\lambda, \alpha \rightarrow 0$, the steady-state weight updates $\Delta W_{i,\mathrm{ss}}$ \eqref{eq:update_W_ss} with a layer-specific learning rate $\eta_i = \eta/\|r_{i-1}\|_2^2$ lie always within 90 degrees of the loss gradient direction. 
% Furthermore, if $\phi$ is linear, training with $\Delta W_{i,\mathrm{ss}}$ and $\eta_i$ converges to the global minimum of the loss function.
% the postsynaptic error signal at steady-state $\ve{v}_{\mathrm{ss}} - \ve{v}^{\text{ff}}_{\mathrm{ss}}$ is always within 90 degrees of the loss gradient w.r.t. $\ve{v}$. Furthermore, if Condition \ref{con:magnitude_r} holds, $\phi$ is linear and $\eta \rightarrow 0$, the training with steady-state updates $\eqref{eq:update_W_ss}$ converges to the global minimum of the loss function. 
\end{proposition}

\begin{proof}
First, we show that the steady-state weight update lies within 90 degrees of the loss gradient, after which we continue to prove convergence for linear networks.
We define $\Delta \ve{v}_{\mathrm{ss}} \triangleq \ve{v}_{\mathrm{ss}} - \ve{v}^{\text{ff}}_{\mathrm{ss}}$, which allows us to rewrite the steady-state update \eqref{eq:update_W_ss} as
\begin{align}
    \Delta \bar{W}_{\mathrm{ss}} = \eta M^{-2} R_{\mathrm{ss}} \Delta \ve{v}_{\mathrm{ss}},
\end{align}
where we use the vectorized notation, $R_{\mathrm{ss}}$ defined in eq. \eqref{eq:R_matrix} with steady-state activations, and $M$ defined in eq. \eqref{eq_app:M_matrix} to represent the layer-specific learning rate $\eta_i = \eta/\|r_{i-1}\|_2^2$. 
Using Lemma \ref{lemma:steady_state_solution_system} and \ref{lemma:v_ss}, we have that 
\begin{align}
    \lim_{\alpha \rightarrow 0} \Delta \ve{v}_{\mathrm{ss}} = U_{Q} \begin{bmatrix}
\tilde{J}_1^{-1} \\
0
\end{bmatrix} V_{Q}^T \ves{\delta}_L.
\end{align}

Using the same vectorized notation, the negative gradient of the loss with respect to the network weights (i.e., the BP updates) can be written as:
\begin{align}
    \Delta \bar{W}^{BP} = \eta R \J^T \ves{\delta}_L.
\end{align}

To show that the steady-state weight update lies within 90 degrees of the loss gradient, we prove that their inner product is greater than zero in the limit of $\lambda,\alpha \rightarrow 0$:
\begin{align}
    \lim_{\lambda,\alpha \rightarrow 0} \langle \Delta \bar{W}^{BP}, \Delta \bar{W}_{\mathrm{ss}}\rangle  &= \lim_{\lambda,\alpha \rightarrow 0}\eta^2 \ves{\delta}_L^T \J R^T M^{-2} R_{\mathrm{ss}} \Delta \ve{v}_{\mathrm{ss}} \\
    &= \lim_{\lambda \rightarrow 0}\eta^2 \ves{\delta}_L^T V_{Q} \tilde{J} \begin{bmatrix}
\tilde{J}_1^{-1} \\
0
\end{bmatrix} V_{Q}^T \ves{\delta}_L \\
&= \lim_{\lambda \rightarrow 0}\eta^2 \ves{\delta}_L^T \ves{\delta}_L >0,
\end{align}
where we used that $\lim_{\lambda \rightarrow 0} R^T M^{-2} R_{\mathrm{ss}} = I$ and took $\eta \propto 1/\lambda$ to have a limit different from zero, as $\ves{\delta}_L$ scales with $\lambda$.

% Now that we proved that $\Delta \bar{W}_{\mathrm{ss}}$ is within 90 degrees of the gradient, we proceed to prove convergence for linear networks. As in the limit of $\lambda\rightarrow 0$, the $\Delta \bar{W}_{\mathrm{ss}}$ are of infinitesimal size and within 90 degrees of the loss gradient, we are guaranteed that each update will decrease the loss if it is not at a fixed point. What remains to be shown is that the 
\end{proof}

\subsection{An intuitive interpretation of Condition \ref{con:Q_GN}} \label{app:intuitive_condition2}
In the previous sections, we showed that Condition \ref{con:Q_GN} is needed to enable precise CA through GN or MN optimization. Here, we discuss a more intuitive interpretation of why Condition \ref{con:Q_GN} is needed. 

DFC has three main components that influence the feedback signals given to each neuron. First, we have the network dynamics \eqref{eq:network_dynamics} (here repeated for convenience).
\begin{align}
    \tau_v \ddt \ve{v}_i(t) &= -\ve{v}_i(t) + W_i\phi\big(\ve{v}_{i-1}(t)\big) + Q_i\ve{u}(t) \quad 1\leq i \leq L.
\end{align}
The first two terms $-\ve{v}_i(t) + W_i\phi\big(\ve{v}_{i-1}(t)\big)$ pull the neural activation $\ve{v}_i$ close to its feedforward compartment $\vff_i$, while the third term $Q_i\ve{u}(t)$ provides an extra push such that the network output is driven to its target. This interplay between pulling and pushing is important, as it makes sure that $\ve{v}_i$ and $\vff_i$ remain as close as possible together, while driving the output towards its target.

Second, we have the feedback weights $Q$. As $Q$ is of dimensions $\sum_{i=1}^L n_i \times n_L$, with $n_i$ the layer size, it has always much more rows than columns. Hence, the few but long columns of $Q$ can be seen as the `modes' that the controller $\ve{u}$ can use to change network activations $\ve{v}$. Due to the low-dimensionality of $\ve{u}$ compared to $\ve{v}$, $Q\ve{u}$ cannot change the activations $\ve{v}$ in arbitrary directions, but is constrained by the column space of $Q$, i.e., the `modes' of $Q$.

Third, we have the feedback controller, that through its own dynamics, combined with the network dynamics \eqref{eq:network_dynamics} and $Q$, selects an `optimal' configuration for $\ve{u}$, i.e., $\ve{u}_{\sss} = (JQ)^{-1}\ves{\delta}_L$, that selects and weights the different modes (columns) of $Q$ to push the output to its target in the `most efficient manner'. 

To make `most efficient manner' more concrete, we need to define the \textit{nullspace} of the network. As the dimension of $\ve{v}$ is much bigger than the output dimension, there exist changes in activation $\Delta \ve{v}$ that do not result in a change of output $\Delta \ve{r}_L$, because they lie in the nullspace of the network. In a linearized network, this is reflected by the network Jacobian $J$, as we have that $\Delta \ve{r}_L = J\Delta \ve{v}$. As J is of dimensions $n_L \times \sum_{i=1}^L n_i$, it has many more columns than rows and thus a non-zero nullspace. When $\Delta \ve{v}$ lies inside the nullspace of $J$, it will result in $\Delta \ve{r}_L=0$. Now, if the column space of $Q$ overlaps partially with the nullspace of $J$, one could make $\ve{u}$, and hence $\Delta \ve{v} = Q\ve{u}$, arbitrarily big, while still making sure that the output is pushed exactly to its target, when the `arbitrarily big' parts of $\Delta \ve{v}$ lie inside the nullspace of $J$ and hence do not influence $\ve{r}_L$. Importantly, the feedback controller combined with the network dynamics ensure that this does not happen, as $\ve{u}_{\sss} = (JQ)^{-1}\ves{\delta}_L$ selects the smallest possible $\ve{u}_{\sss}$ to push the output to its target.

However, when the column space of $Q$ partially overlaps with the nullspace of $J$, there will inevitably be parts of $\Delta \ve{v}$ that lie inside the nullspace of $J$, even though the controller selects the smallest possible $\ve{u}_{\sss}$. This can easily be seen as in general, each column of $Q$ overlaps partially with the nullspace of $J$, so $\Delta \ve{v} = Q\ve{u}$, which is a linear combination of the columns of $Q$, will also overlap partially with the nullspace of $J$. This is where Condition \ref{con:Q_GN} comes into play. 

Condition \ref{con:Q_GN} states that the column space of $Q$ is equal to the row space of $J$. When this condition is fulfilled, the column space of $Q$ does not overlap with the nullspace of $J$. Hence, all the feedback $Q\ve{u}$ produces a change in the network output and no unnecessary changes in activations $\Delta \ve{v}$ take place. With Condition \ref{con:Q_GN} satisfied, the occurring changes in activations $\Delta \ve{v}$ are MN, as they lie fully in the row-space of $J$ and push the output exactly to its target. This interpretation lies at the basis of Theorem \ref{theorem:MN_DFC} and is also an important part of Theorem \ref{theorem:GN_TPDI}.

\subsection{Gauss-Newton optimization with a mini-batch size of 1} \label{app:GN_minibatch1}
In this section, we review the GN optimization method and discuss the unique properties that arise when a mini-batch size of 1 is taken. 
\paragraph{Review of GN optimization.}
Gauss-Newton (GN) optimization is an iterative optimization method used for non-linear regression problems with an $L^2$ output loss, defined as follows:
\begin{align}\label{eq:least_squares_regression}
    \argmin_{\ve{\theta}} \quad \mathcal{L} &= \frac{1}{2}\sum_{b=1}^B \|\ves{\delta}^{(b)}\|_2^2\\
    \ves{\delta}^{(b)} &\triangleq \ve{y}^{(b)}-\ve{r}^{(b)},
\end{align}
with B the minibatch size, $\ves{\delta}$ the regression error, $\ve{r}$ the model output, and $\ve{y}$ the corresponding regression target. 
There exist two main derivations of the GN optimization method: (i) through an approximation of the Newton-Raphson method and (ii) through linearizing the parametric model that is being optimized. We focus on the latter, as this derivation is closely connected to DFC. 

GN is an iterative optimization method and hence aims to find a parameter update $\Delta \ves{\theta}$ that leads to a lower regression loss: 
\begin{align}
    \ves{\theta}^{(m+1)} \leftarrow \ves{\theta}^{(m)} + \Delta \ves{\theta},
\end{align}
with $m$ indicating the iteration number. The end goal of the optimization scheme is to find a local minimum of $\mathcal{L}$, hence, finding $\ves{\theta}^*$ for which holds
\begin{align}
    0 & \overset{!}{=} \frac{\partial \mathcal{L}}{\partial \vt} \Big\rvert_{\vt = \vt^*}^T = J_{\vt}^T\vd \label{eq:4GNLossDerivative}\\
    J_{\vt} &\triangleq \frac{\partial \ve{r}}{\partial \vt}\Big\rvert_{\vt = \vt^*},
\end{align}
with $\vd$ and $\vr$ the concatenation of all $\vd^{(b)}$ and $\vr^{(b)}$, respectively. To obtain a closed-form expression for $\vt^*$ that fulfills eq. \eqref{eq:4GNLossDerivative} approximately, one can make a first-order Taylor approximation of the parameterize model around the current parameter setting $\vt^{(m)}$:
\begin{align}
    \ve{r}^{(m+1)} &\approx \ve{r}^{(m)} + J_{\vt}\Delta \vt \label{eq_app:gn_linearized_model}\\
    \ves{\delta}^{(m+1)} &= \ve{y} - \vr^{(m+1)} \approx \vd^{(m)} - J_{\vt}\Delta \vt.
\end{align}
Filling this approximation into eq. \eqref{eq:4GNLossDerivative}, we get:
\begin{align}
    &\frac{\partial \mathcal{L}}{\partial \vt} \approx J_{\vt}^T\big(\vd^{(m)} - J_{\vt}\Delta \vt\big) = 0\\
    \Leftrightarrow \quad &J_{\vt}^T J_{\vt} \Delta \vt = J_{\vt}^T \vd^{(m)}. \label{eq:4GNleastsquares1}
\end{align}
In an under-parameterized setting, i.e., the dimension of $\vd$ is bigger than the dimension of $\vt$, $J_{\vt}^T J_{\vt}$ can be interpreted as an approximation of the loss Hessian matrix used in the Newton-Raphson method and is known as the \textit{Gauss-Newton curvature matrix}. In the under-parameterized setting, $J_{\vt}^T J_{\vt}$ is invertible, leading to the update
\begin{align}
    \Delta \vt &= \big(J_{\vt}^T J_{\vt}\big)^{-1} J_{\vt}^T \vd^{(m)}\\
    \Delta \vt &= J_{\vt}^{\dagger}\vd^{(m)} \label{eq_app:gn_pinv_solution},
\end{align}
with $J_{\vt}^{\dagger}$ the Moore-Penrose pseudoinverse of $J_{\vt}$. In the under-parameterized setting, eq. \eqref{eq:4GNleastsquares1} can be interpreted as a linear least-squares regression for finding a parameter update $\Delta \vt$ that results in a least-squares solution on the linearized parametric model \eqref{eq_app:gn_linearized_model}. Until now we considered the under-parameterized case. However, DFC is related to GN optimization with a mini-batch size of 1, which concerns the over-parameterized case.

\paragraph{GN optimization with a mini-batch size of 1.}
When the minibatch size $B=1$, the dimension of $\vd$ is smaller than the dimension of $\vt$ in neural networks, hence we need to consider the over-parameterized case of GN \citep{cai2019gram, zhang2019fast}. Now, the matrix $J_{\vt}^T J_{\vt}$ is not of full rank and hence an infinite amount of solutions exist for eq. \eqref{eq:4GNleastsquares1}. To enforce a unique solution for the parameter update $\Delta \vt$, a common approach is to take the MN solution, i.e., the smallest possible solution $\Delta \vt$ that satisfies \eqref{eq:4GNleastsquares1}. Using the MN properties of the Moore-Penrose pseudoinverse, this results in:
\begin{align}\label{eq_app:gn_mn_solution}
    \Delta \vt &= J_{\vt}^{\dagger}\vd^{(m)}.
\end{align}
Although the solution has the same form as before \eqref{eq_app:gn_pinv_solution}, its interpretation is fundamentally different, as we did not use a linear least-squares solution, but a MN solution instead. In the under-parameterized case considered before, the parameter update $\Delta \vt$ will not be able to drive $\vd^{(m+1)}$ to zero (in the linearized model). In the over-parameterized case however, there exist many solutions for $\Delta \vt$ that drive $\vd^{(m+1)}$ exactly to zero, and GN picks the MN solution \eqref{eq_app:gn_mn_solution}. 

With this interpretation, we see clearly the connection to DFC. In DFC, the feedback controller drives the network activations (i.e., finds an `activation update' solution) such that the output of the network reaches its target $\ve{r}_L^*$ (i.e., the error $\vd^{(m+1)}$ is driven to zero). When Condition \ref{con:Q_GN} holds, this activation update solution is the MN solution. Furthermore, when Condition \ref{con:magnitude_r} holds, this MN activation update results also in a MN parameter update $\Delta W_{i,\sss}$ \eqref{eq:update_W_ss}. 

% \paragraph{Unique properties of GN with mini-batch size 1.}
% As discussed in the previous paragraph, GN optimization with a mini-batch size of 1 can best be interpreted as \textit{minimum-norm} optimization. This minimum-norm optimization has not yet been studied extensively and many of its unique properties need to be discovered still. In this paragraph, we elaborate on one particular property of minimum-norm optimization, which motivated us to include the nonlinearity $\phi$ into the update for $W_i$ \eqref{eq:W_dynamics}.

% In minimum-norm optimization, we use a parameter update $\Delta \vt$ that exactly pushes the output error $\vd^{(m+1)}$ to zero for the considered datasample. However, in networks with saturating nonlinearities, it can be that a very large parameter update is needed to exactly get the output error to zero. In general, such a large parameter update will interfere with other parameter updates, as it is tuned towards one single data sample and hence `overfits' on it. We found that a good heuristic to prevent this overfitting issue is to include the nonlinearity $\phi$ into the parameter update $\eqref{eq:W_dynamics}$, such that saturated neurons do not update their weights, which leads to a significant increase in performance. Note that this is only one possible heuristic, and other options also exist to prevent large parameter updates, such as trust-region approaches commonly used in combination with GN optimization \citep{levenberg1944method, marquardt1963algorithm, sorensen1982newton}.

\paragraph{DFC updates with larger batch sizes.} 
For computational efficiency, we average the DFC updates over a minibatch size bigger than 1. However, this averaging over a minibatch is distinct from doing Gauss-Newton optimization on a minibatch. The GN iteration with minibatch size $B$ is given by
\begin{align}\label{eq:gn_update_batch}
    \Delta \bar{W}^{GN} &= \lim_{\gamma \xrightarrow{} 0}\Big[\sum_{b=1}^B J_{\bar{W}}^{(b)T}J_{\bar{W}}^{(b)} + \gamma I \Big]^{-1} \Big[\sum_{b=1}^B J_{\bar{W}}^{(b)T}\ves{\delta}_L^{(b)}\Big],
\end{align}
with $J_{\bar{W}}^{(b)}$ the Jacobian of the output w.r.t. the concatenated weights $\bar{W}$ for batch sample $b$, and $\gamma$ a damping parameter. Note that we accumulate the GN curvature $J_{\bar{W}}^{(b)T}J_{\bar{W}}^{(b)}$ over all minibatch samples before taking the inverse. 

When the assumptions of Theorem \ref{theorem:GN_TPDI} hold, the DFC updates with a minibatch size $B$ can be written by
\begin{align}\label{eq:dfc_update_batch}
    \Delta \bar{W} &= \lim_{\gamma \xrightarrow{} 0} \sum_{b=1}^B \Big[\big( J_{\bar{W}}^{(b)T}J_{\bar{W}}^{(b)} + \gamma I \big)^{-1} J_{\bar{W}}^{(b)T}\ves{\delta}_L^{(b)}\Big] \\
    &= \sum_{b=1}^B \Big[J_{\bar{W}}^{(b)\dagger}\ves{\delta}_L^{(b)}\Big].
\end{align}
For $B=1$, the DFC update \eqref{eq:dfc_update_batch} overlaps with the GN update \eqref{eq:gn_update_batch}. However, for $B>1$ these are not equal anymore, due to the order of summation and inversion being reversed.

\subsection{Effects of the nonlinearity $\phi$ in the weight update}\label{app:linear_undamped_lr}

In this section, we study in detail the experimental consequences of using the nonlinear learning rule \eqref{section:plasticity_rules} instead of the linear learning rule \eqref{eq:update_W_ss}. First, we investigate the case where the assumptions in Theorem \ref{theorem:MN_DFC} are perfectly satisfied and then we investigate the more realistic case where the assumptions are not perfectly satisfied.

When considering the ideal case where Condition \ref{con:Q_GN} is perfectly satisfied and in the limit of $\lambda$ and $\alpha$ to zero, MN updates \eqref{eq_app:MN_updates} are obtained if the linear learning rule is used, and the following updates are obtained when the nonlinear learning rule is used:
\begin{align}
    \Delta \bar{W} = RDJ^T(JJ^T)^{-1}\boldsymbol{\delta}_L,
\end{align}
with $D$ a diagonal matrix with $\partial \phi(v_j) / \partial (v_j)$ for each neuron in the network on its diagonal and $R$ as defined in eq. \eqref{eq_app:MN_updates}. For this ideal case, we performed experiments on MNIST comparing the linear to the nonlinear learning rules, and obtained a test error of $2.18^{\pm 0.14}\%$ and $2.11^{\pm 0.10}\%$, respectively. These experiments demonstrate that for this ideal case the nonlinear learning rule \eqref{section:plasticity_rules} has no significant benefit over the linear learning rule \eqref{eq:update_W_ss}.

On the other hand, to investigate the influence of the nonlinear learning rule for the practical case where Condition \ref{con:Q_GN} is not perfectly satisfied, we performed a new hyperparameter search on MNIST for DFC-SSA with the linear learning rule \eqref{eq:update_W_ss}. This resulted in a test error of $5.28^{\pm 0.14}\%$. Comparing this result with the corresponding test performance in Table \ref{tab:test_results} ($2.29^{\pm 0.097}\%$ test error), we conclude that DFC benefits from the introduction of the chosen nonlinearities in the learning rule \eqref{section:plasticity_rules}, as the results improve significantly. Hence, we can infer that this increase in performance is due to the way the introduction of the nonlinearity in the learning rule compensates for when the feedback weights do not perfectly satisfy Condition \ref{con:Q_GN}.

Lastly, to investigate where this performance gap originates from, we performed another toy experiment similar to Fig. \ref{fig:toy_experiment_tanh} (see Fig. \ref{fig:toy_experiment_lienar_vs_nonlinear}) for the linear versus nonlinear learning rule in DFC. The new results show that the updates resulting from the nonlinear learning rule are much better aligned with the MN and GN updates, compared to the linear learning rule, explaining its better performance. Overall, we conclude that introducing the nonlinearity in the learning rule, which prevents saturated neurons from updating their weights, is a useful heuristic to improve the alignment of DFC with the MN and GN updates and consequently improve its performance, when Condition \ref{con:Q_GN} is not perfectly satisfied.

\begin{figure}[h!]
    \centering
    \includegraphics[width=0.97\linewidth]{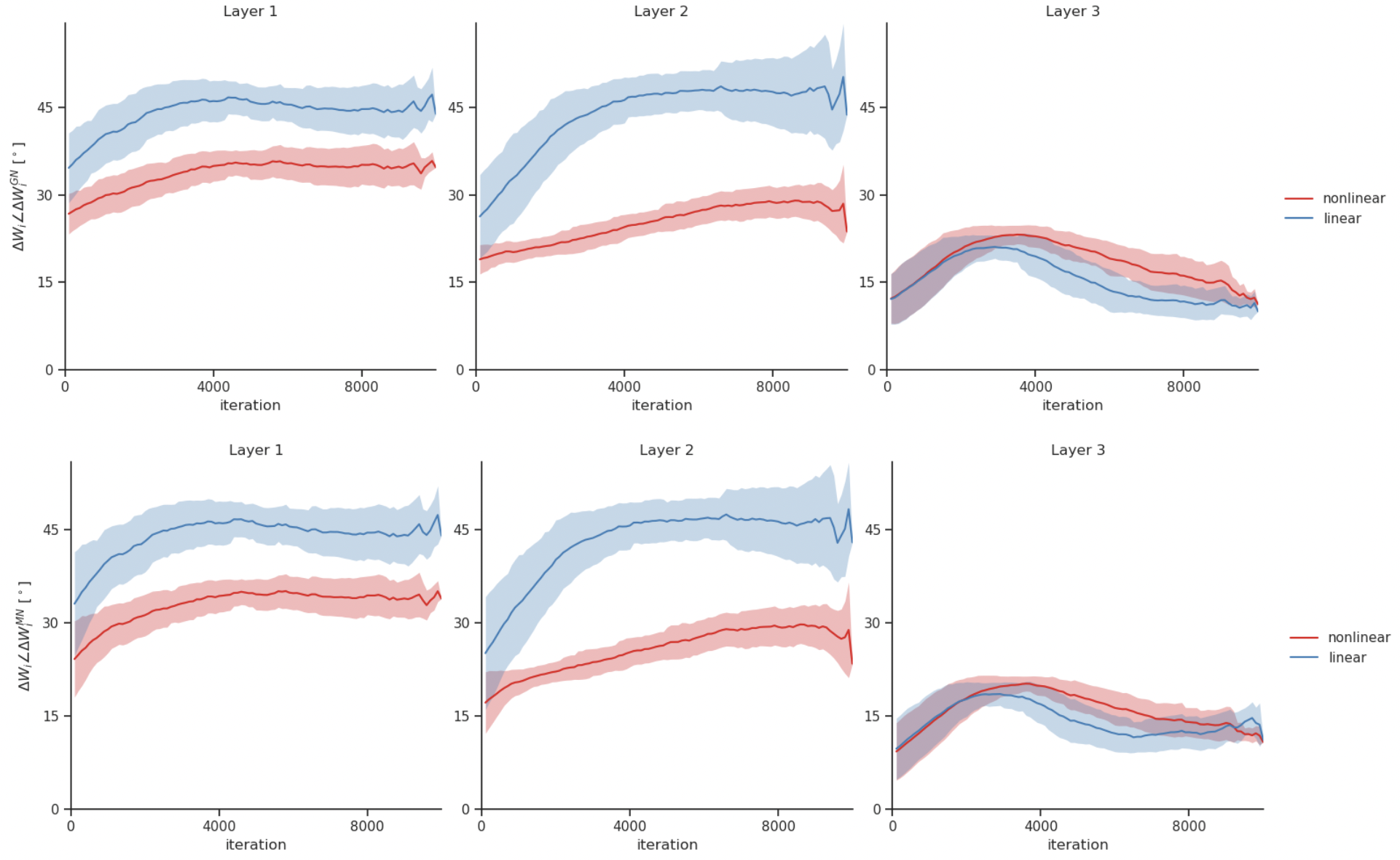}
  \vspace{-0.2cm}
  \caption[Results]{Layer-wise comparison of the angle alignment between the DFC updates and the GN and MN updates, for the linear and nonlinear learning rule variants, when performing nonlinear student-teacher regression task.} \label{fig:toy_experiment_lienar_vs_nonlinear}
\end{figure}

\subsection{Relation between continuous DFC weight updates and steady-state DFC weight updates} \label{app:relation_ss_cont}

All developed learning theory in section \ref{sec:learning_TPDI} considers an update $\Delta W_i$ at the steady-state of the network \eqref{eq:network_dynamics} and controller \eqref{eq:controller_dynamics} dynamics instead of a continuous update as defined in \eqref{eq:W_dynamics}. Fig. \ref{fig:toy_experiment_tanh}F shows that the accumulated continuous updates \eqref{eq:W_dynamics} of DFC align well with the analytical steady-state updates. Here, we indicate why this steady-state update is a good approximation of the accumulated continuous updates \eqref{eq:W_dynamics}. We consider two main reasons: (i) the network and controller dynamics settle quickly to their steady-state and (ii) when the dynamics are not settled yet, they oscillate around the steady-state, thereby causing oscillations to cancel each other out approximately. 

Addressing the first reason, consider an input that is presented to the network from time $T_1$ until $T_2$ and that the network and controller dynamics converge at $T_{ss}<T_2$. The change in weight prescribed by \eqref{eq:W_dynamics} is then equal to 
\begin{align}\label{eq_app:integral_W_dynamics}
    \int_{T_1}^{T_2}\dd W_i = \int_{T_1}^{T_{ss}}\dd W_i + \frac{T_2 - T_{ss}}{\tau_W} \big(\phi(\ve{v}_{i,\mathrm{ss}}) - \phi(\ve{v}_{i,\mathrm{ss}}^{\mathrm{ff}})\big)\ve{r}_{i-1,\mathrm{ss}}^T,
\end{align}
where we assumed a separation of timescales, i.e., $W_i$ is considered constant over the integration interval. If the dynamics settle quickly, i.e., $T_{ss}-T_1 \ll T_2 - T_{ss}$, the second term in the right-hand-side (RHS) dominates and we have that $\int_{T_1}^{T_2}\dd W_i \approx \eta \big(\phi(\ve{v}_{i,\mathrm{ss}}) - \phi(\ve{v}_{i,\mathrm{ss}}^{\mathrm{ff}})\big)\ve{r}_{i-1,\mathrm{ss}}^T$ with $\eta = (T_2 - T_{ss})/\tau_W$. 

Addressing the second reason, note that Fig. \ref{fig:DFC_schematics}D shows qualitatively that the dynamics oscillate briefly around the steady-state before settling. The first integral term in the RHS of eq. \eqref{eq_app:integral_W_dynamics} hence integrates over oscillating variables around the steady-state value. These oscillations will partially cancel each other out in the integration, causing $\int_{T_1}^{T_{ss}}\dd W_i$ to be approximately equal to $\eta \big(\phi(\ve{v}_{i,\mathrm{ss}}) - \phi(\ve{v}_{i,\mathrm{ss}}^{\mathrm{ff}})\big)\ve{r}_{i-1,\mathrm{ss}}^T$ for some $\eta$.

\subsection{DFC is compatible with various controller types}\label{app:controller_type}

Throughout the main manuscript, we focused on a proportional-integral (PI) controller. However, the DFC framework is compatible with various other controller types. In the following, we show that the results on learning theory (Section \ref{sec:learning_TPDI} can be generalized to pure integral control, pure proportional control or any combination thereof with derivative control added. Note that for each new controller type, a new stability analysis is needed and whether the feedback learning rule is still compatible with the controller also needs to be checked, which we leave to future work. 
% Many of our theoretical results can also hold for proportional control instead of proportional integral (PI) control, which makes the DFC framework more generally applicable to different controller types. In the following, we present this new theoretical analysis regarding proportional control only.

\subsubsection{Pure integral control}
For pure integral control, the steady-state solutions of Lemma \ref{lemma:steady_state_solution_system} still apply, with $\tilde{\alpha}=\alpha$. Hence, all learning theory results of Section \ref{sec:learning_TPDI} directly apply to this case. Furthermore, Proposition \ref{prop:local_stability} and Theorem \ref{theorem:fb_learning_simplified} are already designed for pure integral control.

\subsubsection{Pure proportional control}

By making a first-order Taylor approximation of the network dynamics with only proportional control (putting $K_I=0$ in eq. \eqref{eq:controller_dynamics}), we obtain the following steady-state solution:
\begin{align}
    \mathbf{v}_{\mathrm{ss}} = \mathbf{v}_{\mathrm{ss}}^{\mathrm{ff}} + (QJ + \frac{1}{k_p}  I)^{-1}Q\boldsymbol{\delta}_L + \mathcal{O}(\lambda^2),
\end{align}
with $k_p$ the proportional gain, i.e., $\mathbf{u} = k_p \mathbf{e}$. When comparing with eq. \eqref{eq:l1_dynamical_inversion}, we see that the steady-state solution only with proportional control has a similar structure to the one with PI control, with $\tilde{\alpha}$ replaced by $\frac{1}{k_p}$. Note, however, that the damped inverse of $QJ$ is taken instead of $JQ$. By using a similar proof technique as for Lemma \ref{lemma:penrose_conditions_Q}, we can show that $\lim_{k_p\rightarrow \infty} (QJ + \frac{1}{k_p}  I)^{-1}Q = J^{\dagger}$ iff Condition \ref{con:Q_GN} holds.\footnote{We leave the proof as an exercise for the interested reader. The proof follows the same approach as Lemma \ref{lemma:penrose_conditions_Q} and uses l'Hôpital's rule for taking the correct limit of $k\rightarrow \infty$ .} Consequently, Theorems \ref{theorem:GN_TPDI} and \ref{theorem:MN_DFC} and Proposition \ref{prop:descent_direction} hold also for proportional control, if the limit of $\alpha$ to zero is replaced by the limit of $k_p$ to infinity. Furthermore, the main intuitions of Theorem \ref{theorem:fb_learning_simplified} for training the feedback can be applied to proportional control, given that one finds a way to keep the network stable during the initial feedback weights training phase.

Despite these theoretical similarities between proportional and PI control in DFC, there are some significant practical differences. First, for finite $k_p$ in proportional control, there is always a residual error that remains and hence the output target will never be exactly reached. Second, if noise is present in the network, it gets amplified by the same factor $k_p$. Hence, using a high $k_p$ in proportional control makes the controlled network sensitive to noise. Adding an integral control component can alleviate these issues by replacing the need for a large gain, $k_p$, with the need for a good integrator circuit (i.e., low $\alpha$) \citep{franklin2015feedback}, for which a rich neuroscience literature exists \citep{seung1996brain, koulakov2002model, goldman2003robust, goldman2010neural, lim2013balanced}. This way, we can use a smaller gain, $k_p$, without increasing the residual error and consequently make the network less sensitive to noise. This is also interesting from a biological point of view since biological networks are considered to be substantially noisy.

\subsubsection{Adding derivative control}
Proportional, integral or proportional-integral control can be combined with derivative control. As the derivative term disappears at the steady state, the steady-state solutions of Lemma \ref{lemma:steady_state_solution_system} remain unaltered and the learning theory results can be directly applied. However, note that the derivative control term can significantly impact the stability and feedback learning of the network.

\section{Proofs and extra information for Section \ref{sec:stability_TPDI}: Stability of DFC} \label{app:stability}

\subsection{Stability analysis with instantaneous system dynamics} \label{app:stability_inst_system}

In this section, we first derive eq. \eqref{eq:linearization_u}, which corresponds to the dynamics of the controller obtained when assuming a separation of timescales between the controller and the network ($\tau_u \gg \tau_v$), and only having integrative control ($k_p=0$).

Let us recall that $\ve{v}_{\mathrm{ss}}$ and $\ve{v}^{-}$ are the steady-state solutions of the dynamical system \eqref{eq:network_dynamics} with and without control, respectively. Now, by linearizing the network dynamics \eqref{eq:network_dynamics} around the feedforward steady-state, $\ve{v}^{-}$, we can write
    \begin{align}\label{eq:r_L_approx}
        \ve{r}_{L} = \ve{r}_{L}^{-} + JQ\ve{u} + \mathcal{O}(\lambda^2),
    \end{align}
    with $J \triangleq \left.\left[\frac{\partial \ve{r}^-_L}{\partial \ve{v}_1},...,\frac{\partial \ve{r}^-_L}{\partial \ve{v}_L}\right]\right\rvert_{\ve{v}=\ve{v}^-}$ the network Jacobian evaluated at the steady state, and where we dropped the time dependence $(t)$ for conciseness.
    
Taking into account the results of equations \eqref{eq:output_target} and \eqref{eq:r_L_approx}, the control error can then be rewritten as
\begin{align}
    \ve{e}= \ve{r}_{L}^{*} - \ve{r}_{L} = (\ve{r}_{L}^{*} - \ve{r}_{L}^{-}) -  (\ve{r}_{L} - \ve{r}_{L}^{-})= \delta_{L} - JQ\ve{u} + \mathcal{O}(\lambda^2).
\end{align}
Consequently, eq. \eqref{eq:linearization_u} follows:
\begin{align}\label{eq:J_u}
    \tau_u\dot{\ve{u}} = \ve{e} - \alpha\ve{u} = \delta_{L} - JQ\ve{u} - \alpha\ve{u} + \mathcal{O}(\lambda^2) = \delta_{L} - (JQ + \alpha I)\ve{u} + \mathcal{O}(\lambda^2),
\end{align}
where we changed the notation $\ddt \ve{u}$ to $\dot{\ve{u}}$ for conciseness. Now, we continue by proving Proposition \ref{prop:local_stability}, restated below for convenience.

\begin{proposition}[Proposition \ref{prop:local_stability} in main manuscript]\label{prop:SM_local_stability}
    Assuming $\tau_u \gg \tau_v$ and $k_p=0$, the network and controller dynamics are locally asymptotically stable around its equilibrium iff Condition \ref{con:local_stability} holds.
\end{proposition}

\begin{proof}
    Assuming instantaneous system dynamics ($\tau_u \gg \tau_v$), then the stability of the system is entirely up to the controller dynamics. To prove that the system's equilibrium is locally asymptotically stable, we need to guarantee that the Jacobian associated to the controller dynamics evaluated at its steady-state solution, $\ve{v}_{\mathrm{ss}}$, has only eigenvalues with a strictly negative real part \citep{doi:10.1080/00207179208934253}. This Jacobian can be obtained in a similar fashion to that of eq. \eqref{eq:linearization_u}, and is given by
    \begin{align}\label{eq:inst_dyn_stability}
        J_{u} = - (J_{\mathrm{ss}}Q+\alpha I), \,\, \text{ with } \,\, J_{\mathrm{ss}} \triangleq \left.\left[\frac{\partial \ve{r}^-_L}{\partial \ve{v}_1},...,\frac{\partial \ve{r}^-_L}{\partial \ve{v}_L}\right]\right\rvert_{\ve{v}=\ve{v}_{\mathrm{ss}}}.
    \end{align}
    
    To fulfill the local asymptotic stability condition, $J_{\mathrm{ss}}Q + \alpha I$ can only have eigenvalues with strictly positive real parts. 
    As adding $\alpha I$ to $J_{\mathrm{ss}}Q$ results in adding $\alpha$ to the eigenvalues of $J_{\mathrm{ss}}Q$, the local asymptotic stability condition requires that the real parts of the eigenvalues of $J_{\mathrm{ss}}Q$ are all greater than $-\alpha$, corresponding to Condition \ref{con:local_stability}.
    % Now, note that $J_{\mathrm{ss}}Q$ can be diagonalizable as $P\Lambda P^{-1}$, for some invertible matrix $P$ ($\Lambda$ a diagonal matrix containing the eigenvalues of $J_{\mathrm{ss}}Q$). This way, we can also rewrite $J_{\mathrm{ss}}Q + \alpha I$ as $P(\Lambda + \alpha I)P^{-1}$. Condition \ref{con:local_stability} ensures that the real parts of the eigenvalues of $J_{\mathrm{ss}}Q$ are all greater than $-\alpha$ implying that $\Lambda + \alpha I$ can only have strictly positive values.
\end{proof}

\subsection{Stability of the full system}

In this section, we derive a concise representation of the full dynamics of the network \eqref{eq:network_dynamics} and controller dynamics \eqref{eq:controller_dynamics} in the general case where the timescale of the neuronal dynamics, $\tau_v$, is not negligible and we have proportional control ($k_p>0$). Proposition \ref{prop:SM_full_system_local_stability} provides the abstract conditions that guarantee local asymptotic stability of the steady states of the full dynamical system.

\begin{proposition}\label{prop:SM_full_system_local_stability}
    The network and controller dynamics are locally asymptotically stable around its equilibrium iff the following matrix has strictly negative eigenvalues:
    
    \begin{align}\label{eq_app:a_pi}
        A_{PI} =
        \begin{bmatrix}
            -\frac{1}{\tau_v}(I-\hat{J}_{\mathrm{ss}}) & \frac{1}{\tau_v}(I-\hat{J}_{\mathrm{ss}})Q \\
            J_{\mathrm{ss}}\big((\frac{k_p}{\tau_v}-\frac{1}{\tilde{\tau}_u})I - \frac{k_p}{\tau_v}\hat{J}_{\mathrm{ss}}\big) & -\frac{k_p}{\tau_v}J_{\mathrm{ss}}(I-\hat{J}_{\mathrm{ss}})Q - \frac{\tilde{\alpha}}{\tilde{\tau}_u}I
        \end{bmatrix}
    \end{align}
    
    with $\tilde{\alpha} = \frac{\alpha}{1 + k_p \alpha}$, $\tilde{\tau}_u = \frac{\alpha}{1 + k_p \alpha}$, $J_{\mathrm{ss}} = \frac{\partial \ve{r}_L}{\partial \ve{v}}\big\rvert_{\ve{v}=\ve{v}_{\mathrm{ss}}}$ and $\hat{J}_{\mathrm{ss}}$ defined in equations \eqref{eq_app:J_hat} and \eqref{eq:rewrite_network_dyn_v1}.
    
    \begin{proof}
        Recall that the controller is given by \eqref{eq:controller_dynamics}
        \begin{align}
            \ve{u} = \ve{u}^{\text{int}} + k_p\ve{e},
        \end{align}
        where $\tau_u\dot{\ve{u}}^{\text{int}} = \ve{e} - \alpha\ve{u}^{\text{int}}$.
        Then, the controller dynamics can be written as
        \begin{align}\label{eq:controller_dyn}
        \begin{split}
            \dot{\ve{u}} &= \frac{1}{\tau_u}\big(\ve{e} - \alpha(\ve{u}-k_p\ve{e})\big) + k_p\dot{\ve{e}} \\
            \Leftrightarrow \tau_u\dot{\ve{u}} &= (1 + \alpha k_p)\ve{e} + k_p\tau_u\dot{\ve{e}} - \alpha\ve{u} \\
            \Leftrightarrow \tilde{\tau}_u\dot{\ve{u}} &= \ve{e} + k_p\tilde{\tau}_u\dot{\ve{e}} - \tilde{\alpha}\ve{u},
        \end{split}
        \end{align}
        with $\tilde{\tau}_u=\tau_u/(1+\alpha k_p)$ and $\tilde{\alpha}=\alpha/(1+\alpha k_p)$.
        
        Recall that the network dynamics are given by \eqref{eq:network_dynamics}
        \begin{align}
            \tau_v\dot{\ve{v}}_{i} = -\ve{v}_{i} + W_{i}\phi(\ve{v}_{i-1}) + Q_{i}\ve{u} = -\Delta\ve{v}_{i} + Q_{i}\ve{u},
        \end{align}
        with $\Delta\ve{v}_{i} = \ve{v}_{i}-W_{i}\phi(\ve{v}_{i-1})$. Which allows us to write
        \begin{align}
            \Delta\dot{\ve{v}}_{i} = \dot{\ve{v}}_{i} - W_{i}D(\ve{v}_{i-1})\dot{\ve{v}}_{i-1}, \,\, \text{ with } \,\, D(\ve{v}_{i-1}) \triangleq \frac{\partial\phi(\ve{v}_{i-1})}{\partial\ve{v}_{i-1}}\big\rvert_{\ve{v}_{i-1}=\ve{v}_{i-1}(t)}.
        \end{align}
        We can now obtain the network dynamics in terms of $\Delta\dot{\ve{v}}$ as
        \begin{align}
        \begin{split}
            \tau_v\Delta\dot{\ve{v}}_{i} &= -\Delta\ve{v}_{i} + Q_{i}\ve{u} - W_{i}D(\ve{v}_{i-1})(\Delta\ve{v}_{i-1} + Q_{i-1}\ve{u}) \\
            &= -\Delta\ve{v}_{i} + W_{i}D(\ve{v}_{i-1})\Delta\ve{v}_{i-1} + (Q_{i} - W_{i}\Delta\ve{v}_{i-1}Q_{i-1})\ve{u},
        \end{split}
        \end{align}
        which for the entire system is
        \begin{align}\label{eq:rewrite_network_dyn}
            \tau_v\Delta\dot{\ve{v}} = -(I-\hat{J}(\ve{v}))\Delta\ve{v} + (I-\hat{J}(\ve{v}))Q\ve{u},
        \end{align}
        with 
        \begin{align} \label{eq_app:J_hat}
        \hat{J}(\ve{v}) \triangleq
        \begin{bmatrix}
            0 & 0 & 0 & \ldots & 0 \\
            W_{2}D(\ve{v}_{1}) & 0 & 0 & \ldots & 0 \\
            0 & W_{3}D(\ve{v}_{2}) & 0 & \ldots & 0 \\
            \vdots & \ddots & \ddots & \ddots & \vdots \\
            0 & \ldots & 0 & W_{L}D(\ve{v}_{L-1}) & 0 
            \end{bmatrix}
        \end{align}
        
        Let us now proceed to linearize the network and controller dynamical systems by defining
        \begin{align}\label{eq:useful_deltas}
            \tilde{\Delta}\ve{v}=\Delta\ve{v}-\Delta\ve{v}_{\mathrm{ss}} \,\,\, \text{  and  } \,\,\, \tilde{\Delta}\ve{u}=\ve{u}-\ve{u}_{\mathrm{ss}},
        \end{align}
        with $\ve{u}_{\sss}$ and $\Delta \vv_{\sss}$ the steady states of the network and controller (c.f. Lemma \ref{lemma:steady_state_solution_system}).
        With a first order Taylor approximation, we can write $\ve{r}_{L} \approx \ve{r}_{L,\mathrm{ss}} + J_{\mathrm{ss}}\tilde{\Delta}\ve{v}$, where $J_{\mathrm{ss}}$ is as in eq. \eqref{eq:inst_dyn_stability}.
        
        The controller dynamics \eqref{eq:controller_dyn} can now be rewritten as
        \begin{align}\label{eq:rewrite_controller_dyn}
            \tilde{\tau}_u\tilde{\Delta}\dot{\ve{u}} = \ve{r}^{*}_{L} - \ve{r}_{L,\mathrm{ss}} - J_{\mathrm{ss}}\tilde{\Delta}\ve{v} - \tilde{\tau}_u k_{p}J_{\mathrm{ss}}\tilde{\Delta}\dot{\ve{v}} - \tilde{\alpha}\ve{u}_{\mathrm{ss}} - \tilde{\alpha}\tilde{\Delta}\ve{u}.
        \end{align}
        When the network and the controller are at equilibrium, eq. \eqref{eq:controller_dyn} yields
        \begin{align}
            0 = \ve{e}_{\mathrm{ss}} - \tilde{\alpha}\ve{u}_{\mathrm{ss}} = \ve{r}^{*}_{L} - \ve{r}_{L,\mathrm{ss}} - \tilde{\alpha}\ve{u}_{\mathrm{ss}},
        \end{align}
        and we can rewrite eq. \eqref{eq:rewrite_controller_dyn} as
        \begin{align}\label{eq:rewrite_controller_dyn_v1}
            \tilde{\tau}_u\tilde{\Delta}\dot{\ve{u}} = - J_{\mathrm{ss}}\tilde{\Delta}\ve{v} - \tilde{\tau}_u k_{p}J_{\mathrm{ss}}\tilde{\Delta}\dot{\ve{v}} - \tilde{\alpha}\tilde{\Delta}\ve{u}.
        \end{align}
        
        Once again, when the network and the controller are at equilibrium, incorporating the definitions in \eqref{eq:useful_deltas} into eq. \eqref{eq:rewrite_network_dyn}, it follows that
        \begin{align}\label{eq:rewrite_network_dyn_v1}
            \tau_v\tilde{\Delta}\dot{\ve{v}} = -(I-\hat{J}_{\mathrm{ss}})(\Delta\ve{v}_{\mathrm{ss}} + \tilde{\Delta}\ve{v}) + (I-\hat{J}_{\mathrm{ss}})Q(\ve{u}_{\mathrm{ss}} + \tilde{\Delta}\ve{u}),\,\,\,\hat{J}_{\mathrm{ss}}\triangleq\hat{J}(\ve{v}_{\mathrm{ss}}).
        \end{align}
        At steady-state, eq. \eqref{eq:rewrite_network_dyn} yields
        \begin{align}
            0 = -(I-\hat{J}_{\mathrm{ss}})\Delta\ve{v}_{\mathrm{ss}} + (I-\hat{J}_{\mathrm{ss}})Q\ve{u}_{\mathrm{ss}},
        \end{align}
        which allows us to rewrite eq. \eqref{eq:rewrite_network_dyn_v1} as
        \begin{align}\label{eq:network_dyn_final}
            \tau_v\tilde{\Delta}\dot{\ve{v}} = -(I-\hat{J}_{\mathrm{ss}})\tilde{\Delta}\ve{v} + (I-\hat{J}_{\mathrm{ss}})Q\tilde{\Delta}\ve{u}.
        \end{align}
        Using the results from eq. \eqref{eq:network_dyn_final}, we can write eq. \eqref{eq:rewrite_controller_dyn_v1} as
        \begin{align}\label{eq:controller_dyn_final}
            \tilde{\tau}_u\tilde{\Delta}\dot{\ve{u}} = - J_{\mathrm{ss}}\tilde{\Delta}\ve{v} - \frac{\tilde{\tau}_u}{\tau_v} k_{p}J_{\mathrm{ss}}\Big(-(I-\hat{J}_{\mathrm{ss}})\tilde{\Delta}\ve{v} + (I-\hat{J}_{\mathrm{ss}})Q\tilde{\Delta}\ve{u}\Big) - \tilde{\alpha}\tilde{\Delta}\ve{u}.
        \end{align}
        Finally, as $\tilde{\Delta}\dot{\ve{v}} = \Delta\dot{\ve{v}} = \dot{\ve{v}}$ and $\tilde{\Delta}\dot{\ve{u}}=\dot{\ve{u}}$ \eqref{eq:useful_deltas}, this allows us to to infer local stability results for the full system dynamics by looking into the dynamics of $\tilde{\Delta}\dot{\ve{v}}$ and $\tilde{\Delta}\dot{\ve{u}}$ around the steady state:
        \begin{align}
            \begin{bmatrix}
                \tilde{\Delta}\dot{\ve{v}} \\
                \tilde{\Delta}\dot{\ve{u}}
            \end{bmatrix}
            = 
            \begin{bmatrix}
                -\frac{1}{\tau_v}(I-\hat{J}_{\mathrm{ss}}) & \frac{1}{\tau_v}(I-\hat{J}_{\mathrm{ss}})Q \\
                J_{\mathrm{ss}}\big((\frac{k_p}{\tau_v}-\frac{1}{\tilde{\tau}_u})I - \frac{k_p}{\tau_v}\hat{J}_{\mathrm{ss}}\big) & -\frac{k_p}{\tau_v}J_{\mathrm{ss}}(I-\hat{J}_{\mathrm{ss}})Q - \frac{\tilde{\alpha}}{\tilde{\tau}_u}I
            \end{bmatrix}
            \begin{bmatrix}
                \tilde{\Delta}\ve{v} \\
                \tilde{\Delta}\ve{u}
            \end{bmatrix}
            \triangleq A_{PI}
            \begin{bmatrix}
                \tilde{\Delta}\ve{v} \\
                \tilde{\Delta}\ve{u}
            \end{bmatrix}
        \end{align}
        Now, to guarantee local asymptotic stability of the system's equilibrium, then the eigenvalues of $A_{PI}$ must have strictly negative real parts \citep{doi:10.1080/00207179208934253}.
    \end{proof}
    
\end{proposition}
The current form of the system matrix $A_{PI}$ provides no straightforward intuition on finding interpretable conditions for the feedback weights $Q$ such that local stability is reached. One can apply Gershgoring's circle theorem to infer sufficient restrictions on $J$ and $Q$ to ensure local asymptotic stability \cite{bejarano2018stability}. However, the resulting conditions are too conservative and do not provide intuition in which types of feedback learning rules are needed to ensure stability.

\subsection{Toy experiments for relation of Condition \ref{con:local_stability} and full system dynamics}
To investigate whether Condition \ref{con:local_stability} is a good proxy for the local stability of the actual dynamics, we plotted the maximum real parts of the eigenvalues of $JQ + \alpha I$ (Condition \ref{con:local_stability}, see Fig. \ref{fig:stability_toy_exp}.a) and of $A_{PI}$ (the actual dynamics, see eq. \eqref{eq_app:a_pi} and Fig. \ref{fig:stability_toy_exp}.b). We used the same student-teacher regression setting and configuration as in the toy experiments of Fig. \ref{fig:toy_experiment_tanh}. 

Fig. \ref{fig:stability_toy_exp} shows that the maximum real part of the eigenvalues of $A_{PI}$ follow the same trend as the eigenvalues of $JQ + \alpha I$. Although they differ in exact value, both eigenvalue trajectories are slowly decreasing during training and are strictly negative, thereby indicating that Condition \ref{con:local_stability} is a good proxy for the local stability of the actual dynamics.

When we only consider leaky integral control ($k_p=0$, see Fig. \ref{fig:stability_toy_exp}.c), the dynamics become unstable during late training, highlighting that adding proportional control is crucial for the stability of the dynamics. Interestingly, training the feedback weights (blue curve) does not help in this case for making the system stable, on the contrary, it pushes the network to become unstable more quickly. These leaky integral control dynamics are equal to the simplified dynamics used in Condition \ref{con:local_stability} in the limit of $\tau_v/\tau_u \rightarrow 0$, which are stable (see Fig. \ref{fig:stability_toy_exp}.a). Hence, slower network dynamics (finite time constant $\tau_v$) cause the leaky integral control to become unstable, due to a communication delay between controller and network, causing unstable oscillations. For this toy experiment, we used $\tau_v/\tau_u = 0.2$.

\begin{figure}[h!]
    \centering
    \includegraphics[width=0.9\linewidth]{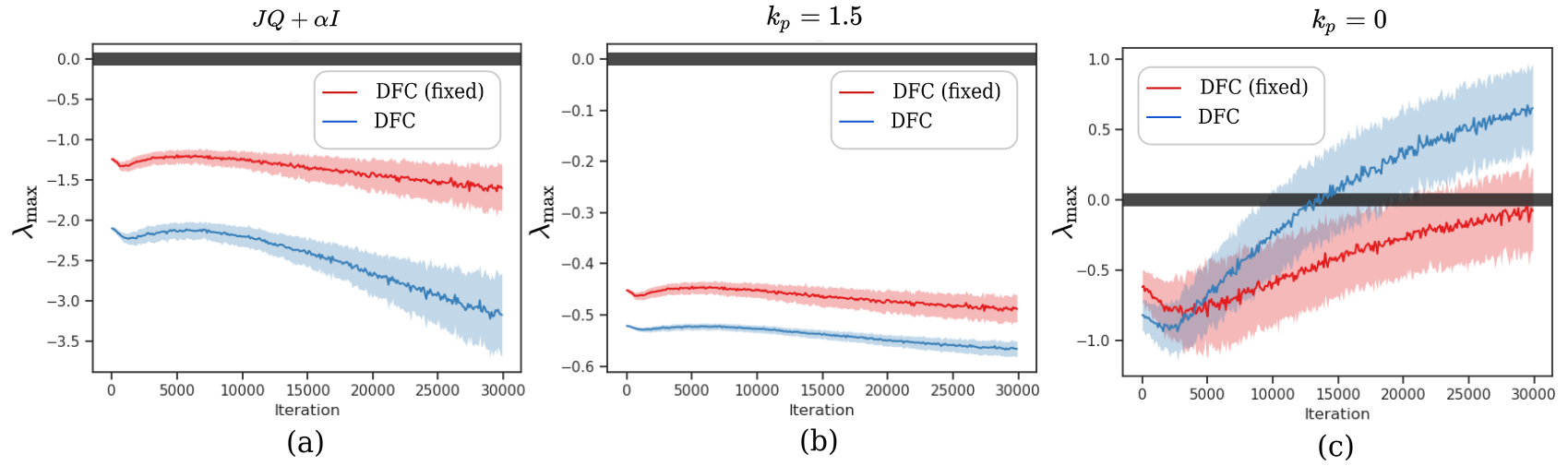}
    \caption{Visualization of the maximum real parts of the eigenvalues of the system matrices during the student teacher regression for DFC with the same configuration as in Fig. \ref{fig:toy_experiment_tanh}. (a) The maximum real part of the eigenvalues of $JQ+\alpha I$, hence representing Condition \ref{con:local_stability}. (b) The maximum real part of the eigenvalues of the system dynamics matrix $A_{PI}$ \eqref{eq_app:a_pi}, representing the actual local stability of DFC. (c) The maximum real part of the eigenvalues of the system dynamics matrix $A_{PI}$ \eqref{eq_app:a_pi} without proportional control ($k_p=0$), representing the local stability of DFC with only integral control. The black bar represents the local stability threshold.}
    \label{fig:stability_toy_exp}
\end{figure}

\section{Proofs and extra information for Section \ref{sec:fb_learning}: Learning the feedback weights} \label{app:fb_learning}
\subsection{Learning the feedback weights in a sleep phase}
In this section, we show that the plasticity rule for the apical synapses \eqref{eq:Q_dynamics} drives the feedback weights to fulfill Conditions \ref{con:Q_GN} and \ref{con:local_stability}. We first sketch an intuitive argument on why the feedback learning rule works. Next, we state the full Theorem and give its proof.
\subsubsection{Intuition behind the feedback learning rule}
Inspired by the Weight Mirroring Method \citep{akrout2019deep} we use white noise in the network to carry information from the network Jacobian $J$ into the output $\ve{r}_L$. To gain intuition, we first consider a normal feedforward neural network 
\begin{align}
    \ve{r}_i^- = \phi(\ve{v}^-_i) = \phi(W_i \ve{r}_{i-1}^-), \quad 1\leq i \leq L.
\end{align}
Now, we perturb each layer's pre-nonlinearity activation with white noise $\ves{\xi}_i$ and propagate the perturbations forward:
\begin{align}
    \tilde{\ve{v}}_i^- = W_i \phi(\tilde{\ve{v}}_{i-1}^-) + \sigma \ves{\xi}_i, \quad 1\leq i \leq L,
\end{align}
with $\tilde{\ve{r}}_0^- = \ve{r}_0^-$. For small $\sigma$, a first-order Taylor approximation of the perturbed output gives
\begin{align}
    \tilde{\ve{r}}_L^- = \ve{r}_L^- + \sigma J \ves{\xi} + \mathcal{O}(\sigma^2),
\end{align}
with $\ves{\xi}$ the concatenated vector of all $\ves{\xi}_i$. If we now take as output target $\ve{r}_L^* = \ve{r}_L^-$, the output error is equal to 
\begin{align}
    \ve{e} = \ve{r}_L^* -\tilde{\ve{r}}_L^- = -\sigma J \ves{\xi} + \mathcal{O}(\sigma^2).
\end{align}
We now define a simple learning rule $\Delta Q = -\sigma \ves{\xi} \ve{e}^T - \beta Q$, which is a simple anti-Hebbian rule with as presynaptic signal the output error $\ve{e}$ and as postsynaptic signal the noise inside the neuron $\sigma\ves{\xi}$, combined with weight decay. If $\ves{\xi}$ is uncorrelated white noise with correlation matrix equal to the identity matrix, the expectation of this learning rule is
\begin{align}
    \mathbb{E}[\Delta Q] = \sigma^2 J^T - \beta Q.
\end{align}
We see that this learning rule lets the feedback weights $Q$ align with the transpose of the networks Jacobian $J$ and has a weight decay term to prevent $Q$ from diverging. 

There are three important differences between this simplified intuitive argumentation for the feedback learning rule and the actual feedback learning rule \eqref{eq:Q_dynamics} used by DFC, which we will address in the next section.
\begin{enumerate}
    \item DFC considers continuous dynamics, hence, the incorporation of noise leads to stochastic differential equations (SDEs) instead of a discrete perturbation of the network layers. The handling of SDEs needs special care, leading to the use of exponentially filtered white noise instead of purely white noise (see next section).
    \item The postsynaptic part of the feedback learning rule \eqref{eq:Q_dynamics} for DFC is the control signal $\ve{u}$ instead of the output error $\ve{e}$. The control signal integrates the output error over time, causing correlations over time to arise in the feedback learning rule. 
    \item The presynaptic part of the feedback learning rule \eqref{eq:Q_dynamics} for DFC is the feedback compartment $\ve{v}^{\text{fb}}$ which consists of both the controller input $Q\ve{u}$ and the noise instead of only the noise. This will lead to an extra term in the expectation $\mathbb{E}[\Delta Q]$, which results in $Q$ aligning with $J^T(JJ^T + \gamma I)^{-1}$, $\gamma > 0$ instead of $J^T$.
\end{enumerate}

\subsubsection{Theorem and proof}
\paragraph{Noise dynamics.} For simplicity of the argument and proof, we assume that the noise only enters through the feedback compartment of the neuron. The proposed theorem below and its proof also holds when extra noise is added in the feedforward and central compartment, as long as it is independent from the noise in the feedback compartment. Throughout the main manuscript, we assumed instantaneous dynamics of the feedback compartment, i.e., $\ve{v}^{\text{fb}}_i(t) = Q_i \ve{u}(t)$. If we now add white noise to the feedback compartment, the limit to instantaneous dynamics is not well defined anymore ($\lim_{\tau_{v^{\text{fb}}}\rightarrow 0}$ in eq. \eqref{eq_app:dynamics_feedback_compartment}). Hence, we introduce the following dynamics for the feedback compartment
\begin{align}\label{eq_app:dynamics_feedback_compartment}
    \tau_{v^{\text{fb}}} \ddt \vfb_i(t) = - \vfb_i(t) + Q_i \ve{u}(t) + \sigma \ves{\xi}_i(t).
\end{align}
The network dynamics \eqref{eq:network_dynamics} are now given by 
\begin{align} \label{eq_app:network_dynamics_vfb}
    \tau_v \frac{\dd \ve{v}_i(t)}{\dd t} &= -\ve{v}_i(t) + W_i\phi(\ve{v}_{i-1}(t)) + \vfb_i(t).
\end{align}
When we remove the noise $\ves{\xi}_i$ and take $\lim_{\tau_{v^{\text{fb}}}\rightarrow 0}$, we recover the original network dynamics \eqref{eq:network_dynamics} of the main manuscript. If we assume that $\ve{u}$ is independent from $\vfb_i$, eq. \eqref{eq_app:dynamics_feedback_compartment} is a linear time-invariant stochastic differential equation that can be solved with the Variation of Constants method, without a special treatment for the white noise \citep{sarkka2019applied}, leading to the following solution for $\vfb_i$:
\begin{align}\label{eq_app:solution_v_fb}
    \vfb_i(t) = \frac{1}{\taufb} \int_{-\infty}^{t} \exp\big(-\frac{1}{\taufb}(t-\tau)\big) Q_i \ve{u}(\tau) \dd \tau  + \frac{\sigma}{\taufb} \int_{-\infty}^{t} \exp\big(-\frac{1}{\taufb}(t-\tau)\big) \ves{\xi}_i(\tau) \dd \tau.
\end{align}
If we now assume that $\taufb \ll \tau_u$, and hence the dynamics of the feedback compartment is much faster than $\ve{u}$, $\vfb$ can be approximated by
\begin{align}\label{eq_app:solution_v_fb_approx}
    \vfb_i(t) &\approx  Q_i \ve{u}(t)  + \frac{\sigma}{\taufb} \int_{-\infty}^{t} \exp\big(-\frac{1}{\taufb}(t-\tau)\big) \ves{\xi}_i(\tau) \dd \tau \\
    &\triangleq Q_i \ve{u}(t) + \sigma \ves{\epsilon}_i(t),
\end{align}
with $\ves{\epsilon}_i(t)$ exponentially filtered white noise, i.e., an Ornstein-Uhlenbeck process \citep{sarkka2019applied} with zero mean and covariance equal to \footnote{Note that in $\lim_{\taufb \rightarrow 0}$, $\ves{\epsilon}_i(t)$ has infinite variance, and is hence not well defined, explaining why we need to make this detour of assuming non-instantaneous dynamics for the feedback compartment.}
\begin{align}\label{eq_app:covariance_eps}
    \mathbb{E}\big[\ves{\epsilon}_i(t) \ves{\epsilon}_i(t+\Delta t)^T\big] = \frac{1}{2 \taufb} \exp \big(-\frac{1}{\taufb} |\Delta t|\big).
\end{align} 
In the remainder of the section, we assume this approximation to be exact. The network dynamics \eqref{eq_app:network_dynamics_vfb} can then be written as
\begin{align} \label{eq_app:network_dynamics_noisy}
    \tau_v \frac{\dd \ve{v}_i(t)}{\dd t} &= -\ve{v}_i(t) + W_i\phi(\ve{v}_{i-1}(t)) + Q_i \ve{u}(t) + \sigma \ves{\epsilon}_i(t).
\end{align}

Now, we are ready to state and prove the main theorem of this section, which shows that the feedback weight plasticity rule \eqref{eq:Q_dynamics} pushes the feedback weights to align with a damped pseudoinverse of the forward Jacobian $J$ of the network.

\begin{theorem}\label{theorem:fb_learning_full}
Assuming stable dynamics, a separation of timescales $\taufb, \tau_v \ll \tau_u \ll \tau_Q$, $k_p=0$, $\alpha \gg |\lambda_{\text{max}}(JQ_M)|$ and $J$ is of full rank, and given $\ve{r}_L^*=\ve{r}_L^-$ and uncorrelated white noise $\ves{\xi}_i \sim \mathcal{N}(\ves{0}, I)$ entering the feedback compartment, the dynamics of the first moment $Q_M$ of the apical weight plasticity \eqref{eq:Q_dynamics} has the following approximate solution in the limit of $\sigma \rightarrow 0$:

\begin{align}
    \lim_{\sigma \rightarrow 0} \ddt Q_M \approx -\frac{1}{\tau_u \alpha} Q_M JJ^T + \frac{1}{2\tau_u} J^T - \beta Q_M,
\end{align}
and the first moment converges to:
\begin{align}
    Q_M^{\mathrm{ss}} \approx \frac{\alpha}{2}J^T(JJ^T + \gamma I)^{-1},
\end{align}
with $\gamma = \alpha \beta \tau_u$. Furthermore, $Q_M^{\mathrm{ss}}$ satisfies Conditions \ref{con:Q_GN} and \ref{con:local_stability}, even if $\alpha=0$ in the latter.
\end{theorem}

\begin{proof}
Linearizing the system dynamics (which becomes exact in the limit of $\sigma \rightarrow 0$ and assuming stable dynamics), results in the following dynamical equation for the controller, recalling that $\ve{r}_L^*=\ve{r}_L^-$ (c.f. App. \ref{app:linearized_dyn}):
\begin{align}\label{eq:app_controller_dynamics_delta_v}
    \tau_u \ddt \ve{u}(t) = -J\Delta \ve{v}(t) - \alpha \ve{u}(t),
\end{align}
with $\Delta \ve{v}_i \triangleq \ve{v}_i - W_i \phi(\ve{v}_{i-1})$ and $\Delta \ve{v}$ the concatenation of all $\Delta \ve{v}_i$. When we have a separation of timescales between the network and controller, i.e., $\tau_v \ll \tau_u$, which corresponds with \textit{instant system dynamics} of the network \eqref{eq_app:network_dynamics_noisy}, we get
\begin{align}\label{eq:app_delta_v_instant_noise}
    \Delta \ve{v}_i &= Q_i\ve{u} + \sigma \ves{\epsilon}_i \\
    \Delta \ve{v} &= Q\ve{u} + \sigma \ves{\epsilon}
\end{align}
where the latter is the concatenated version of the former.
% Let us introduce some notations to make the equations more compact:
% \begin{align}
%     \ves{\xi}^{S} &\triangleq [\ves{\xi}^{ST}_1 ... \ves{\xi}^{ST}_L]^T \\
%     \ves{\xi}^{B} &\triangleq [\ves{\xi}^{BT}_1 ... \ves{\xi}^{BT}_L]^T \\
%     \ves{\xi}^{A} &\triangleq [\ves{\xi}^{AT}_1 ... \ves{\xi}^{AT}_L]^T \\
%     \ves{\xi} &\triangleq [\ves{\xi}^{T}_1 ... \ves{\xi}^{T}_L]^T \\
%     \bar{\ves{\xi}} &\triangleq [\ves{\xi}^{AT} \ves{\xi}^{BT} \ves{\xi}^{ST}]^T
% \end{align}
% Hence, we have that 
% \begin{align}
%     \sigma \ves{\xi} =  \sigma(c_A \ves{\xi}_i^A + c_B \ves{\xi}_i^B + c_S \ves{\xi}_i^S)
% \end{align}
% Now we can summarize \eqref{eq:app_delta_v_instant_noise} for all layers $i$ into
% \begin{align}
%     \Delta \ve{v} &= Q\ve{u} + \sigma \ves{\xi} = Q\ve{u} + M_{\sigma} \bar{\ves{\xi}} \\
%     M_{\sigma} &\triangleq \sigma \begin{bmatrix} 
%     c_A I & c_B I & c_S I \end{bmatrix}
% \end{align}
Combining this with eq. \eqref{eq:app_controller_dynamics_delta_v} gives the following stochastic differential equation for the controller dynamics:
\begin{align}\label{eq:app_controller_dynamics_noisy}
    \tau_u \ddt \ve{u}(t) = -(JQ + \alpha I) \ve{u}(t) - \sigma J \ves{\epsilon}(t).
\end{align}
When we have a separation of timescales between the synaptic plasticity and controller dynamics, i.e., $\tau_u \ll \tau_Q$, we can treat $Q$ as constant and therefore eq. \eqref{eq:app_controller_dynamics_noisy} represents a linear time-invariant stochastic differential equation, which has as solution \citep{sarkka2019applied}
\begin{align}
    \ve{u}(t) &=  -\frac{\sigma}{\tau_u}\int_{-\infty}^t e^{-\frac{1}{\tau_u} A (t-\tau)} J \ves{\epsilon}(\tau) \dd \tau  \\
    A &\triangleq JQ + \alpha I.
\end{align}
Using the approximate solution of the feedback compartment \eqref{eq_app:solution_v_fb_approx} (which we consider exact due to the separation of timescales $\taufb \ll \tau_u$), we can write the expectation of the first part of the feedback learning rule \eqref{eq:Q_dynamics} as 
\begin{align}
    &\mathbb{E}\big[-\vfb(t)\ve{u}(t)^T\big] = \mathbb{E}\Big[-\big(Q\ve{u}(t) + \sigma \ves{\epsilon}(t)\big)\ve{u}(t)^T\Big] \\
    &= \mathbb{E} \bigg[ \underbrace{-\frac{\sigma^2}{\tau_u^2}Q\int_{-\infty}^t \exp\big(-\frac{1}{\tau_u}A(t-\tau)\big) J \ves{\epsilon}(\tau)\dd \tau \int_{-\infty}^t \ves{\epsilon}(\tau)^T J^T \exp\big(-\frac{1}{\tau_u}A^T(t-\tau)\big) \dd \tau}_{\text{(a)}} \\
    &\quad ...+ \underbrace{\frac{\sigma^2}{\tau_u} \ves{\epsilon}(t)  \int_{-\infty}^t \ves{\epsilon}(\tau)^T J^T \exp\big(-\frac{1}{\tau_u}A^T(t-\tau)\big) \dd \tau}_{\text{(b)}}\bigg].
\end{align}

Focusing on (a) and using the covariance of $\ves{\epsilon}$ \eqref{eq_app:covariance_eps}, we get:
\begin{align}
    \mathbb{E}[\text{(a)}] &= -\frac{\sigma^2}{\tau_u^2}Q\int_{-\infty}^t e^{\frac{1}{\tau_u}A(t-\tau_1)} J \int_{-\infty}^t \frac{1}{2 \taufb} e^{ -\frac{1}{\taufb} |\tau_1 - \tau_2|} J^T e^{-\frac{1}{\tau_u}A^T(t-\tau_2)} \dd \tau_2\dd \tau_1 \\
    &=-\frac{\sigma^2}{\tau_u^2}Q\int_{-\infty}^t e^{-\frac{1}{\tau_u}A(t-\tau_1)} J \bigg(\int_{-\infty}^{\tau_1} \frac{1}{2 \taufb} J^T e^{ -\big(\frac{1}{\taufb}I + \frac{1}{\tau_u} A^T\big)(\tau_1 - \tau_2)} e^{-\frac{1}{\tau_u}A^T(t-\tau_1)} \dd \tau_2 \\
    &\qquad ... + \int_{\tau_1}^{t} \frac{1}{2 \taufb} J^T e^{ -\big(\frac{1}{\taufb}I - \frac{1}{\tau_u} A^T\big)(\tau_2 - \tau_1)} e^{-\frac{1}{\tau_u}A^T(t-\tau_1)}\dd \tau_2 \bigg) \dd \tau_1 \\
    &\approx -\frac{\sigma^2}{\tau_u^2}Q \int_{-\infty}^t e^{-\frac{1}{\tau_u}A(t-\tau_1)}JJ^T e^{-\frac{1}{\tau_u}A^T(t-\tau_1)}\dd \tau_1,
\end{align}
where we used in the last step that $\taufb \ll \tau_u$, hence $\frac{1}{\taufb}I - \frac{1}{\tau_u} A^T \approx \frac{1}{\taufb}I$ and $\frac{1}{\taufb}\int_{-t_1}^{0} e^{-\frac{1}{\taufb}\tau}\dd \tau \approx 1$ when $\taufb \ll t_1 $ for $t_1>0$. If we further assume that $\alpha \gg \max\big(\{|\lambda_i(JQ)|\}\big)$ with $\lambda_i(JQ)$ the eigenvalues of $JQ$, we have that
\begin{align}
    e^{A} = e^{JQ + \alpha I} \approx e^{\alpha},
\end{align}
and hence 
\begin{align}
    \mathbb{E}[\text{(a)}] &\approx -\frac{\sigma^2}{\tau_u^2}Q \int_{-\infty}^t e^{-\frac{2}{\tau_u}\alpha (t-\tau_1)}JJ^T \dd \tau_1\\
    &= -\frac{\sigma^2}{\tau_u \alpha} Q JJ^T.
\end{align}

Focusing on part (b), we get
\begin{align}
    \mathbb{E}[\text{(b)}] &= \frac{\sigma^2}{\tau_u}\int_{-\infty}^t  \frac{1}{2 \taufb} e^{ -\frac{1}{\taufb} (t-\tau)}J^T \exp\big(-\frac{1}{\tau_u}A^T(t-\tau)\big) \dd \tau \\
    &= \frac{\sigma^2}{\tau_u} J^T \int_{-\infty}^t \frac{1}{2 \taufb} e^{ -\big(\frac{1}{\taufb}I + \frac{1}{\tau_u} A^T\big) (t-\tau)} \dd \tau \\
    & \approx \frac{\sigma^2}{2\tau_u} J^T.
\end{align}

Taking everything together, we get the following approximate dynamics for the first moment of $Q$:
\begin{align}
    &\lim_{\sigma \rightarrow 0} \ddt \mathbb{E}\big[Q\big] = \lim_{\sigma \rightarrow 0} \ddt Q_M = \lim_{\sigma \rightarrow 0} \mathbb{E}\big[ -\frac{1}{\sigma^2}\vfb \ve{u}^T - \beta Q \big] \\
    &\approx -\frac{1}{\tau_u \alpha} Q_M JJ^T + \frac{1}{2\tau_u} J^T - \beta Q_M. \label{eq_app:Q_dynamics_approx}
\end{align}

Assuming the approximation exact and solving for the steady state, we get:
\begin{align}
    0 &= -\frac{1}{\tau_u \alpha} Q_M^{\mathrm{ss}} JJ^T + \frac{1}{2\tau_u} J^T - \beta Q_M^{\mathrm{ss}} \\
    \Rightarrow \quad Q_M^{\mathrm{ss}} &= \frac{\alpha}{2}J^T(JJ^T + \alpha \beta \tau_u I)^{-1}.
\end{align}
The only thing remaining to show is that the dynamics of $Q_M$ are convergent. By vectorizing eq. \eqref{eq_app:Q_dynamics_approx}, we get
\begin{align}
    \tau_Q \ddt \text{vec}\big(Q_M) = -\frac{1}{\tau_u \alpha} \big( JJ^T \otimes I + \alpha \beta \tau_u I\big)\text{vec}\big(Q_M\big) + \frac{1}{2\tau_u} \text{vec}\big(J^T\big).
\end{align}
As the eigenvalues of a Kronecker product $A\otimes B$ are equal to the products $\lambda^A_i \lambda^B_j$, the eigenvalues of $JJ^T \otimes I$ are equal to the eigenvalues of $JJ^T$ (in higher multiplicity) and hence all positive. This makes the above dynamical system convergent, thereby concluding the main part of the proof. Finally, Lemma \ref{lemma_app:Q_satisfies_conditions} shows that, if $J$ is full rank, $Q_M^{\mathrm{ss}} = \frac{\alpha}{2}J^T(JJ^T + \alpha \beta \tau_u I)^{-1}$ satisfies Conditions \ref{con:Q_GN} and \ref{con:local_stability}, even if $\alpha=0$ in the latter.

\end{proof}

\begin{lemma}\label{lemma_app:Q_satisfies_conditions}
$Q=J^T(JJ^T + \gamma I)^{-1}$ with $\gamma \geq 0$ satisfies Condition \ref{con:Q_GN} and the product $JJ^T(JJ^T + \gamma I)^{-1}$ with $\gamma \geq 0$ has strictly positive eigenvalues if $J$ is of full rank. 
\end{lemma}
\begin{proof}
When $J$ is of full rank, $J^T(JJ^T + \gamma I)^{-1}$ can be written as $J^T M$, with $M$ a square full rank matrix. As $M$ is full rank, $J^TM$ has the same column space as $J^T$, thereby proving that $Q=J^T(JJ^T + \gamma I)^{-1}$ with $\gamma \geq 0$ satisfies Condition \ref{con:Q_GN}. 

Next, consider the singular value decomposition of $J$:
\begin{align}
    J = U\Sigma V^T.
\end{align}
Now, $JJ^T(JJ^T + \gamma I)^{-1}$ can be written as
\begin{align}\label{eq_app:svd_jjt}
    JJ^T(JJ^T + \gamma I)^{-1} = U\Sigma \Sigma^T (\Sigma \Sigma^T + \gamma I)^{-1} U^T,
\end{align}
with $\Sigma \Sigma^T (\Sigma \Sigma^T + \gamma I)^{-1}$ a diagonal matrix with $\frac{\sigma_i^2}{\sigma_i^2 + \gamma} > 0$ on its diagonal, and $\sigma_i$ the singular values of $J$. As $U^T = U^{-1}$, and $\Sigma \Sigma^T (\Sigma \Sigma^T + \gamma I)^{-1}$ is diagonal, eq. \eqref{eq_app:svd_jjt} is the eigenvalue decomposition of $JJ^T(JJ^T + \gamma I)^{-1}$, with eigenvalues $\frac{\sigma_i^2}{\sigma_i^2 + \gamma} > 0$, thereby concluding the proof. 
\end{proof}

\subsection{Toy experiments corroborating the theory}
To test whether Theorem \ref{theorem:fb_learning_full} can also provide insight into more realistic settings, we conducted a series of student-teacher toy regression experiments with a one-hidden-layer network of size $20-10-5$ for more realistic values of $\taufb$, $\tau_v$, $\alpha$ and $k_p>0$. For details about the simulation implementation, see App. \ref{app:simulation_and_algorithms}. We investigate the learning of $Q$ during pre-training, hence, when the forward weights $W_i$ are fixed. In contrast to Theorem \ref{theorem:fb_learning_full}, we use multiple batch samples for training the feedback weights. When the network is linear, $J$ remains the same for each batch sample, hence mimicking the situation of Theorem \ref{theorem:fb_learning_full} where $Q$ is trained on only one sample to convergence. When the network is nonlinear, however, $J$ will be different for each sample, causing $Q$ to align with an average configuration over the batch samples. 

We start by investigating which damping value $\gamma$ accurately describes the alignment of $Q$ with $J^T(JJ^T + \gamma I)^{-1}$ in this more realistic case. Fig. \ref{fig_app:fb_learning_5}.a shows the alignment of $Q$ with $J^T(JJ^T + \gamma I)^{-1}$ for different damping values $\gamma$ in a linear network. Interestingly, the damping value that optimally describes the alignment of $Q$ is $\gamma=5$, which is much larger than would be predicted by Theorem \ref{theorem:fb_learning_full} which uses simplified conditions. Hence, the more realistic settings used in the simulation of these toy experiments result in a larger damping value $\gamma$. For nonlinear networks, similar conclusions can be drawn (see Fig. \ref{fig_app:fb_learning_5}.b), however, with slightly worse alignment due to $J$ changing for each batch sample. Note that almost perfect compliance to Condition \ref{con:Q_GN} is reached for both the linear and nonlinear case (not shown here).

\begin{figure}[h!]
\centering
\begin{subfigure}{.49\textwidth}
  \centering
  % include first image
  \includegraphics[width=1.\linewidth]{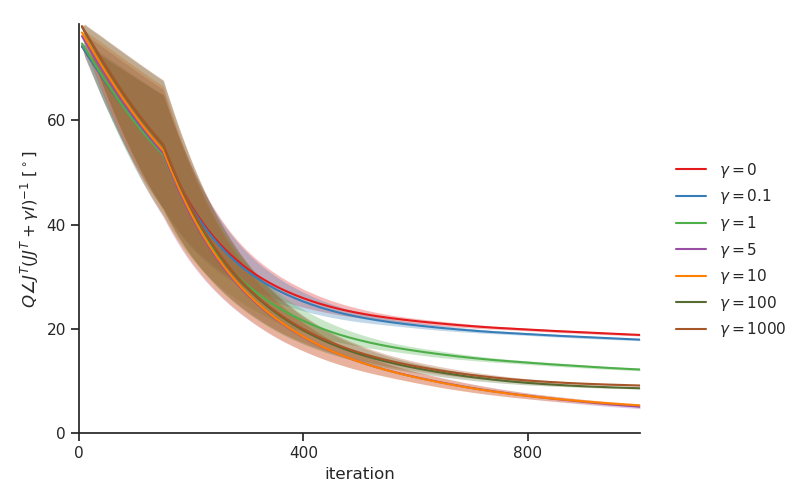}  
  \caption{Linear}
\end{subfigure}
\begin{subfigure}{.49\textwidth}
  \centering
  % include second image
  \includegraphics[width=1.\linewidth]{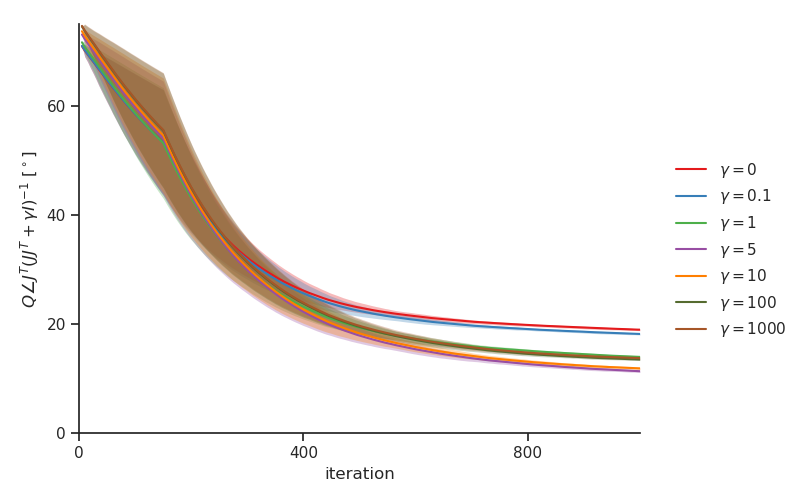}  
  \caption{Nonlinear}
\end{subfigure}
\caption{Alignment of the feedback weights $Q$ with the damped pseudoinverse $J^T(JJ^T + \gamma I)^{-1}$ for various values of $\gamma$. We used a one-hidden-layer network of size 20-10-5 with a linear output layer and (a) a linear hidden layer or (b) a $\tanh$ hidden layer. Hyperparameters: $k_p=0$, $\alpha=0.5$, $\tau_u=1.$, $\taufb = 0.3$, $\tau_v=0.005$, $\beta =0.01$ and $\sigma=0.01$. We used 300 Euler-Maruyama simulation steps of size $\Delta t = 0.001$. A window-average is plotted together with the window-std (shade).}
\label{fig_app:fb_learning_5}
\end{figure}

Next, we investigate how big $\alpha$ needs to be for good alignment. Surprisingly, Fig. \ref{fig_app:fb_learning_6} shows that $Q$ reaches almost perfect alignment for all values of $\alpha \in [0,1]$, both for linear and nonlinear networks. We hypothesize that this is due to the short simulation window (300 steps of $\Delta t = 0.001$) that we used to reduce computational costs, preventing the dynamics from diverging, even when they are unstable. Interestingly, this hypothesis leads to another case where the feedback learning rule \eqref{eq:Q_dynamics} can be used besides for big $\alpha$: when the network activations can be `reset' when they start diverging, e.g., by inhibition from other brain areas, the feedback weights can be learned properly, even with unstable dynamics. 

\begin{figure}[h!]
\centering
\begin{subfigure}{.49\textwidth}
  \centering
  % include first image
  \includegraphics[width=1.\linewidth]{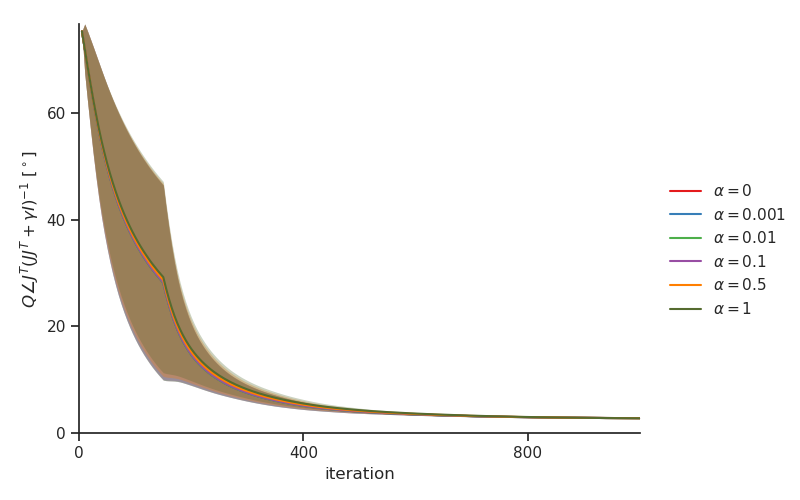}  
  \caption{Linear}
\end{subfigure}
\begin{subfigure}{.49\textwidth}
  \centering
  % include second image
  \includegraphics[width=1.\linewidth]{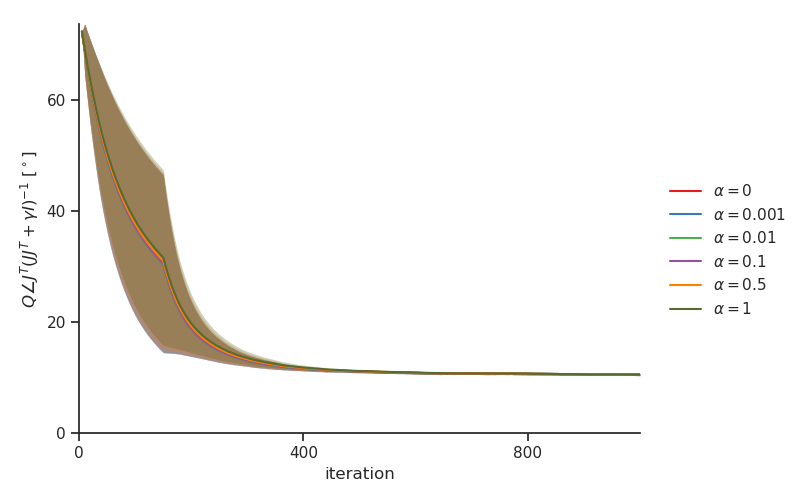}  
  \caption{Nonlinear}
\end{subfigure}
\caption{Alignment of the feedback weights $Q$ with the damped pseudoinverse $J^T(JJ^T + \gamma I)^{-1}$ for various values of $\alpha$. We used a one-hidden-layer network of size 20-10-5 with a linear output layer and (a) a linear hidden layer or (b) a $\tanh$ hidden layer. Hyperparameters: $k_p=0.3$, $\gamma=5$, $\tau_u=1.$, $\taufb = 0.3$, $\tau_v=0.005$, $\beta =0.01$ and $\sigma=0.01$. We used 300 Euler-Maruyama simulation steps of size $\Delta t = 0.001$. A window-average is plotted together with the window-std (shade).}
\label{fig_app:fb_learning_6}
\end{figure}

Finally, we investigate how $k_p$ influences the feedback learning. Fig. \ref{fig_app:fb_learning_7} shows that bigger $k_p$ increase the speed of alignment with $J^T(JJ^T + \gamma I)^{-1}$.

\begin{figure}[h!]
\centering
\begin{subfigure}{.49\textwidth}
  \centering
  % include first image
  \includegraphics[width=1.\linewidth]{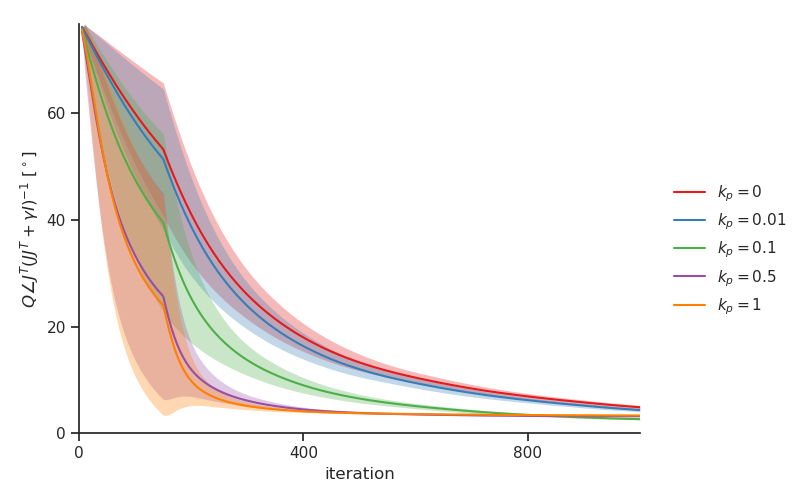}  
  \caption{Linear}
\end{subfigure}
\begin{subfigure}{.49\textwidth}
  \centering
  % include second image
  \includegraphics[width=1.\linewidth]{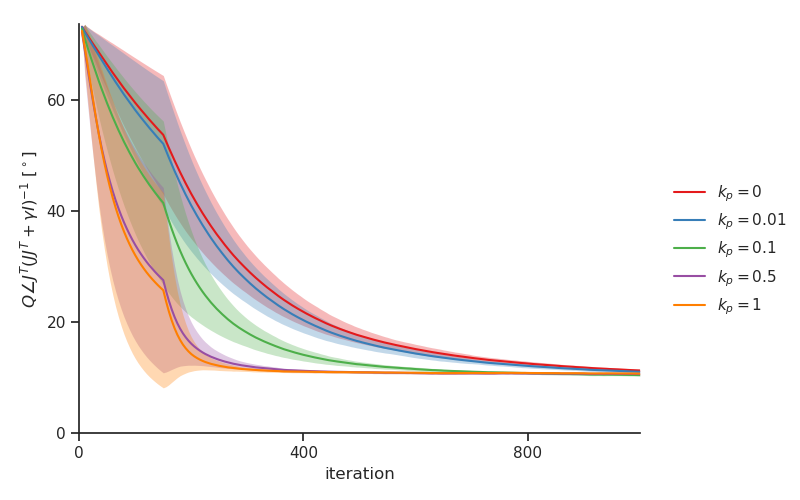}  
  \caption{Nonlinear}
\end{subfigure}
\caption{Alignment of the feedback weights $Q$ with the damped pseudoinverse $J^T(JJ^T + \gamma I)^{-1}$ for various values of $k_p$. We used a one-hidden-layer network of size 20-10-5 with a linear output layer and (a) a linear hidden layer or (b) a $\tanh$ hidden layer. Hyperparameters: $\alpha=0.1$, $\gamma=5$, $\tau_u=1.$, $\taufb = 0.3$, $\tau_v=0.005$, $\beta =0.01$ and $\sigma=0.01$. We used 300 Euler-Maruyama simulation steps of size $\Delta t = 0.001$. A window-average is plotted together with the window-std (shade).}
\label{fig_app:fb_learning_7}
\end{figure}

\subsection{Learning the forward and feedback weights simultaneously}\label{app:fb_learning_simultaneously}

In this section, we show that the forward and feedback weights can be learned simultaneously, when noise is added to the feedback compartment, resulting in the noisy dynamics of eq. \eqref{eq_app:network_dynamics_noisy}, and when the feedback plasticity rule \eqref{eq:Q_dynamics} uses a high-pass filtered version of $\ve{u}$ as presynaptic plasticity signal. 

We make the same assumptions as in Theorem \ref{theorem:fb_learning_full}, except now the output target $\ve{r}_L^*$ is the one for learning the forward weights, hence given by eq. \eqref{eq:output_target}. 
Linearizing the network dynamics, gives us the following expression for the control error
\begin{align}
    \ve{e}(t) = \ves{\delta}_L - J\Delta \ve{v}(t),
\end{align}
and for the controller dynamics (with $k_p=0$) 
\begin{align}
    \tau_u \ddt \ve{u}(t) = \ves{\delta}_L - J\Delta \ve{v}(t) - \alpha \ve{u}(t).
\end{align}
Using instantaneous network dynamics ($\tau_v \ll \tau_u$), we have that $\Delta \ve{v}(t) = Q\ve{u}(t) + \sigma \ves{\epsilon}(t)$, giving us:
\begin{align}\label{eq_app:u_dyn_inst_fblearning}
    \tau_u \ddt \ve{u}(t) = \ves{\delta}_L - (JQ + \alpha I)\ve{u}(t) -\sigma J\ves{\epsilon}(t).
\end{align}
We now continue by investigating the dynamics of newly defined signal $\tdu(t)$ that subtracts a baseline from the control signal $\ve{u}(t)$:
\begin{align}
    \tdu(t) &\triangleq \ve{u}(t) - \ve{u}_{\sss}, \,\,\,\, \ve{u}_{\mathrm{ss}} = (JQ + \alpha I)^{-1}\ves{\delta}_L,
\end{align}
with $\ve{u}_{\mathrm{ss}}$ being the steady state of $\ve{u}$ in the dynamics without noise (see Lemma \ref{lemma:steady_state_solution_system}). Rewriting the dynamics \eqref{eq_app:u_dyn_inst_fblearning} for $\tdu$ gives us
\begin{align}
    \tau_u \ddt \Delta \ve{u}(t) = - (JQ + \alpha I)\Delta \ve{u}(t) -\sigma J\ves{\epsilon}(t).
\end{align}
We now recovered exactly the same dynamics for $\Delta \ve{u}$ as was the case for $\ve{u}$ \eqref{eq:app_controller_dynamics_noisy} during the sleep phase where $\ve{r}_L^*=\ve{r}_L^-$ in Theorem \ref{theorem:fb_learning_full}. Now, we introduce a new plasticity rule for $Q$ using $\tdu$ instead of $\ve{u}$ as presynaptic plasticity signal:
\begin{align} \label{eq_app:Q_dynamics_hpf}
    \tau_Q \ddt Q(t) = -\vfb(t) \tdu(t)^T - \beta Q(t).
\end{align}
Upon noting that $\tdu$ (representing the noise fluctuations in $\ve{u}$) is independent of $\ve{u}_{\sss}$ (representing the control input needed to drive the network to $\ve{r}_L^*$), the approximate first moment dynamics described in Theorem \ref{theorem:fb_learning_full} also hold for the new plasticity rule \eqref{eq_app:Q_dynamics_hpf}. Furthermore, when the controller dynamics \eqref{eq_app:u_dyn_inst_fblearning} have settled, $\ve{u}_{\sss}$ is the average of $\ve{u}(t)$ (which has zero-mean noise fluctuations on top of $\ve{u}_{\sss}$), hence, $\tdu$ can be seen as a high-pass filtered version of $\ve{u}(t)$. 

To conclude, we have shown that the sleep phase for training the feedback weights $Q$ can be merged with the phase for training the forward weights with $\ve{r}_L^*$ as defined in eq. \eqref{eq:output_target}, if the plasticity rule for $Q$ \eqref{eq_app:Q_dynamics_hpf} uses a high-pass filtered version $\tdu$ of $\ve{u}$ as presynaptic plasticity signal and when the network and controller are fluctuating around their equilibrium, as we did not take initial conditions into account. We hypothesize that even with initial dynamics that have not yet converged to the steady-state, the plasticity rule for Q \eqref{eq_app:Q_dynamics_hpf} with $\Delta \ve{u}$ a high-pass filtered version of $\ve{u}$ will result in proper feedback learning, as high-pass filtering $\ve{u}(t)$ will extract high-frequency noise fluctuations\footnote{Not all noise fluctuations are high-frequency. However, the important part of the hypothesis is that the high-pass filtering selects noise components that are zero-mean and correlate with $\vfb$.} out of it which are correlated with $\vfb$ and can hence be used for learning $Q$. We leave it to future work to experimentally verify this hypothesis. Merging the two phases into one has as a consequence that there is also noise present during the learning of the forward weights \eqref{eq:W_dynamics}, which we investigate in the next subsection.

\subsection{Influence of noisy dynamics on learning the forward weights}
When there is noise present in the dynamics during learning the forward weights, this will have an influence on the updates of $W_i$. It turns out that the same noise correlations that we used in the previous sections to learn the feedback weights will cause bias terms to appear in the updates of the forward weights $W_i$ \eqref{eq:W_dynamics}. This issue is not unique to our DFC setting with a feedback controller but appears in general in methods that use error feedback and have realistic noise dynamics in their hidden layers. In this section, we lay down the issues caused by noise dynamics for learning forward weights for general methods that use error feedback. At the end of the section, we comment on the implications of these issues for DFC. 

For simplicity, we consider a normal feedforward neural network 
\begin{align}
    \ve{r}_i^- = \phi(\ve{v}^-_i) = \phi(W_i \ve{r}_{i-1}^-), \quad 1\leq i \leq L.
\end{align}

To incorporate the notion of noisy dynamics, we perturb each layer's pre-nonlinearity activation with zero-mean noise $\ves{\epsilon}_i$ and propagate the perturbations forward:
\begin{align}
    \tilde{\ve{v}}_i^- = W_i \phi(\tilde{\ve{v}}_{i-1}^-) + \sigma \ves{\epsilon}_i, \quad 1\leq i \leq L,
\end{align}
with $\tilde{\ve{r}}_0^- = \ve{r}_0^-$. For small $\sigma$, a first-order Taylor approximation of the perturbed output gives
\begin{align}
    \tilde{\ve{r}}_L^- = \ve{r}_L^- + \sigma J \ves{\epsilon} + \mathcal{O}(\sigma^2),
\end{align}
with $\ves{\epsilon}$ the concatenated vector of all $\ves{\epsilon}_i$. If the task loss is an $L^2$ loss and we have the training label $\ve{r}_L^*$, the output error is equal to 
\begin{align}
    \ve{e}_L = \ve{r}_L^* -\tilde{\ve{r}}_L^- = \ves{\delta}_L -\sigma J \ves{\epsilon} + \mathcal{O}(\sigma^2),
\end{align}
with $\ves{\delta}_L = \ve{r}_L^* - \ve{r}_L^-$, the output error without noise perturbations. To remain general, we define the feedback path $\ve{e}_i = g_i(\ve{e}_L)$ that transports the output error $\ve{e}_L$ to the hidden layer $i$, at the level of the pre-nonlinearity activations. E.g., for BP, $\ve{e}_i = g_i(\ve{e}_L) = J_i^T\ve{e}_L$, and for direct linear feedback mappings such as DFA, $\ve{e}_i = g_i(\ve{e}_L) = Q_i\ve{e}_L$. Now, the commonly used update rule of postsynaptic error signal multiplied with presynaptic input gives (after a first-order Taylor expansion of all terms)
\begin{align}
    \Delta W_i &= \eta \ve{e}_i \tilde{\ve{r}}_{i-1}^{-T}\\
    &= \eta \big( \ves{\delta}_i - \sigma J_{g_i} J \ves{\epsilon} + \mathcal{O}(\sigma^2)\big) \big(\ve{r}_{i-1}^{-} + \sigma D_{i-1} \ves{\epsilon}_{i-1} + \mathcal{O}(\sigma^2)\big)^T,
\end{align}
with $\ves{\delta}_i = g_i(\ves{\delta}_L)$, $J_{g_i} = \frac{\partial g_i(\ve{e}_L)}{\partial \ve{e}_L}\big\rvert_{\ve{e}_L=\ves{\delta}_L}$ and $D_i = \frac{\partial \ve{r}^-_i}{\partial \ve{v}_i}\big\rvert_{\ve{v}_i=\ve{v}_i^-}$. Taking the expectation of $\Delta W_i$, we get
\begin{align}
    \mathbb{E}[\Delta W_i] = \eta \ves{\delta_i}\ve{r}_{i-1}^{-T} - \eta \sigma^2 J_{g_i}J_{i-1}\Sigma_{i-1} D_{i-1}  + \mathcal{O}(\sigma^3),
\end{align}
with $\Sigma_{i-1}$ the covariance matrix of $\ves{\epsilon}_{i-1}$.
We see that besides the desired update $\eta \ves{\delta_i}\ve{r}_{i-1}^{-T}$, there also appears a bias term due to the noise, which scales with $\sigma^2$ and cannot be avoided by averaging over weight updates. The noise bias arises from the correlation between the noise in the presynaptic input $\tilde{\ve{r}}_{i-1}$ and the postsynaptic error $\ve{e}_i$. Note that it is not a valid strategy to assume that the noise in $\ve{e}_i$ is uncorrelated from the noise in $\tilde{\ve{r}}_{i-1}$ due to a time delay between the two signals, as in more realistic cases, $\ves{\epsilon}$ originates from stochastic dynamics that integrate noise over time (e.g., one can think of $\ves{\epsilon}$ as an Ornstein-Uhlenbeck process \citep{sarkka2019applied}) and is hence always correlated over time. 

In DFC, similar noise biases arise in the average updates of $W_i$. To reduce the relative impact of the noise bias on the weight update, the ratio $\|\ves{\delta}_i\|_2/\sigma^2$ must be big enough, hence strong error feedback is needed. In DFC, $\|\ves{\delta}_L\|_2$, and hence also the postsynaptic error term in the weight updates for $W_i$, scales with the target stepsize $\lambda$. Interestingly, this causes a trade-off to appear in DFC: on the one hand, $\lambda$ needs to be small such that the weight updates \eqref{eq:W_dynamics} approximate GN and MN optimization (the theorems used Taylor approximations which become exact for $\lambda \rightarrow 0$), and on the other hand, $\lambda$ needs to be big to prevent the forward weight updates from being buried in the noise bias.

A possible solution for removing the noise bias from the average forward weight updates is to either buffer the postsynaptic error term or the presynaptic input $\ve{r}_{i-1}$, or both (e.g., accumulating them or low-pass filtering them), before they are multiplied with each other to produce the weight update. This procedure would average the noise out in the signals, before they have the chance to correlate with each other in the weight update. Whether this procedure could correspond with biophysical mechanisms in a neuron is an interesting question for future work.

% We now define a simple learning rule $\Delta Q = -\sigma \ves{\xi} \ve{e}^T - \beta Q$ which is a simple anti-Hebbian rule with as presynaptic signal the output error $\ve{e}$ and as postsynaptic signal the noise inside the neuron $\sigma\ves{\xi}$, combined with weight decay. If $\ves{\xi}$ is uncorrelated white noise with correlation matrix equal to the identity matrix, the expectation of this learning rule is 
% \begin{align}
%     \mathbb{E}[\Delta Q] = \sigma^2 J^T - \beta Q.
% \end{align}
% We see that this learning rule lets the feedback weights $Q$ align with the transpose of the networks Jacobian $J$ and has a weight decay term to prevent $Q$ from diverging. 

% We how continue by investigating the above dynamics around the average steady state of $\ve{u}$ by recalling from Lemma \ref{lemma:steady_state_solution_system}
% \begin{align}
%     \ve{u}_{\mathrm{ss}} &= (JQ + \alpha I)^{-1}\ves{\delta}_L \\
%     \Delta \ve{v}_{\mathrm{ss}} &= Q(JQ + \alpha I)^{-1}\ves{\delta}_L
% \end{align}
% and by defining
% \begin{align}
%     \tdu(t) &\triangleq \ve{u}(t) - \ve{u}_{\sss}\\
%     \tdv(t) &\triangleq \Delta\ve{v}(t) - \Delta \ve{v}_{\sss}
% \end{align}
% By filling the above definitions into the controller dynamics \eqref{eq_app:u_dyn_inst_fblearning} and upon noting that $\ve{\delta}_L -J\Delta \ve{v}_{\sss} - \alpha \ve{u}_{\sss} = 0$, we get

\section{Related work}\label{app:related_work}
 
Our learning theory analysis that connects DFC to Gauss-Newton (GN) optimization was inspired by three independent recent studies that, on the one hand, connect Target Propagation (TP) to GN optimization \citep{meulemans2020theoretical, bengio2020deriving} and, on the other hand, point to a possible connection between Dynamic Inversion (DI) and GN optimization \citep{podlaski2020biological}. There are however important distinctions between how DFC approximates GN and how TP and DI approximate GN. In the following subsections, we discuss these related lines of work in detail.

\subsection{Comparison of DFC to TP and variants}\label{app:subsec_theoretical_TP}

Recent work \citep{meulemans2020theoretical, bengio2020deriving} discovered that learning through inverses of the forward pathway can in certain cases lead to an approximation of GN optimization. Although this finding inspired our theoretical results on the CA capabilities of DFC, there are fundamental differences between DFC and TP.
% indicating that DFC is more than a mere evolution of the TP framework. 
The main conceptual difference between DFC and the variants of TP \citep{bengio2014auto, lee2015difference, meulemans2020theoretical, bengio2020deriving} is that DFC uses the combination of network dynamics and a controller to dynamically invert the forward pathway for CA, whereas TP and its variants learn parametric inverses of the forward pathway, encoded in the feedback weights. Although dynamic and parametric inversion seem closely related, they lead to major methodological and theoretical differences.

\paragraph{Methodological differences between DFC and TP.}

First, for TP and its variants, the task of approximating the inverse of the forward pathway is completely put onto the feedback weights, resulting in the need for a strict relation between the feedforward and feedback pathway at all times during training. DFC, in contrast, reuses the forward pathway to dynamically compute its inverse, resulting in a more flexible relation between the feedforward and feedback pathway, described by Condition \ref{con:Q_GN}. To the best of our knowledge, DFC is the first method that approximates a principled optimization method for feedforward neural networks of arbitrary dimensions, compatible with a wide range of feedback connectivity. The recent work of Bengio \citep{bengio2020deriving} iteratively improves the inverse and, hence, can compensate for imperfect parametric inverses. However, this method is developed only for invertible networks, which require all layers to have equal dimensions.
 
Second, DFC drives the hidden neural activations to target values simultaneously, hence letting `target activations' from upstream layers influence `target activations' from downstream layers. TP, in contrast, computes each target as a (pseudo)inverse of the output target independently. This is a subtle yet important difference between DFC and TP, which leads to significant theoretical differences, on which we will expand later. To gain intuition, consider the case where we update the weights of both DFC and TP to reach exactly the local layer targets. In TP, if we update the weights of a hidden layer to reach its target, all downstream layers will also reach their target without updating the weights. Hence, if we update all weights simultaneously, the output will overshoot its target. DFC, in contrast, takes the effect of the updated target values of upstream layers already into account, hence, when all weight updates are done simultaneously, the output target is reached exactly (in the linearized dynamics, c.f. Theorem \ref{theorem:MN_DFC}). 

Third, DFC needs significantly less external coordination compared to the recent TP variants. The new variants of TP with a link to GN \citep{meulemans2020theoretical} need highly coordinated noise phases for computing the Difference Reconstruction Loss (one separated noise phase for each layer). For DTP \citep{lee2015difference}, similar coordination is needed if noisy activations are used for computing the reconstruction loss, as proposed by the authors. The iterative variant of TP \citep{bengio2020deriving} needs coordination in propagating the target values, as the target iterations for a layer can only start when the iterations of the downstream layer have converged. As DFC uses dynamic inversion instead of parametric inversion, possible learning rules for the feedback weights do not need to use the Difference Reconstruction Loss \citep{meulemans2020theoretical} or variants thereof, opening the route to alternative, more biologically realistic learning rules. We propose a first feedback learning rule compatible with DFC, that makes use of noise and Hebbian learning, without the need for extensive external coordination (see also App. \ref{app:fb_learning_simultaneously} that merges feedforward and feedback weight training in a single-phase).

Finally, DFC uses a multi-compartment neuron model closely corresponding to recent models of cortical pyramidal neurons, to obtain plasticity rules fully local in space and time. Presently, it is unclear whether there exist similar neuron and network models for TP that result in plasticity rules local in time.

\paragraph{Theoretical differences between DFC and TP.}
First, computing layerwise inverses, as is done in TP \citep{bengio2014auto}, DTP \citep{lee2015difference}, and iterative TP \citep{bengio2020deriving}, can only be linked to GN for invertible networks but breaks down for non-invertible networks, as shown by Meulemans et al. \citep{meulemans2020theoretical}. Both DFC and the DRL variants of TP \citep{meulemans2020theoretical} establish a link to GN for both invertible and non-invertible feedforward networks of arbitrary dimensions. However, the DRL variants of TP are linked to a hybrid version of GN and gradient descent, whereas DFC, under appropriate conditions, is linked to pure GN optimization on the parameters. Our Theorems \ref{theorem:GN_TPDI} and \ref{theorem:MN_DFC} differ from the theoretical results on the DRL variants of TP \citep{meulemans2020theoretical} due to the fact that: (i) the DRL variants compute targets for the post-nonlinearity activations and the DFC target activations, $\mathbf{v}_i$, are pre-nonlinearity activations; and (ii) the DRL variants compute the targets for each layer independently, whereas DFC dynamically computes the targets while taking into account the changed target activations of other layers. We continue with expanding on this second point.

As explained intuitively before, TP and its variants compute each layer target independently from the other layer targets. Consequently, to link their variants of TP to GN optimization, Meulemans et al. \citep{meulemans2020theoretical} and Bengio \citep{bengio2020deriving} need to make a block-diagonal approximation of the GN curvature matrix, with each block corresponding to a single layer. As off-diagonal blocks are put to zero, influences of upstream target values on the downstream targets are ignored. The block-diagonal approximation of the GN curvature matrix was proposed in studies that used GN optimization to train deep neural networks with big minibatch sizes \citep{martens2015optimizing, botev2017practical}. However, similar to DFC, TP is connected to GN with a minibatch size of 1. In this case, the GN curvature matrix is of low rank, and a block-diagonal approximation of this matrix will change its rank and hence its properties. In the analysis of DFC, in contrast, we do not need to make this block-diagonal approximation, as the target activations, $\mathbf{v}_i$, influence each other. Consequently, DFC has a closer connection to GN optimization than the TP variants \citep{meulemans2020theoretical, bengio2020deriving}. 

Finally, DFC does not use a reconstruction loss to train the feedback weights but instead uses noise and Hebbian learning. 
%Our Theorem \ref{theorem:fb_learning_simplified} that analyzes the behavior of this learning rule using stochastic differential equations is also a novel theoretical contribution to the field, to the best of our knowledge. 

\paragraph{Empirical comparison of DFC to TP and variants}\label{app:subsec_empirical_TP}
Table \ref{tab:tp_results} shows the results for DTP \citep{lee2015difference}, and DDTP-linear \citep{meulemans2020theoretical} (the best performing variant of TP in \citep{meulemans2020theoretical}) on MNIST, Fashion MNIST, MNIST-autoencoder, and MNIST (train), for the same architectures as used for Table \ref{tab:test_results}. 

\begin{comment}
The classification error and loss on MNIST and MNIST-autoencoder, respectively, for DTP and DDTP-linear.
|             |         MNIST         |          MNIST-autoencoder         |
|:-----------:|:---------------------:|:----------------------------------:|
|     DTP     | $2.61^{\\pm 0.13}\\%$ | $22.36^{\\pm 0.59} \\cdot 10^{-2}$ |
| DDTP-linear | $2.22^{\\pm 0.22}\\%$ | $14.60^{\\pm 0.10} \\cdot 10^{-2}$ |
\end{comment}

\begin{table}[H]
\centering
\caption{The test error (MNIST, Fashion MNIST), test loss (MNIST-autoencoder), and training loss MNIST (train) for DTP and DDTP-linear. Same network architectures and settings as for Table \ref{tab:test_results}.}
\label{tab:tp_results}
\begin{tabular}{*5c}
\toprule
{} & MNIST & Fashion-MNIST & MNIST-autoencoder  & MNIST (train) \\
\midrule
DTP & $2.61^{\pm 0.13}\%$ & $11.26^{\pm 0.23}\%$ & $22.36^{\pm 0.59} \cdot 10^{-2}$ &$8.36^{\pm 4.09} \cdot 10^{-6}$ \\
DDTP-linear & $2.22^{\pm 0.22}\%$ & $10.84^{\pm 0.22}\%$ & $14.60^{\pm 0.10} \cdot 10^{-2}$ &$1.97^{\pm 0.70} \cdot 10^{-8}$ \\
\bottomrule
\end{tabular}
\end{table}

Comparing these results to the ones in Table \ref{tab:test_results}, we see that DFC outperforms DTP on all datasets and DDTP-linear on MNIST-autoencoder, while having similar performance on the other datasets. These encouraging results suggest that the closer connection of DFC to GN, when compared to the one of DDTP-linear to GN (see section \ref{app:subsec_theoretical_TP}), leads to practical improvements in performance in some more challenging datasets.

\subsection{Comparison of DFC to Dynamic Inversion}
Recent work introduced DI \citep{podlaski2020biological}, which, similar to DFC, dynamically inverts the forward pathway through the use of a controller. However, some fundamental differences between DI and DFC lead to various new desirable properties of DFC. First, DFC introduces a principled way to control all layers simultaneously, hence requiring less external control. Second, in contrast to DI, the learning rules in DFC are fully local in time. Third, DFC can train the feedback weights to continuously adapt themselves to the changing forward pathway, leading to more accurate CA. Finally, \citet{podlaski2020biological} only explored the link between DI and GN for linear one-hidden layer networks and it requires a block-diagonal approximation of the curvature matrix, similar to TP. Upon closer inspection, the link between DI and GN cannot be generalized to networks with multiple hidden layers of various sizes or nonlinear activation functions, in contrast to DFC. This is because the layerwise dynamical inversion in DI does not result in the pseudoinverses of $J_i = \frac{\partial \vr_L}{\partial \vr_i}$ since: (i) the pseudoinverse cannot be factorized over the layers \citep{meulemans2020theoretical}; and (ii) in nonlinear networks, the Jacobians are evaluated at a wrong value because DI transmits errors instead of controlled layer activations through the forward path of the network during the dynamical inversion phase.

\subsection{The core contributions of DFC}\label{app:subsec_orignality_and_significance}

In summary, we see that DFC merges various insights from different fields resulting in a novel biologically plausible CA technique with unique and interesting properties that transcend the sheer sum of its parts. To clarify the novelty of our work, we summarize here again the core contributions of DFC:
\begin{itemize}
    \item DFC extends the idea of using a feedback controller to adjust network activations to also provide CA to DNNs by using it to track the desired output target, opening a new route for designing principled CA methods for DNNs.
    \item To the best of our knowledge, DFC is the first method that approximates a principled optimization method for feedforward neural networks of arbitrary dimensions, while allowing for a wide and flexible range of feedback connectivity, in contrast to a single allowed feedback configuration.
    \item The learning rules of DFC for the forward and feedback weights are fully local both in time and space, in contrast to many other biologically plausible learning rules. Furthermore, DFC does not need highly specific connectivity motives nor tightly coordinated plasticity mechanisms and can have all weights plastic simultaneously, if the adaptations explained in appendix C.3 are used. 
    \item The multi-compartment neuron model needed for DFC naturally corresponds to recent multi-compartment models of pyramidal neurons.
\end{itemize}

\section{Simulations and algorithms of DFC}\label{app:simulation_and_algorithms}
In this section, we provide details on the simulation and algorithms used for DFC, DFC-SS, DFC-SSA and for training the feedback weights. 
\subsection{Simulating DFC and DFC-SS for training the forward weights}
For simulating the network dynamics \eqref{eq:network_dynamics} and controller dynamics \eqref{eq:controller_dynamics} without noise, we used the forward Euler method with some slight modifications. First, we implemented the controller dynamics \eqref{eq:controller_dynamics} as follows:
\begin{align} \label{eq_app:controller_dynamics_modified}
    \ve{u}(t) = \ve{u}^{\text{int}}(t) + k_p \ve{e}(t), \quad 
    \tau_u \ddt \ve{u}^{\text{int}}(t) = \ve{e}(t) - \tilde{\alpha} \ve{u}(t).
\end{align}
Note that we changed the leakage term from $\alpha \ve{u}^{\mathrm{int}}$ to $\tilde{\alpha} \ve{u}$, such that we have direct control over the hyperparameter $\tilde{\alpha}$ (the damping factor in Lemma \ref{lemma:steady_state_solution_system}) that is now independent of $k_p$. Note that both \eqref{eq:controller_dynamics} and \eqref{eq_app:controller_dynamics_modified} result in exactly the same dynamics for $\ve{u}(t)$, if $\tilde{\alpha}=\frac{\alpha}{1+k_p \alpha}$ and $\tau_u$ scaled by $\frac{\alpha}{1+k_p \alpha}$. Hence, this is just an implementation strategy to gain direct control over $\tilde{\alpha}$ as a hyperparameter independent from $k_p$.

Algorithm \ref{al:DFC_forward} provides the pseudo-code of our simulation of the network and controller dynamics during the training of the forward weights $W_i$ and biases $\ve{b}_i$. We use the forward Euler method \citep{sarkka2019applied} with stepsize $\Delta t$ and make two nuanced modifications. First, to better reflect the layered structure of the network, we use $\vff_i[k+1] = W_i\phi(\ve{v}_{i-1}[k+1]) + \ve{b}_i$ instead of $\vff_i[k+1] = W_i\phi(\ve{v}_{i-1}[k]) + \ve{b}_i$.\footnote{In the code repository, this modification to Euler's method is indicated with the command line argument \texttt{inst\_transmission}} For small stepsizes $\Delta t$, this modification has almost no effect. However, for larger stepsizes, the modification better reflects the underlying continuous dynamics with its layerwise structure. Second, using insights from discrete control theory, we use $\vfb_i[k+1] = Q_i \ve{u}[k+1]$ instead of $\vfb_i[k+1] = Q_i \ve{u}[k]$, such that the control error $\ve{e}[k]$ of the previous timestep is used to provide feedback, instead of the control error $\ve{e}[k-1]$ of two timesteps ago.\footnote{In the code repository, this modification to Euler's method is indicated with the command line argument \texttt{proactive\_controller}} Again, this modification has almost no effect for small stepsizes $\Delta t$, but better reflects the underlying continuous dynamics for bigger stepsizes. In our simulations, the stepsize $\Delta t$ that worked best for the experiments was small, hence, the discussed modifications had only minor effects on the simulation.

\begin{algorithm}[h]
Initialize layer activations and parameter update buffers:\\
\For{i in \text{range}(1,L)}{
$\vv_i[1] = \vv_i^-$\\
$\vr_i[1] = \vr_i^-$\\
$\uint[1] = 0$ \\
$\Delta W_i = 0 $ \\
$\Delta \ve{b}_i = 0$\\
}
\For{k in \text{range}(\text{1,$K_{\max}$})}{
Update controller:\\
$\ve{e}[k] = \ve{r}_L^* - \ve{r}_L[k]$ \\
$\uint[k+1] = \uint[k] + \frac{\Delta t}{\tau_u} (\ve{e}[k] - \tilde{\alpha} \ve{u}[k])$ \\
$\ve{u}[k+1] = \uint[k+1] + k_p \ve{e}[k] $ \\

Update network:\\
\For{i in \text{range}(1,L)}{
$\vff_i[k+1] = W_i\phi(\ve{v}_{i-1}[k+1]) + \ve{b}_i$\\
$\vfb_i[k+1] = Q_i \ve{u}[k+1]$\\
$\vv_i[k+1] = \vv_i[k] + \frac{\Delta t}{\tau_v} ( -\vv_i[k] + \vff_i[k+1] + \vfb_i[k+1])$ \\
$\vr_i[k+1] = \phi(\vv_i[k+1])$ \\
Buffer forward parameter updates:\\
$\Delta W_i = \Delta W_i + \big(\phi(\vv_i[k+1]) - \phi(\vff_i[k+1])\big)\vr_{i-1}[k+1]^T$ \\
$\Delta \ve{b}_i = \Delta \ve{b}_i + \phi(\vv_i[k+1]) - \phi(\vff_i[k+1])$ \\
}
}
Update forward parameters with $\Delta W_i/K_{\max}$ and $\Delta \ve{b}_i/K_{\max}$ and an optimizer of choice
\caption{Simulation of DFC for training the forward parameters.}
\label{al:DFC_forward}
\end{algorithm}

For DFC-SS, the same simulation strategy is used, with as only difference that the weight updates $\Delta W_i$ only use the network activations of the last simulation step (see Algorithm \ref{al:DFC_SS}). Finally, for DFC-SSA, we directly compute the steady-state solutions according to Lemma \ref{lemma:steady_state_solution_system} (see Algorithm \ref{al:DFC_SSA}).

\begin{algorithm}[h]
Initialize layer activations and parameter update buffers:\\
\For{i in \text{range}(1,L)}{
$\vv_i[1] = \vv_i^-$\\
$\vr_i[1] = \vr_i^-$\\
$\uint[1] = 0$ \\
}
\For{k in \text{range}(\text{1,$K_{\max}$})}{
Update controller:\\
$\ve{e}[k] = \ve{r}_L^* - \ve{r}_L[k]$ \\
$\uint[k+1] = \uint[k] + \frac{\Delta t}{\tau_u} (\ve{e}[k] - \tilde{\alpha} \ve{u}[k])$ \\
$\ve{u}[k+1] = \uint[k+1] + k_p \ve{e}[k] $ \\

Update network:\\
\For{i in \text{range}(1,L)}{
$\vff_i[k+1] = W_i\phi(\ve{v}_{i-1}[k+1]) + \ve{b}_i$\\
$\vfb_i[k+1] = Q_i \ve{u}[k+1]$\\
$\vv_i[k+1] = \vv_i[k] + \frac{\Delta t}{\tau_v} ( -\vv_i[k] + \vff_i[k+1] + \vfb_i[k+1])$ \\
$\vr_i[k+1] = \phi(\vv_i[k+1])$ \\
}
}
Compute forward parameter updates using the last simulation step:\\
$\Delta W_i = \big(\phi(\vv_i[K_{\max}]) - \phi(\vff_i[K_{\max}])\big)\vr_{i-1}[K_{\max}]^T$ \\
$\Delta \ve{b}_i = \phi(\vv_i[K_{\max}]) - \phi(\vff_i[K_{\max}])$ \\
Update forward parameters with $\Delta W_i$ and $\Delta \ve{b}_i$ and an optimizer of choice
\caption{Simulation of DFC-SS for training the forward parameters.}
\label{al:DFC_SS}
\end{algorithm}

\begin{algorithm}[h]
Compute the network Jacobian $J$ \\
$\vdl = \vr_L^* - \vr_L^-$ \\
Compute steady-state solution using Lemma \ref{lemma:steady_state_solution_system}:\\
$\ve{u}_{\sss} = \big(JQ + \tilde{\alpha}I\big)^{-1} \vdl$ \\
$\Delta \vv_{\sss} = Q \ve{u}_{\sss}$ \\

Split $\Delta \vv_{\sss}$ over the layers into $\Delta \vv_{i,\sss}$ \\
Compute steady-state network activations: \\
$\vr_{0,\sss} = \vr_0^-$ \\
\For{i in \text{range}(1,L)}{
$\vv_{i,\sss} = W_i \vr_{i-1, \sss} + \ve{b}_i + \Delta \vv_{i, \sss} $\\
$\vr_{i,\sss} = \phi(\vv_{i,\sss})$\\
}

Compute forward parameter updates using the analytical steady-state solutions:\\
$\Delta W_i = \big(\phi(\vv_{i,\sss}) - \phi(\vff_{i,\sss})\big)\vr_{i-1,\sss}^T$ \\
$\Delta \ve{b}_i = \phi(\vv_{i,\sss}) - \phi(\vff_{i,\sss})$ \\
Update forward parameters with $\Delta W_i$ and $\Delta \ve{b}_i$ and an optimizer of choice

\caption{DFC-SSA iteration for training the forward parameters.}
\label{al:DFC_SSA}
\end{algorithm}

\subsection{Simulating DFC with noisy dynamics for training the feedback weights}
For simulating the noisy dynamics during the training of the feedback weights, we use the Euler-Maruyama method \citep{sarkka2019applied}, which is the stochastic version of the forward Euler method. As discussed in App. \ref{app:fb_learning}, we let white noise $\ves{\xi}$ enter the dynamics of the feedback compartment and we now take a finite time constant $\taufb$ for the feedback compartment, as the instantaneous form with $\taufb \rightarrow 0$ (that we used for simulating the network dynamics without noise) is not well defined when noise enters the dynamics:
\begin{align} \label{eq_app:dynamics_fb_compartment_noisy}
    \taufb \ddt \vfb_i(t) = -\vfb_i(t) + Q_i \ve{u}(t) + \sigma \ves{\xi}_i.
\end{align}
The dynamics for the network then becomes 
\begin{align}
    \tau_v \ddt \vv_i(t) = -\vv_i(t) + W_i \vr_{i-1}(t) + \vfb_i(t),
\end{align}
and, as before, eq. \eqref{eq_app:controller_dynamics_modified} is taken for the controller dynamics. Using the Euler-Maruyama method \citep{sarkka2019applied}, the feedback compartment dynamics \eqref{eq_app:dynamics_fb_compartment_noisy} can be simulated as
\begin{align}
    \vfb_i[k+1] = \vfb_i[k] +  \frac{\Delta t}{\taufb}\big( -\vfb_i[k] + Q_i \ve{u}[k+1]\big) + \frac{\sqrt{\Delta t}}{\taufb} \sigma \Delta \ves{\xi}_i, \quad \Delta \ves{\xi}_i \sim \mathcal{N}(0,I).
\end{align}
As all other dynamical equations do not have noise, their simulation remains equivalent to the simulation with the forward Euler method. Algorithm \ref{al:DFC_feedback} provides the pseudo code of the simulation of DFC during the feedback weight training phase.

\begin{algorithm}[h]
Initialize layer activations and parameter update buffers:\\
\For{i in \text{range}(1,L)}{
$\vv_i[1] = \vv_i^-$\\
$\vr_i[1] = \vr_i^-$\\
$\uint[1] = 0$ \\
$\Delta Q_i = 0 $ \\
}
\For{k in \text{range}(\text{1,$K_{\max}$})}{
Update controller:\\
$\ve{e}[k] = \ve{r}_L^- - \ve{r}_L[k]$ \\
$\uint[k+1] = \uint[k] + \frac{\Delta t}{\tau_u} (\ve{e}[k] - \tilde{\alpha} \ve{u}[k])$ \\
$\ve{u}[k+1] = \uint[k+1] + k_p \ve{e}[k] $ \\

Update network:\\
\For{i in \text{range}(1,L)}{
$\vff_i[k+1] = W_i\phi(\ve{v}_{i-1}[k+1]) + \ve{b}_i$\\
Sample noise and let it enter in the feedback compartment with non-instantaneous dynamics:\\
$\Delta \ves{\xi}_i \sim \mathcal{N}(0,I)$ \\
$\vfb_i[k+1] = \vfb_i[k] +  \frac{\Delta t}{\taufb}\big( -\vfb_i[k] + Q_i \ve{u}[k+1]\big) + \frac{\sqrt{\Delta t}}{\taufb} \sigma \Delta \ves{\xi}_i$ \\
$\vv_i[k+1] = \vv_i[k] + \frac{\Delta t}{\tau_v} ( -\vv_i[k] + \vff_i[k+1] + \vfb_i[k+1])$ \\
$\vr_i[k+1] = \phi(\vv_i[k+1])$ \\
Buffer feedback weight updates:\\
$\Delta Q_i = \Delta Q_i - \vfb_i[k] \ve{u}[k+1] - \beta Q_i$ \\
}
}
Update feedback parameters with $\Delta Q_i/K_{\max}$ and an optimizer of choice
\caption{Simulation of DFC for training the feedback weights.}
\label{al:DFC_feedback}
\end{algorithm}

\section{Experiments} \label{app:experiments}
\subsection{Description of the alignment measures}\label{app:alignment_measures}
In this section, we describe the alignment measures used in Fig. \ref{fig:toy_experiment_tanh} in detail. 

\paragraph{Condition 2.} Fig. \ref{fig:toy_experiment_tanh}A describes how well the network satisfies Condition \ref{con:Q_GN}. For this, we project $Q$ onto the column space of $J^T$, for which we use a projection matrix $P_{J^T}$:
\begin{align}
    P_{J^T} Q = J^T (JJ^T)^{-1} J Q.
\end{align}
Then, we compare the Frobenius norm of the projection of $Q$ with the norm of $Q$, via its ratio:
\begin{align}
    \mathrm{ratio}_{\mathrm{Con2}} = \frac{\|P_{J^T} Q\|_F}{\|Q\|_F}.
\end{align}
Notice that a $\mathrm{ratio}_{\mathrm{Con2}}=1$ indicates that the column space of $Q$ lies fully inside the column space of $J^T$, hence indicating that Condition \ref{con:Q_GN} is satisfied.\footnote{Note that in degenerate cases, $Q$ could be lower rank and still have $\mathrm{ratio}_{\mathrm{Con2}}=1$ if its (reduced) column space lies inside the column space of $J^T$. As $Q$ is a skinny matrix, we assume it is always of full rank and do not consider this degenerate scenario.} At the opposite extreme, $\mathrm{ratio}_{\mathrm{Con2}}=0$ indicates that the column space of Q is orthogonal on the column space of $J^T$.

\paragraph{Condition 1.} Fig. \ref{fig:toy_experiment_tanh}C describes how well the network satisfies Condition \ref{con:magnitude_r}. This condition states that all layers (except the output layer) have an equal $L^2$ norm. To measure how well Condition \ref{con:magnitude_r} is satisfied, we compute the standard deviation of the layer norms over the layers, and normalize it by the average layer norm:
\begin{align}
    \mathrm{ratio}_{\mathrm{Con1}} &= \frac{\frac{1}{L}\sum_{i=0}^L \big(\|\vr_i\|_2 - \mathrm{mean}(\|\vr\|_2)\big)^2}{\mathrm{mean}(\|\vr\|_2)}\\
    \mathrm{mean}(\|\vr\|_2) &= \frac{1}{L}\sum_{i=0}^L\|\vr_i\|_2
\end{align}
We take $\vr_i = \vr_i^-$ to compute this measure, but other values of $\vr_i$ during the dynamics would also work, as they remain close together for a small target stepsize $\lambda$. Now, notice that $\mathrm{ratio}_{\mathrm{Con1}}=0$ indicates perfect compliance with Condition \ref{con:magnitude_r}, as then all layers have the same norm, and $\mathrm{ratio}_{\mathrm{Con1}}=1$ indicates that the layer norms vary by $\mathrm{mean}(\|\vr\|_2)$ on average, hence indicating that Condition \ref{con:magnitude_r} is not at all satisfied.

\paragraph{Stability measure.} Fig. \ref{fig:toy_experiment_tanh}E describes the stability of DFC during training. For this, we plot the maximum real part of the eigenvalues of the total system matrix $A_{PI}$ around the steady state (see eq. \eqref{eq_app:a_pi}), which describes the dynamics of DFC around the steady state (incorporating $k_p$ and the actual time constants, in contrast to Condition \ref{con:local_stability}).  

\paragraph{Alignment with MN updates.} Fig. \ref{fig:toy_experiment_tanh}B describes the alignment of the DFC updates with the ideal weighted MN updates. The MN updates are computed as follows:
\begin{align}\label{eq_app:MN_updates}
    \Delta \bar{W}^{\mathrm{MN}} = R J^{\dagger} \vdl,
\end{align}

with $R$ defined in eq. \eqref{eq:R_matrix} and $\bar{W}$ the concatenated vectorized form of all weights $W_i$. For the alignment measurements in the computer vision experiments (see Section \ref{sec:alignment_computer_vision}) we use a damped variant of the MN updates:
\begin{align}\label{eq_app:damped_MN_updates}
    \Delta \bar{W}^{\mathrm{MN}} = R J^{T}(JJ^T+\gamma I)^{-1} \vdl,
\end{align}
with $\gamma$ some positive damping constant.
The damping constant is needed to incorporate the damping effect of the leakage constant, $\alpha$, into the dynamical inversion, but also to reflect an \textit{implicit damping} effect. \citet{meulemans2020theoretical} showed that introducing a higher damping constant, $\gamma$, in the pseudoinverse \eqref{eq_app:MN_updates} reflected better the updates made by TP, which uses learned inverses. We found empirically that a higher damping constant, $\gamma$, also reflects better the updates made by DFC. Using a similar argumentation, we hypothesize that this implicit damping in DFC originates from the fact that, in nonlinear networks, $J$ changes for each batch sample and hence $Q$ cannot satisfy Condition \ref{con:Q_GN} for each batch sample. Consequently, $Q$ tries to satisfy Condition \ref{con:Q_GN} as good as possible for all batch samples, but does not satisfy it perfectly, resulting in a phenomenon that can be partially described by implicit damping.

\paragraph{Alignment with GN updates.} Fig. \ref{fig:toy_experiment_tanh}D describes the alignment of the DFC updates with the ideal GN updates. The GN updates are computed as follows:
\begin{align}\label{eq_app:GN_updates}
    \Delta \bar{W}^{\mathrm{GN}} = J_{\bar{W}}^{\dagger}\vdl,
\end{align}
with $J_{\bar{W}} = \frac{\partial \vr_L^-}{\partial \bar{W}}$, evaluated at the feedforward activations $\vr_i^-$. Similarly to the MN updates, we also introduce a damped variant of the GN updates, which is used in the computer vision alignment experiments (Section \ref{sec:alignment_computer_vision}):
\begin{align}\label{eq_app:damped_GN_updates}
    \Delta \bar{W}^{\mathrm{GN}} = J_{\bar{W}}^T(J_{\bar{W}}J_{\bar{W}}^T + \gamma I)^{-1}\vdl,
\end{align}
where the damping constants, $\gamma$ and $\alpha$, reflect the leakage constant and the implicit damping effects, respectively.

\paragraph{Alignment with DFC-SSA updates.} Finally, Fig. \ref{fig:toy_experiment_tanh}F describes the alignment of the DFC updates with the DFC-SSA updates which use the linearized analytical steady-state solution of the dynamics. The DFC-SSA updates are computed as follows (see also Algorithm \ref{al:DFC_SSA}):
\begin{align}\label{eq_app:DFC_SSA_updates}
    \Delta \bar{W}^{\mathrm{SSA}} = R_{\sss} Q(JQ + \tilde{\alpha} I)^{-1} \vdl,
\end{align}
with $R_{\sss}$ defined in eq. \eqref{eq:R_matrix} but with the steady-state values $\ve{r}_{i,\sss}$ instead of $\vr_i^-$.

\subsection{Description of training}\label{app:description_training}

\paragraph{Training phases.}
We iterate between one epoch of training the forward weights and $X$ epochs of training the feedback weights, with $X\in [1,2,3]$ a hyperparameter. The extra epochs for training the feedback weights enable the feedback weights to better satisfy Conditions \ref{con:Q_GN} and \ref{con:local_stability} when the forward weights are changing fast (e.g., during early training), and slightly improve the performance of DFC. Before the training starts, we pre-train the feedback weights for 10 epochs, starting from a random configuration, to ensure that the network is stable when the training begins and Condition \ref{con:Q_GN} is approximately satisfied. 

\paragraph{Student-teacher toy regression.} For the toy experiments of Fig. \ref{fig:toy_experiment_tanh}, we use the student-teacher regression paradigm. Here, a randomly initialized teacher generates a synthetic regression dataset using random inputs. A separate randomly initialized student is then trained on this synthetic dataset. We used more hidden layers and neurons for the teacher network compared to the student network, such that the student network cannot get `lucky' by being initialized close to the teacher network.

\paragraph{Optimizer.} In student-teacher toy regression experiments, we use vanilla SGD without momentum as an optimizer. In the computer vision experiments, we use a separate Adam optimizer \citep{kingma2014adam} for the forward and feedback weights, as this improves training results compared to vanilla SGD. As Adam was designed for BP updates, it will likely not be an optimal optimizer for DFC, which uses MN updates. An interesting future research direction is to design new optimizers that are tailored towards the MN updates of DFC, to further improve its performance. We used gradient clipping for all DFC experiments to prevent too large updates when the inverse of $J$ is poorly conditioned.

\paragraph{Training length and reported test results.} For the classification experiments, we used 100 epochs of training for the forward weights (and a corresponding amount of feedback training epochs, depending on $X$). As the autoencoder experiment was more resource-intensive, we trained the models for only 25 epochs there, as this was sufficient for getting near-perfect autoencoding performance when visually inspected (see Fig. \ref{fig:autoencoding_im}). For all experiments, we split the 60000 training samples into a validation set of 5000 samples and a training set of 55000 samples. The hyperparameter searches are done based on the validation accuracy (validation loss for MNIST-autoencoder and train loss for MNIST-train) and we report the test results corresponding to the epoch with best validation results in Table \ref{tab:test_results}.

\paragraph{Weight initializations.}
All network weights are initialized with the Glorot-Bengio normal initialization \citep{glorot2010understanding}, except when stated otherwise.

\paragraph{Initialization of the fixed feedback weights.}
For the variants of DFC with fixed feedback weights, we use the following initialization:
\begin{align}\label{eq_app:wpi}
    Q_i &= \prod_{k=i+1}^L W_k^T, \quad 1\leq i \leq L-1 \\
    Q_L &= I
\end{align}
For $\tanh$ networks, this initialization approximately satisfies Conditions \ref{con:Q_GN} and \ref{con:local_stability} at the beginning of training. This is because $Q$ will approximate $J^T$, as the forward weights are initialized by Glorot-Bengio normal initialization \citep{glorot2010understanding}, and the network will consequently be in the approximate linear regime of the $\tanh$ nonlinearities. 

\paragraph{Freeze $Q_L$.} For the MNIST-autoencoder experiments, we fixed the output feedback weights to $Q_L = I$, i.e., one-to-one connections between $\vr_L and \ve{u}$. As we did not train $Q_L$, we also did not introduce noise in the output layer during the training of the feedback weights. Freezing $Q_L$ prevents the noise in the high-dimensional output layer from burying the noise information originating from the small bottleneck layer and hence enabling better feedback weight training. This measure modestly improved the performance of DFC on MNIST-autoencoder (without fixing $Q_L$, the performance of all DFC variants was around 0.13 test loss -- c.f. Table \ref{tab:test_results} -- which is not a big decrease in performance). Freezing $Q_L$ does not give us any advantages over BP or DFA, as these methods implicitly assume to have direct access to the output error, i.e., also having fixed feedback connections between the error neurons and output neurons equal to the identity matrix. We provided the option to freeze $Q_L$ into the hyperparameter searches of all experiments but this is not necessary for optimal performance of DFC in general, as this option was not always selected by the hyperparameter searches.

\paragraph{Double precision.} We noticed that the standard data type \texttt{float32} of PyTorch \citep{paszke2017automatic} caused numerical errors to appear during the last epochs of training when the output error $\vdl$ is very small. For small $\vdl$, the difference $\phi(\vv_i) - \phi(\vff_i)$ in the forward weight updates \eqref{eq:W_dynamics} is very small and can result in numerical underflow. We solved this numerical problem by using \texttt{float64} (double precision) as data type.

\subsection{Architecture details}
We use fully connected (FC) architectures for all experiments. 
\begin{itemize}
    \item Classification experiments (MNIST, Fashion-MNIST, MNIST-train): 3 FC hidden layers of 256 neurons with $\tanh$ nonlinearity and 1 softmax output layer of 10 neurons. 
    \item MNIST-autoencoder: 256-32-256 FC hidden layers with tanh-linear-tanh nonlinearities and a linear output layer of 784 neurons.
    \item Student-teacher regression (Fig. \ref{fig:toy_experiment_tanh}): 2 FC hidden layers of 10 neurons and tanh nonlinearities, a linear output layer of 5 neurons, and input dimension 15. 
\end{itemize}

\paragraph{Absorbing softmax into the cross-entropy loss.}
For the classification experiments (MNIST, Fashion-MNIST, and MNIST-train), we used a softmax output nonlinearity in combination with the cross-entropy loss. As the softmax nonlinearity and cross-entropy loss cancel out each others curvatures originating from the exponential and log terms, respectively, it is best to combine them into one output loss:
\begin{align}
    \Lagr^{\text{combined}} = -\sum_{b=1}^B \ve{y}^{(b)T}\log\big(\text{softmax}(\ve{r}_L^{(b)})\big),
\end{align}
with $\ve{y}^{(b)}$ the one-hot vector representing the class label of sample $b$, and $\log$ the element-wise logarithm. Now, as the softmax is absorbed into the loss function, the network output $\ve{r}_L$ can be taken linear and the output target is computed with eq. \eqref{eq:output_target} using $\Lagr^{\text{combined}}$.

\subsection{Hyperparameter searches}
All hyperparameter searches were based on the best validation accuracy (best validation loss for MNIST-autoencoder and last train loss for MNIST-train) over all training epochs, using 5000 validation datasamples extracted from the training set. We use the Tree of Parzen Estimators hyperparameter optimization algorithm \citep{bergstra2011algorithms} based on the Hyperopt \citep{bergstra2013hyperopt} and Ray Tune \citep{liaw2018tune} Python libraries.

Due to the heavy computational cost of simulating DFC, we performed only hyperparameter searches for DFC-SSA, DFC-SSA (fixed), BP and DFA (200 hyperparameter samples for all methods). We used the hyperparameters found for DFC-SSA and DFC-SSA (fixed) for DFC and DFC-SS, and DFC (fixed) and DFC-SS (fixed), respectively, together with standard simulation hyperparameters for the forward weight training that proved to work well ($k_p=2$, $\tau_u=1$, $\tau_v=0.2$, forward Euler stepsize $\Delta t=0.02$ and 1000 simulation steps). 

Tables \ref{tab:hp_symbols} and \ref{tab:hp_DFC} provide the hyperparameters and search intervals that we used for DFC-SSA in all experiments. We included the simulation hyperparameters for the feedback training phase in the search to prevent us from fine-tuning the simulations by hand. Note that we use different simulation hyperparameters for the forward training phase (see paragraph above) and the feedback training phase (see Table \ref{tab:hp_DFC}). This is because the simulation of the feedback training phase needs a small stepsize, $\Delta t_{\mathrm{fb}}$, and a small network time constant, $\tau_{v}$, to properly simulate the stochastic dynamics. For the forward phase, however, we need to simulate over a much longer time interval, so taking small $\Delta t$ and $\tau_{v}$\footnote{The simulation stepsize, $\Delta t$, needs to be smaller than the time constants.} would be too resource-intensive. When using $k_p=2$, $\tau_u=1$, and $\tau_v=0.2$ during the simulation of the forward training phase, much bigger timesteps such as $\Delta t=0.02$ can be used. Note that these simulation parameters do not change the steady state of the controller and network, as $\tilde{\alpha}$ is independent from $k_p$ in our implementation. We also differentiated $\tilde{\alpha}$ in the forward training phase from $\tilde{\alpha}_{\mathrm{fb}}$ in the feedback training phase, as the theory predicted that a bigger leakage constant is needed during the feedback training phase in the first epochs. However, toy simulations in Section \ref{app:fb_learning} suggest that the feedback learning also works for smaller $\tilde{\alpha}$, which we did not explore in the computer vision experiments. Finally, we used $\mathrm{lr}\cdot \lambda$ and $\lambda$ as hyperparameters in the search instead of $\mathrm{lr}$ and $\lambda$ separately, as $\mathrm{lr}$ and $\lambda$ have a similar influence on the magnitude of the forward parameter updates. The specific hyperparameter configurations for all experiments can be found in our codebase.\footnote{PyTorch implementation of all methods is available at \url{https://github.com/meulemansalex/deep_feedback_control}.}

\begin{table}[H]
\centering
\vspace*{-0.4cm}
\caption{Hyperparameter symbols and meaning.}
\label{tab:hp_symbols}
\begin{tabular}{c|l}
\toprule
Symbol & \multicolumn{1}{|c}{Hyperparameter}\\
\midrule
$\text{lr}$ & Learning rate of the Adam optimizer for the forward parameters\\
$\epsilon$ & Parameter of the Adam optimizer for the forward parameters\\
$\tilde{\alpha}$ & Leakage term of the controller dynamics \eqref{eq_app:controller_dynamics_modified} during training of the forward weights\\
$\lambda$ & Output target stepsize (see \eqref{eq:output_target})\\
$\text{lr}_{\text{fb}}$ & Learning rate of the Adam optimizer for the feedback parameters\\
$\text{lr}_{\text{fb, pre-train}}$ & Learning rate of the Adam optimizer for the feedback parameters during pre-training\\
$\epsilon_{\text{fb}}$ & Parameter of the Adam optimizer for the feedback parameters\\
$\tilde{\alpha}_{\text{fb}}$ & Leakage term of the controller dynamics \eqref{eq_app:controller_dynamics_modified} during the training of the feedback weights \\
$\beta$ & Weight decay term for the feedback weights \\
$k_{p,\text{fb}}$ & Proportional control constant during training of the feedback weights\\
$\tau_v$ & Time constant of the network dynamics, during training of the feedback weights \\
$\taufb$ & Time constant of the feedback compartment during feedback weight training \\
$\sigma$ & Standard deviation of the noise perturbation during training of the feedback weights\\
$X$ & Number of feedback training epochs after each forward training epoch\\
$\Delta t_{\mathrm{fb}}$ & Stepsize for simulating the dynamics during feedback weight training\\
$t_{\max, \mathrm{fb}}$ & Number of simulation steps during feedback weight training\\
freeze\_$Q_L$ & Flag for fixing the feedback weights $Q_L$ to the identity matrix\\
\bottomrule
\end{tabular}
\vspace*{-0.4cm}
\end{table}

\begin{table}[H]
\centering
\vspace*{-0.4cm}
\caption{Hyperparameter search intervals for DFC.}
\label{tab:hp_DFC}
\begin{tabular}{c|c}
\toprule
Hyperparameter & Search interval   \\
\midrule
$\text{lr} \cdot \lambda$ & $[10^{-7}: 10^{-4}]$ \\
$\epsilon$ & $[10^{-8}: 10^{-5}]$\\
$\tilde{\alpha}$ & $[10^{-5}: 10^{-1}]$\\
$\lambda$ &  $[10^{-3}: 10^{-1}]$\\
$\text{lr}_{\text{fb}}$ & $[1\cdot 10^{-6}: 5\cdot 10^{-4}]$\\
$\text{lr}_{\text{fb, pre-train}}$ & $[5\cdot 10^{-5}: 1\cdot 10^{-3}]$\\
$\epsilon_{\text{fb}}$ & $[10^{-8}: 10^{-5}]$\\
$\tilde{\alpha}_{\text{fb}}$ & $[0.2: 1.]$\\
$\beta$ & $\{0;10^{-5};10^{-3};10^{-1}\}$ \\
$k_{p,\text{fb}}$ & $[0: 0.2]$\\
$\tau_v$ & $[5\cdot 10^{-3}: 1.5\cdot 10^{-2}]$\\
$\taufb$ & $[0.1: 0.5]$\\
$\sigma$ & $[10^{-3}: 10^{-1}]$\\
$X$ & $\{1;2;3\}$ \\
$\Delta t_{\mathrm{fb}}$ & $[1\cdot 10^{-3}: 5\cdot 10^{-3}]$\\
$t_{\max, \mathrm{fb}}$ & $[30: 60]$\\
freeze\_$Q_L$ & $\{\mathrm{True}; \mathrm{False}\}$\\
\bottomrule
\end{tabular}
\vspace*{-0.4cm}
\end{table}

\begin{table}[H]
\centering
\vspace*{-0.4cm}
\caption{Hyperparameter search intervals for BP and DFA.}
\label{tab:hp_BP}
\begin{tabular}{c|c}
\toprule
Hyperparameter   & Search interval   \\
\midrule
$\text{lr}$ & $[10^{-8}: 10^{-2}]$ \\
$\epsilon$ & $[10^{-8}: 10^{-5}]$\\
\bottomrule
\end{tabular}
\vspace*{-0.4cm}
\end{table}

\subsection{Extended experimental results}
In this section, we provide extra experimental results accompanying the results of Section \ref{sec:experiments}.

\subsubsection{Training losses of the computer vision experiments}
Table \ref{tab:loss_train_results} provides the best training loss over all epochs for all the considered computer vision experiments. Comparing the train losses with the test performances in Table \ref{tab:test_results}, shows that good test performance is not only caused by good optimization properties (i.e., low train loss) but also by other mechanisms, such as implicit regularization. The distinction is most pronounced in the results for MNIST. These results highlight the need to disentangle optimization from implicit regularization mechanisms to study the learning properties of DFC, which we do in the MNIST-train experiments provided in Table \ref{tab:test_results}.

\begin{table}[h]
\centering
% \vspace*{-0.2cm}
\caption[Train errors]{Best training loss over a training of 100 epochs (classification) or 25 epochs (autoencoder). We use the Adam optimizer \citep{kingma2014adam}. Architectures: 3x256 fully connected (FC) tanh hidden layers and softmax output (classification), 256-32-256 FC hidden layers for autoencoder MNIST with tanh-linear-tanh nonlinearities, and a linear output. Mean $\pm$ std (5 random seeds).}
\label{tab:loss_train_results}
\begin{tabular}{*5c}
\toprule

{}  & MNIST & Fashion-MNIST & MNIST-autoencoder \\
\midrule
BP & $7.51^{\pm 1.06}\cdot 10^{-4}$ & $1.19^{\pm 0.064}\cdot 10^{-2}$ &  $9.57^{\pm0.080}\cdot 10^{-2}$ \\
\midrule
DFC & $1.93^{\pm 3.83}\cdot 10^{-4}$ & $8.84^{\pm 0.87}\cdot 10^{-2}$ & $1.16^{\pm 0.023}\cdot 10^{-1}$ \\
DFC-SSA & $2.51^{\pm 1.76}\cdot 10^{-5}$ & $1.74^{\pm 0.037}\cdot 10^{-1}$ & $1.16^{\pm 0.009}\cdot 10^{-1}$ \\
DFC-SS & $7.71^{\pm 5.04}\cdot 10^{-6}$ & $1.72^{\pm 0.022}\cdot 10^{-1}$ & $1.14^{\pm 0.006}\cdot 10^{-1}$\\
DFC (fixed) & $9.04^{\pm 3.87}\cdot 10^{-4}$ & $2.34^{\pm 0.46}\cdot 10^{-1}$ & $3.34^{\pm 0.06}\cdot 10^{-1}$\\
DFC-SSA (fixed) & $1.32^{\pm 0.22}\cdot 10^{-3}$ & $1.69^{\pm 0.24}\cdot 10^{-1}$ & $3.20^{\pm 0.041}\cdot 10^{-1}$\\
DFC-SS (fixed) & $1.66^{\pm 0.44}\cdot 10^{-3}$ & $1.30^{\pm 0.11}\cdot 10^{-1}$ & $3.23^{\pm 0.04}\cdot 10^{-1}$\\
DFA & $3.59^{\pm 0.14}\cdot 10^{-4}$ & $6.43^{\pm 0.37}\cdot 10^{-2}$ & $3.05^{\pm 0.021}\cdot 10^{-1}$\\
\bottomrule
\end{tabular}
\vspace*{-0.2cm}
\end{table}

\subsubsection{Alignment plots for the toy experiment}
Here, we show the alignment of the methods used in the toy experiments of Fig. \ref{fig:toy_experiment_tanh} with MN updates and compare it with the alignment with BP updates. We plot the alignment angles per layer to investigate whether the alignment differs between layers. Fig. \ref{fig:experiment_gnt_angles} shows the alignment of all methods with the damped MN angles and Fig. \ref{fig:experiment_bp_angles} with the BP angles. We see clearly that the alignment with MN angles is much better for the DFC variants with trained feedback weights compared to the alignment with BP angles, hence indicating that DFC uses a fundamentally different approach to learning, compared to BP, and thereby confirming the theory. 

\subsubsection{Alignment plots for computer vision experiments}\label{sec:alignment_computer_vision}
Figures \ref{fig:mnist_gnt_angles} and \ref{fig:mnist_bp_angles} show the alignment of all methods with MN and BP updates, respectively. In contrast to the toy experiments in the previous section, now the alignment with BP is much closer to the alignment with MN updates. There are two main reasons for this. First, the classification networks we used have big hidden layers and a small output layer. In this case, the network Jacobian $J$ has many rows and only very few columns, which causes $J^{\dagger}$ to approximately align with $J^{T}$ (see among others Theorem S12 in \citet{meulemans2020theoretical}). Hence, the BP updates will also approximately align with the MN updates, explaining the better alignment with BP updates on MNIST compared to the toy experiments. Secondly, due to the nonlinearity of the network, $J$ changes for each datasample and $Q$ cannot satisfy Condition \ref{con:Q_GN} exactly for all datasamples. We try to model this effect by introducing a higher damping constant, $\gamma=1$, for computing the ideal damped MN updates (see Section \ref{app:alignment_measures}). However, this higher damping constant is not a perfect model for the phenomena occurring. Consequently, the alignment of DFC with the damped MN updates is suboptimal and a better alignment could be obtained by introducing other variants of MN updates that more accurately describe the behavior of DFC on nonlinear networks.\footnote{Now, we perform a small grid-search to find a $\gamma \in \{0, 10^{-5}, 10^{-4}, 10^{-3}, 10^{-2}, 10^{-1}, 1, 10\}$ that best aligns with the DFC and DFA updates after 3 epochs of training. As this is a very coarse-grained approach, better alignment angles with damped MN updates could be obtained by a more fine-tuned approach for finding an optimal $\gamma$.} Note that nonetheless, the alignment with MN updates is better compared to the alignment with BP updates. 

Surprisingly, for Fashion-MNIST and MNIST-autoencoder, the DFC updates in the last and penultimate layer align better with BP than with MN updates (see Figures \ref{fig:fashion_mnist_bp_angles}-\ref{fig:mnist_auto_gnt_angles}). One notable difference between the configurations used for MNIST on the one hand and Fashion-MNIST and MNIST-autoencoder on the other hand, is that the hyperparameter search selected for the latter two to fix the output feedback weights $Q_L$ to the identity matrix (see Section \ref{app:description_training} for a description and discussion). This freezing of the output feedback weights slightly improved the performance of the DFC methods. Freezing $Q_L$ to the identity matrix explains why the output weight updates align closely with BP, as the postsynaptic plasticity signal is now an integrated plus proportional version of the output error. However, it is surprising that the alignment in the penultimate layer is also changed significantly. We hypothesize that this is due to the fact that the feedback learning rule \eqref{eq:Q_dynamics} was designed for learning all feedback weights (leading to Theorem \ref{theorem:fb_learning_simplified}) and that freezing $Q_L$ breaks this assumption. However, extra investigation is needed to fully understand the occurring phenomena.

\clearpage
\begin{sidewaysfigure}[h]    
    \centering
    \begin{subfigure}{\linewidth}
        \includegraphics[width=\linewidth]{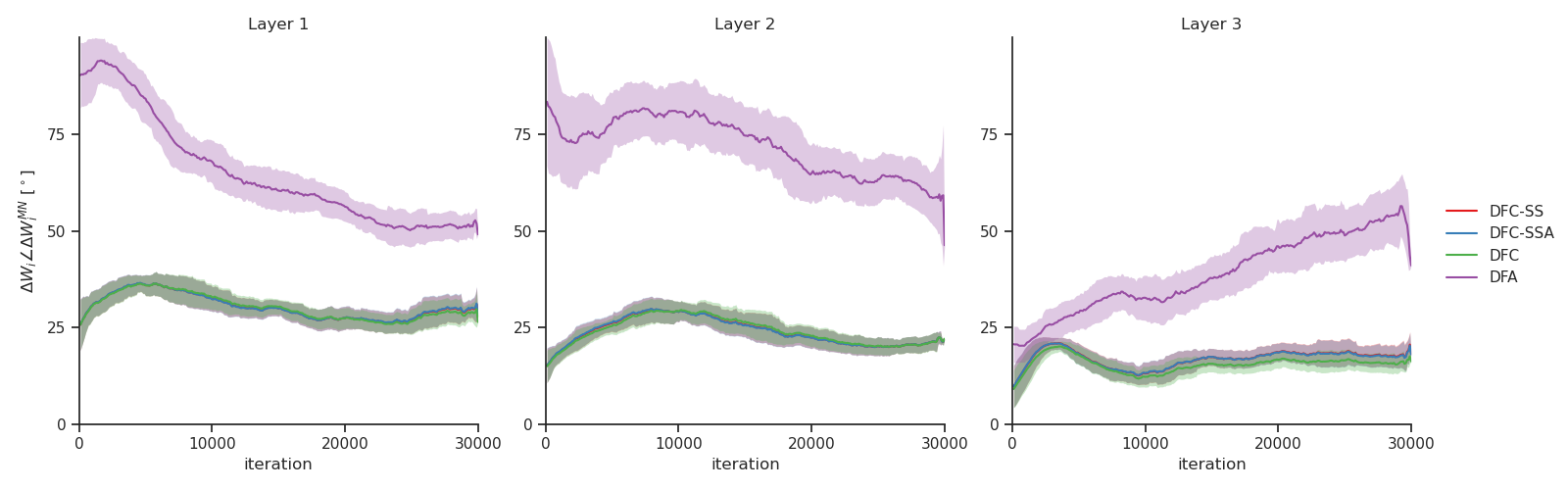}
        \caption{Angles between $\Delta W_i$ and the damped MN updates.}
    \end{subfigure}
     \begin{subfigure}{\linewidth}
        \includegraphics[width=\linewidth]{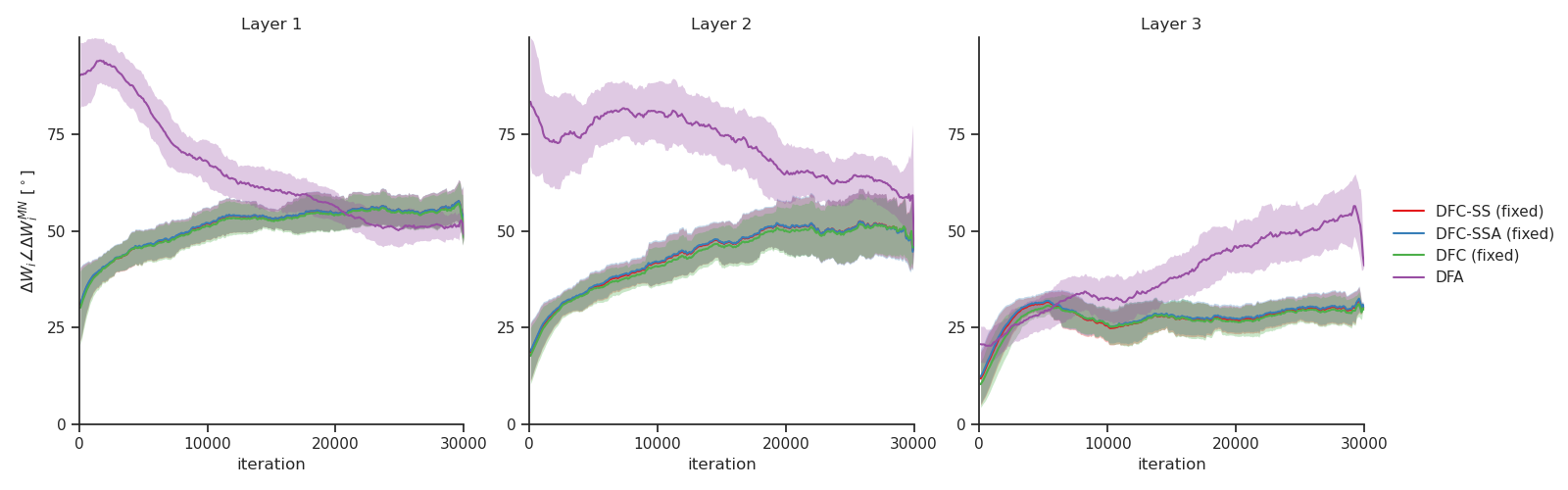}
        \caption{Angles between $\Delta W_i$ with fixed feedback weights and the damped MN updates.}
    \end{subfigure}
    \caption{Angles between the damped MN updates with $\gamma=0.1$ \eqref{eq_app:MN_updates} and the ones computed by DFC-SS, DFC-SSA, DFC, and DFA, plotted for all hidden layers with (a) learned feedback weights and (b) fixed feedback weights on the toy experiment explained in Figure \ref{fig:toy_experiment_tanh}. A window-average is plotted together with the window-std (shade). The x-axis iterations corresponds to the minibatches processed. The curves of DFC, DFC-SS, and DFC-SSA overlap in some of the plots.}
    \label{fig:experiment_gnt_angles}
\end{sidewaysfigure}

\clearpage
\begin{sidewaysfigure}[h]
    \centering
    \begin{subfigure}{\linewidth}
        \includegraphics[width=\linewidth]{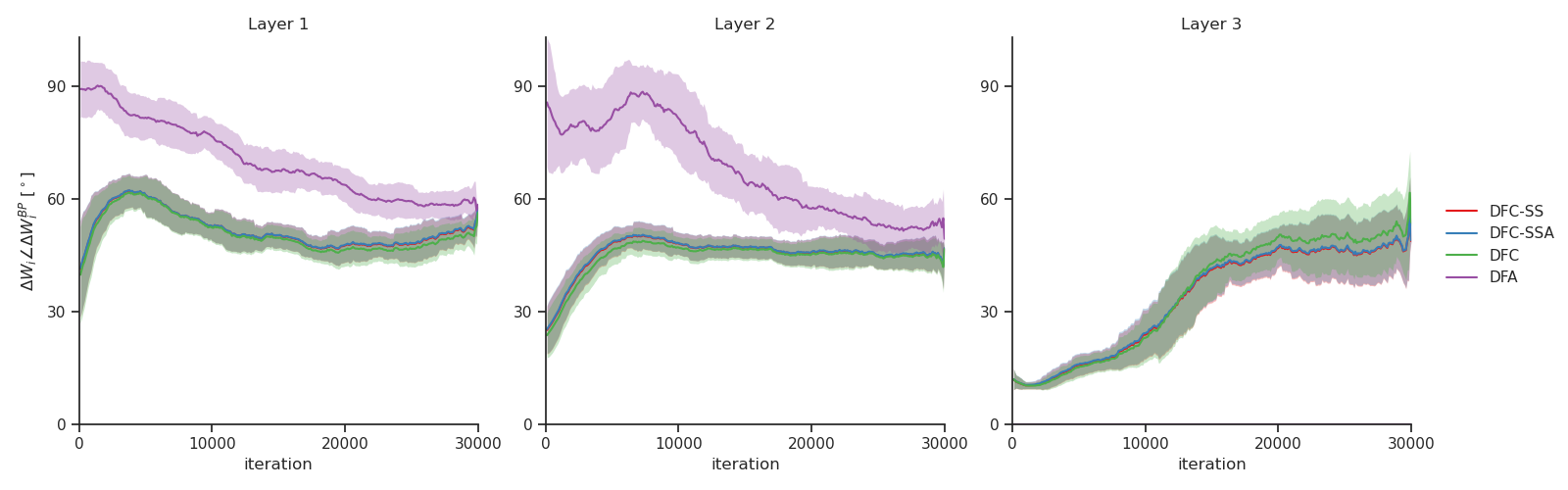}
        \caption{Angles between $\Delta W_i$ and the BP updates.}
    \end{subfigure}
     \begin{subfigure}{\linewidth}
        \includegraphics[width=\linewidth]{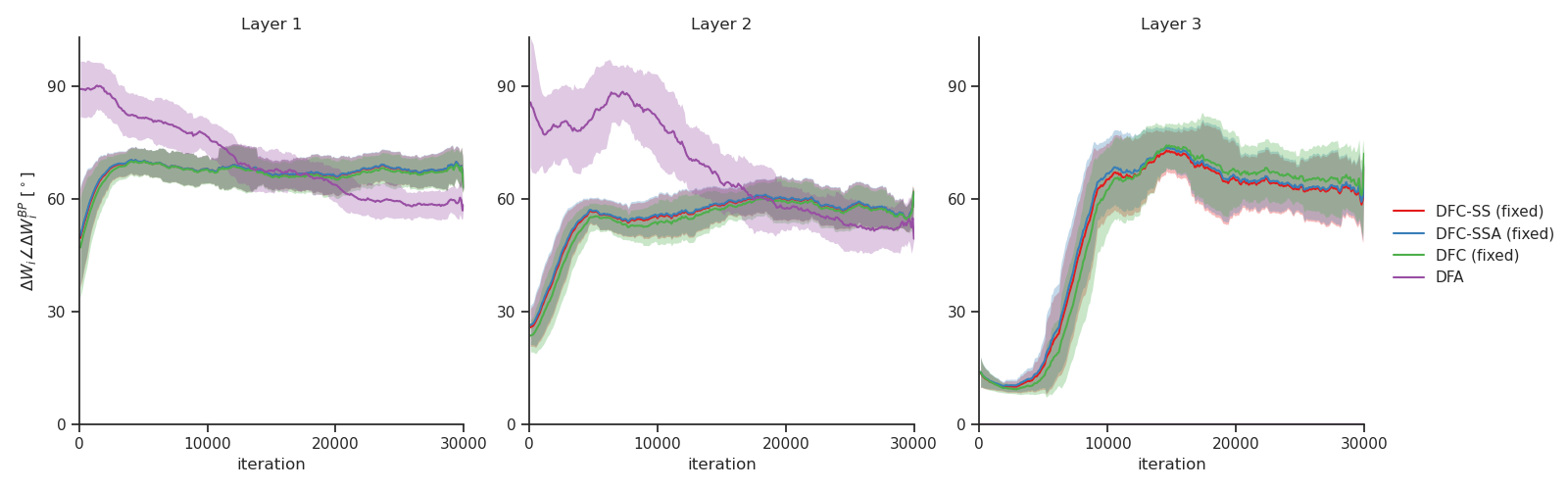}
        \caption{Angles between the BP updates and $\Delta W_i$ with fixed feedback weights.}
    \end{subfigure}
    \caption{Angles between the weight updates $\Delta W_i$ computed by BP and the ones computed by DFC-SS, DFC-SSA, DFC, and DFA, plotted for all hidden layers with (a) learned feedback weights (b) fixed feedback weights on the toy experiment explained in Figure \ref{fig:toy_experiment_tanh}. A window-average is plotted together with the window-std (shade). The x-axis iterations corresponds to the minibatches processed. The curves of DFC, DFC-SS and DFC-SSA overlap in some of the plots.}
    \label{fig:experiment_bp_angles}
\end{sidewaysfigure}

% MNIST----------------------------------------
\clearpage
\begin{sidewaysfigure}[h]    
    \centering
    \begin{subfigure}{\linewidth}
        \includegraphics[width=\linewidth]{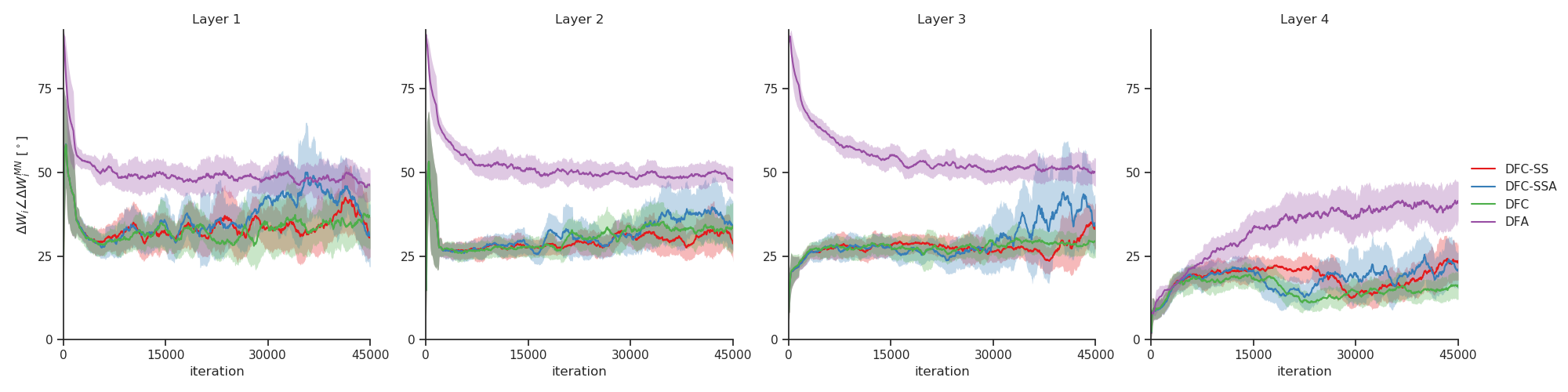}
        \caption{Angles between $\Delta W_i$ and the damped MN updates for MNIST.}
    \end{subfigure}
     \begin{subfigure}{\linewidth}
        \includegraphics[width=\linewidth]{figures/pipeline_fixed_TE1/gnt_angles_subplots.png}
        \caption{Angles between $\Delta W_i$ and the damped MN updates for MNIST with fixed feedback weights.}
    \end{subfigure}
    \caption{Angles between the damped MN updates with $\gamma=1$ \eqref{eq_app:MN_updates} and the ones computed by DFC-SS, DFC-SSA, DFC, and DFA, plotted for all hidden layers with (a) learned feedback weights and (b) fixed feedback weights on MNIST. A window-average is plotted together with the window-std (shade). The x-axis iterations corresponds to the minibatches processed. The curves of DFC, DFC-SS, and DFC-SSA overlap in some of the plots.}
    \label{fig:mnist_gnt_angles}
\end{sidewaysfigure}

\clearpage
\begin{sidewaysfigure}[h]
    \centering
    \begin{subfigure}{\linewidth}
        \includegraphics[width=\linewidth]{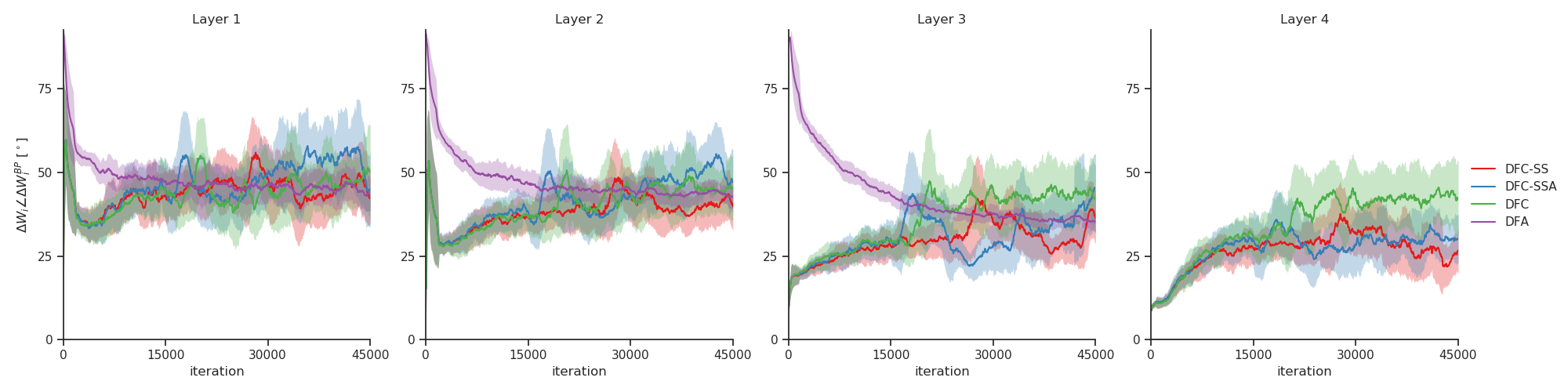}
        \caption{Angles between $\Delta W_i$ and the BP updates for MNIST.}
    \end{subfigure}
     \begin{subfigure}{\linewidth}
        \includegraphics[width=\linewidth]{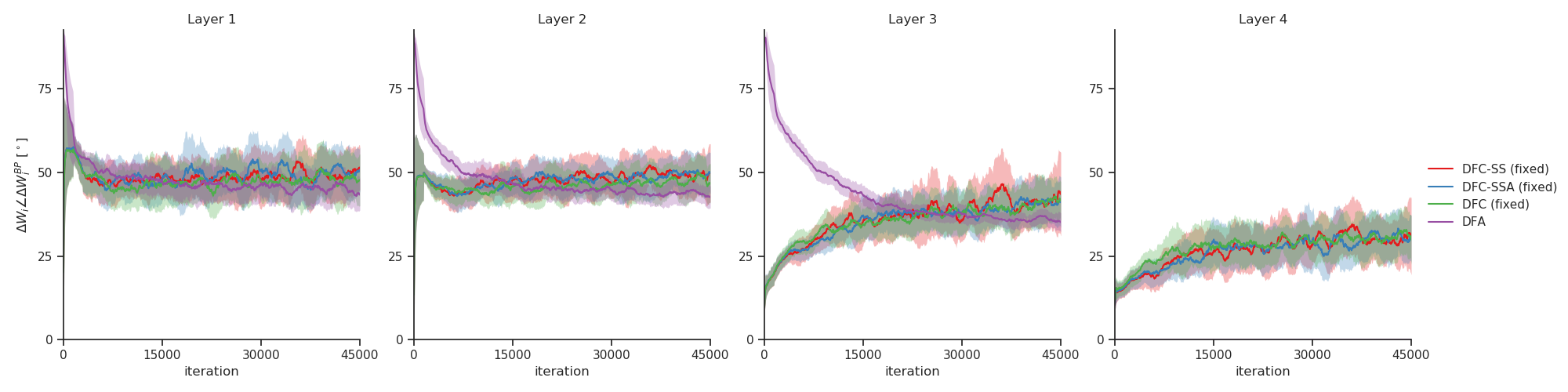}
        \caption{Angles between $\Delta W_i$ and the BP updates for MNIST with fixed feedback weights.}
    \end{subfigure}
    \caption{Angles between the weight updates $\Delta W_i$ computed by BP and the ones computed by DFC-SS, DFC-SSA, DFC, and DFA, plotted for all hidden layers with (a) learned feedback weights (b) fixed feedback weights on MNIST. A window-average is plotted together with the window-std (shade). The x-axis iterations corresponds to the minibatches processed. The curves of DFC, DFC-SS, and DFC-SSA overlap in some of the plots.}
    \label{fig:mnist_bp_angles}
\end{sidewaysfigure}

% Fashion-MNIST----------------------------------------
\clearpage
\begin{sidewaysfigure}[h]
    \centering
    \begin{subfigure}{\linewidth}
        \includegraphics[width=\linewidth]{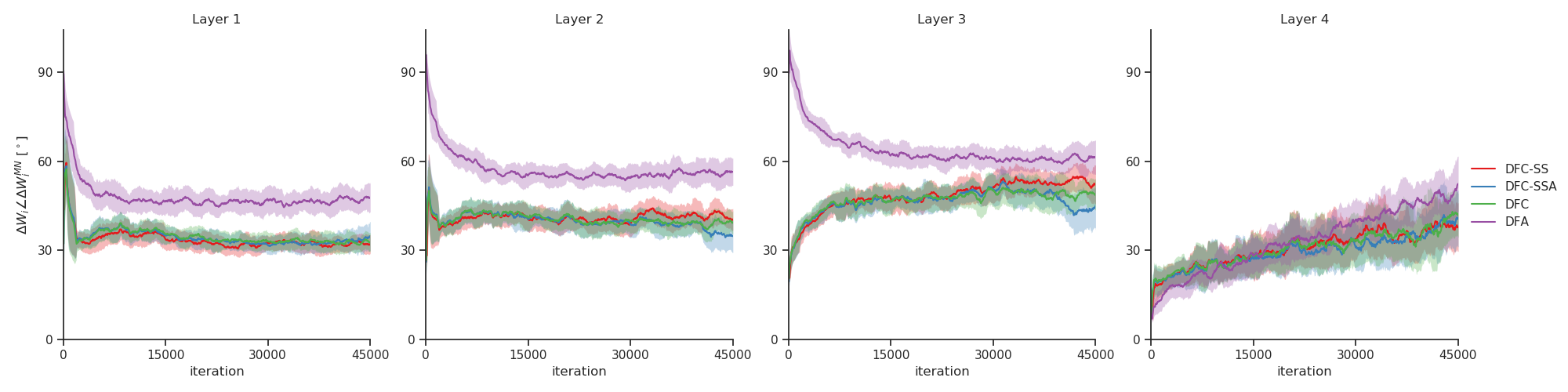}
        \caption{Angles between $\Delta W_i$ and the damped MN updates for Fashion-MNIST.}
    \end{subfigure}
     \begin{subfigure}{\linewidth}
        \includegraphics[width=\linewidth]{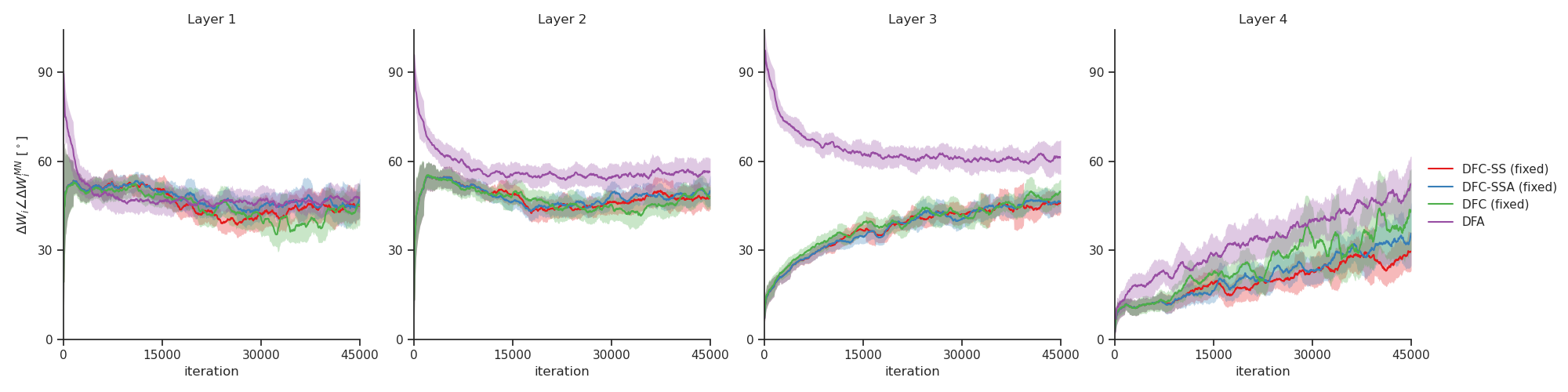}
        \caption{Angles between $\Delta W_i$ and the damped MN updates for Fashion-MNIST with fixed feedback weights.}
    \end{subfigure}
    \caption{Angles between the damped MN updates with $\gamma=1$ \eqref{eq_app:MN_updates} and the ones computed by DFC-SS, DFC-SSA, DFC, and DFA, plotted for all hidden layers with (a) learned feedback weights and (b) fixed feedback weights on Fashion-MNIST. A window-average is plotted together with the window-std (shade). The x-axis iterations corresponds to the minibatches processed. The curves of DFC, DFC-SS and DFC-SSA overlap in some of the plots.}
    \label{fig:fashion_mnist_gnt_angles}
\end{sidewaysfigure}

\clearpage
\begin{sidewaysfigure}[h]
    \centering
    \begin{subfigure}{\linewidth}
        \includegraphics[width=\linewidth]{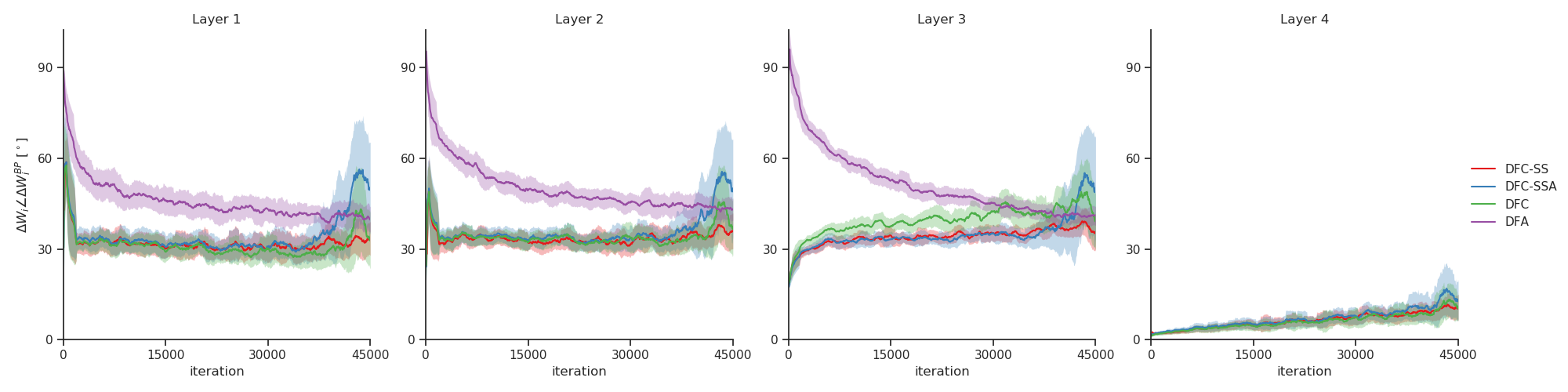}
        \caption{Angles between $\Delta W_i$ and the BP updates for Fashion-MNIST.}
    \end{subfigure}
     \begin{subfigure}{\linewidth}
        \includegraphics[width=\linewidth]{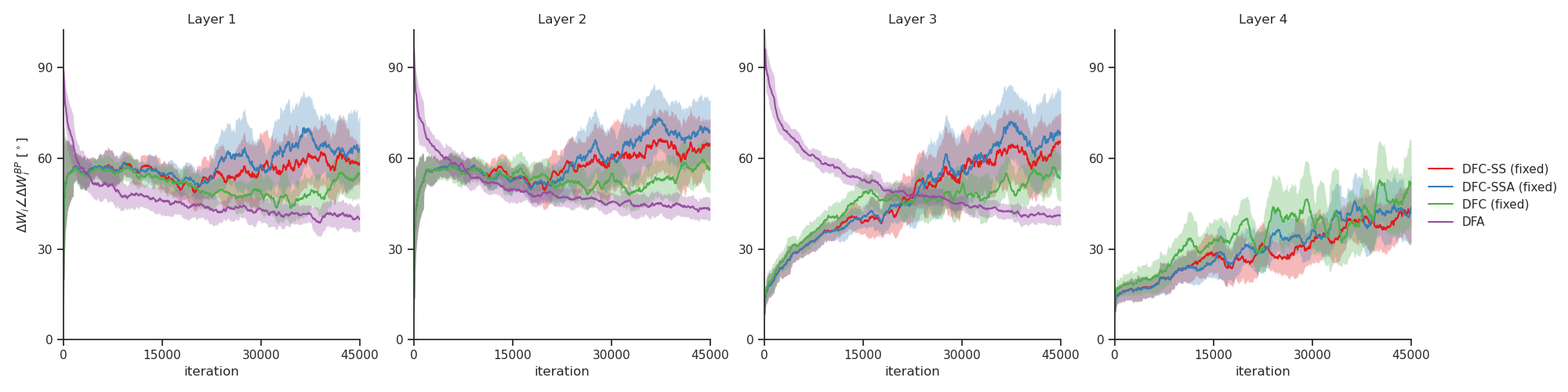}
        \caption{Angles between $\Delta W_i$ and the BP updates for Fashion-MNIST with fixed feedback weights.}
    \end{subfigure}
    \caption{Angles between the weight updates $\Delta W_i$ computed by BP and the ones computed by DFC-SS, DFC-SSA, DFC, and DFA, plotted for all hidden layers with (a) learned feedback weights (b) fixed feedback weights on Fashion-MNIST. A window-average is plotted together with the window-std (shade). The x-axis iterations corresponds to the minibatches processed. The curves of DFC, DFC-SS and DFC-SSA overlap in some of the plots.}
    \label{fig:fashion_mnist_bp_angles}
\end{sidewaysfigure}

% MNIST-autoencoder----------------------------------------
\clearpage
\begin{sidewaysfigure}[h]
    \centering
    \begin{subfigure}{\linewidth}
        \includegraphics[width=\linewidth]{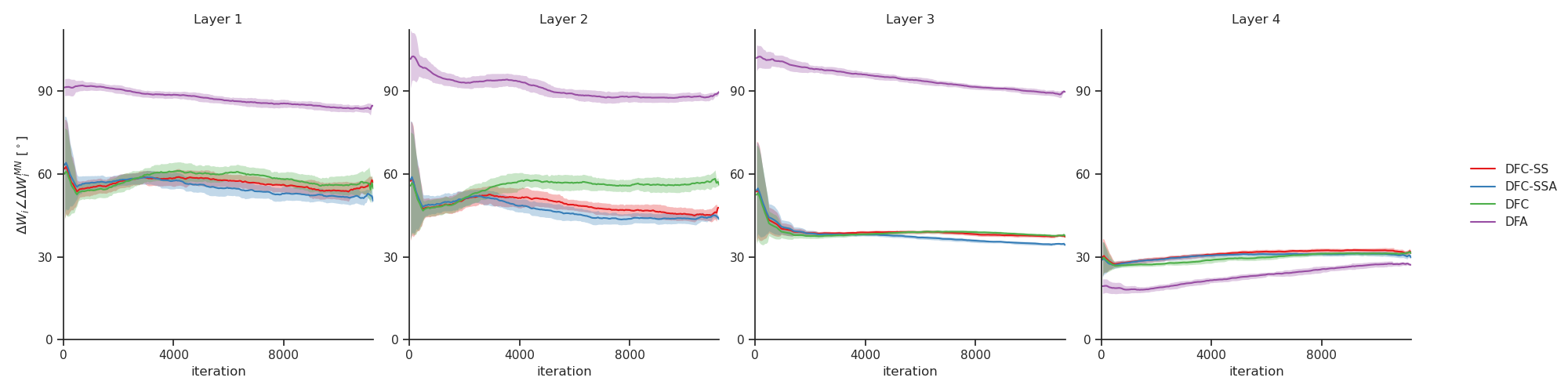}
        \caption{Angles between $\Delta W_i$ and the damped MN updates for MNIST-autoencoder.}
    \end{subfigure}
     \begin{subfigure}{\linewidth}
        \includegraphics[width=\linewidth]{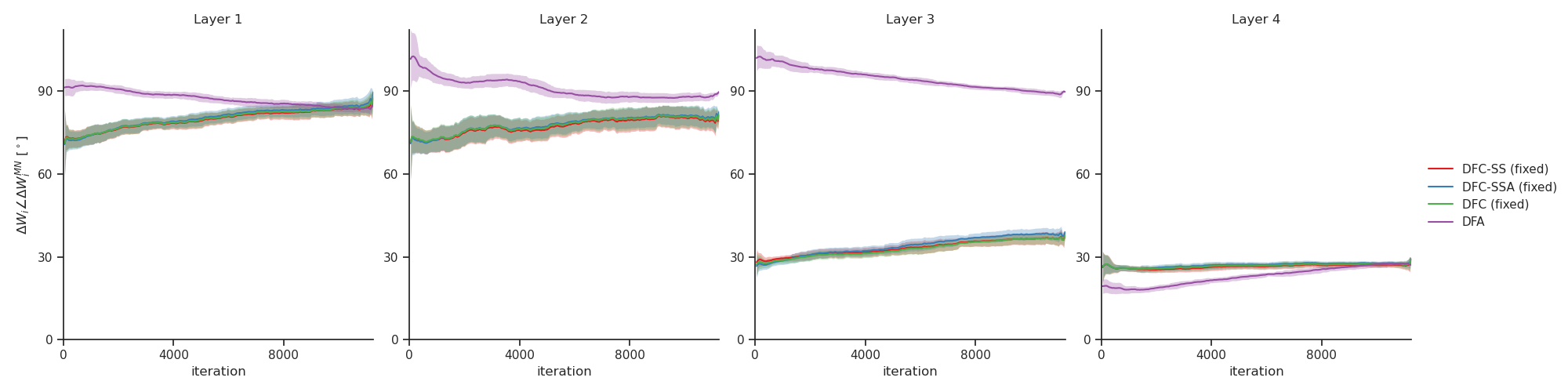}
        \caption{Angles between $\Delta W_i$ and the damped MN updates for MNIST-autoencoder with fixed feedback weights.}
    \end{subfigure}
    \caption{Angles between the damped MN updates with $\gamma=1$ \eqref{eq_app:MN_updates} and the ones computed by DFC-SS, DFC-SSA, DFC, and DFA, plotted for all hidden layers with (a) learned feedback weights and (b) fixed feedback weights on MNIST-autoencoder. A window-average is plotted together with the window-std (shade). The x-axis iterations corresponds to the minibatches processed. The curves of DFC, DFC-SS and DFC-SSA can overlap in some of the plots.}
    \label{fig:mnist_auto_gnt_angles}
\end{sidewaysfigure}

\clearpage
\begin{sidewaysfigure}[h]
    \centering
    \begin{subfigure}{\linewidth}
        \includegraphics[width=\linewidth]{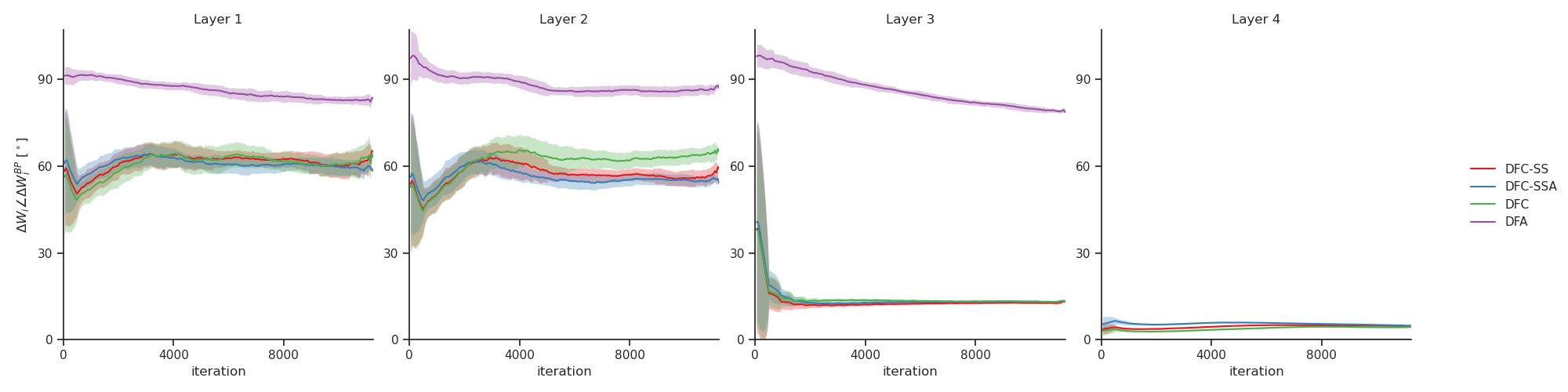}
        \caption{Angles between $\Delta W_i$ and the BP updates for MNIST-autoencoder.}
    \end{subfigure}
     \begin{subfigure}{\linewidth}
        \includegraphics[width=\linewidth]{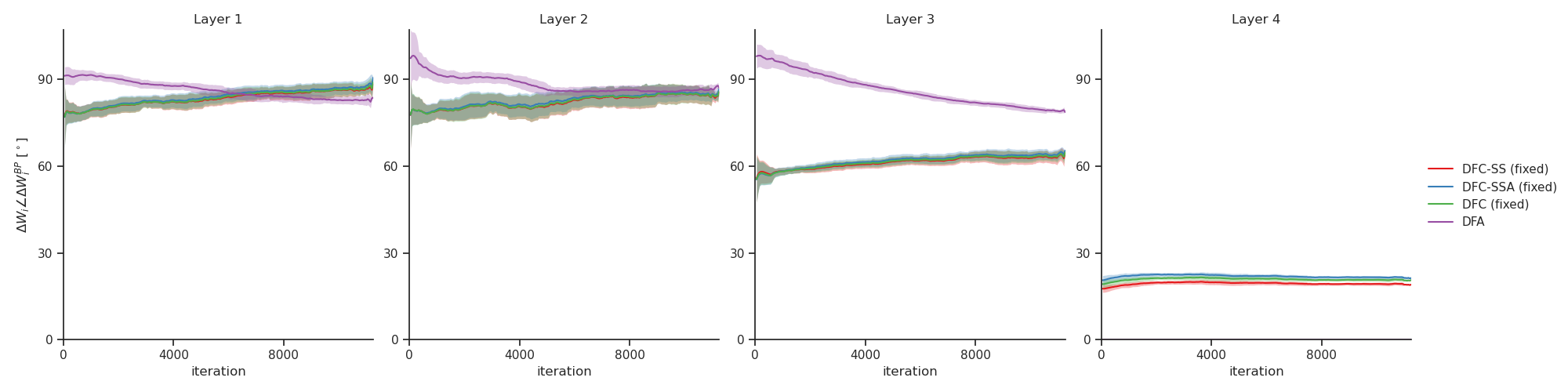}
        \caption{Angles between $\Delta W_i$ and the BP updates for MNIST-autoencoder with fixed feedback weights.}
    \end{subfigure}
    \caption{Angles between the weight updates $\Delta W_i$ computed by BP and the ones computed by DFC-SS, DFC-SSA, DFC, and DFA, plotted for all hidden layers with (a) learned feedback weights (b) fixed feedback weights on MNIST-autoencoder. A window-average is plotted together with the window-std (shade). The x-axis iterations corresponds to the minibatches processed. The curves of DFC, DFC-SS and DFC-SSA can overlap in some of the plots.}
    \label{fig:mnist_auto_bp_angles}
\end{sidewaysfigure}

\clearpage

\subsubsection{Autoencoder images}
Fig. \ref{fig:autoencoding_im} shows the autoencoder output for randomly selected samples of BP, DFC-SSA, DFC-SSA (fixed), and DFA, compared with the autoencoder input. As DFC, DFC-SS, and DFC-SSA have very similar test losses and hence autoencoder performance, we only show the plots for DFC-SSA and DFC-SSA (fixed). Fig. \ref{fig:autoencoding_im} shows that BP and the DFC variants with trained weights have almost perfect autoencoding performance when visually inspected, while DFA and the DFC (fixed) variants do not succeed in autoencoding their inputs, which is also reflected in the performance results (see Table \ref{tab:test_results}.

\begin{figure}[h]
    \centering
    \includegraphics[width=0.5\linewidth]{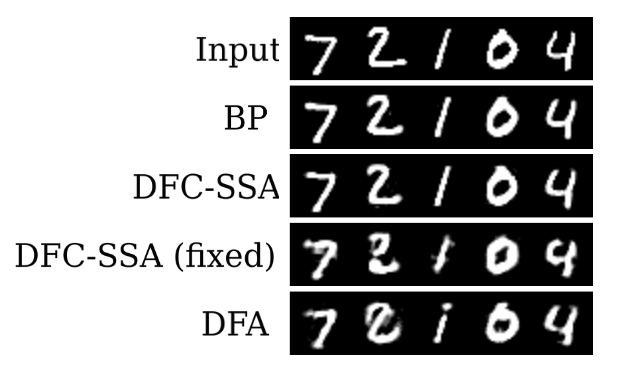}
    \caption{Visual representation of the autoencoder outputs for BP, DFC-SSA, DFC-SSA (fixed), and DFA, compared to the autoencoder input.}
    \label{fig:autoencoding_im}
\end{figure}

\subsection{Resources and compute}\label{app:resources_compute}
For the computer vision experiments, we used GeForce RTX 2080 and GeForce RTX 3090 GPUs. Table \ref{tab:runtimes} provides runtime estimates for 1 epoch of feedforward training and 3 epochs of feedback training (if applicable) for the DFC methods, using a GeForce RTX 2080 GPU. For MNIST and Fashion-MNIST we do 100 training epochs and for MNIST-autoencoder 25 training epochs. We did hyperparameter searches of 200 samples on all datasets for DFC-SSA and DFC-SSA (fixed) and reused the hyperparameter configuration for the other DFC variants. For BP and DFA we also performed hyperparameter searches of 200 samples for all experiments, with computational costs negligible compared to DFC.

\begin{table}[h]
\centering
\caption{Estimated run times in seconds per epoch (both feedforward and feedback training included) on GeForce RTX 2080 GPU for the experiments of Table \ref{tab:test_results}}
\label{tab:runtimes}
\begin{tabular}{*3c}
\toprule
{}   & MNIST \& Fashion-MNIST & MNIST-autoencoder \\
\midrule
DFC   &  500s & 1500s \\
    DFC-SSA  &  130s & 450s \\
DFC-SS  &  500s & 1500s\\
DFC (fixed)   &  370s & 1350s \\
DFC-SSA (fixed)   &  4s & 300s\\
DFC-SS (fixed) 
&  370s & 1350s \\
% MN   &  $2.11^{\pm 0.046}\%$ & $11.39^{\pm 0.39}\%$ & - & $12.01^{\pm -} \cdot 10^{-2}$ \\

\bottomrule
\end{tabular}
\end{table}

\subsection{Dataset and Code licenses}\label{app:licenses}
For the computer vision experiments, we used the MNIST dataset \citep{lecun1998mnist} and the Fashion-MNIST dataset \citep{xiao2017fashion}, which have the following licenses:
\begin{itemize}
    \item MNIST: \url{https://creativecommons.org/licenses/by-sa/3.0/}
    \item Fashion-MNIST:\url{https://opensource.org/licenses/MIT}
\end{itemize}

For the implementation of the methods, we used PyTorch \citep{pytorch} and built upon the codebase of \citet{meulemans2020theoretical}, which have the following licenses:
\begin{itemize}
    \item Pytorch: \url{https://github.com/pytorch/pytorch/blob/master/LICENSE}
    \item \citet{meulemans2020theoretical}: \url{https://www.apache.org/licenses/LICENSE-2.0}
\end{itemize}

\section{DFC and multi-compartment models of cortical pyramidal neurons}
As mentioned in the Discussion, the multi-compartment neuron of DFC (see Fig. \ref{fig:DFC_schematics}C) is closely related to recent dendritic compartment models of the cortical pyramidal neuron \citep{sacramento2018dendritic, guerguiev2017towards, payeur2020burst, richards2019dendritic}. In the terminology of these models, our central, feedforward, and feedback compartments, correspond to the somatic, basal dendritic, and apical dendritic compartments of pyramidal neurons. Here, we relate our network dynamics \eqref{eq:network_dynamics} in more detail to the proposed pyramidal neuron dynamics of \citet{sacramento2018dendritic}. Rephrasing their dynamics for the somatic membrane potentials of pyramidal neurons (equation (1) of \citet{sacramento2018dendritic}) with our own notation, we get
\begin{align}\label{eq_app:pyramidal_dynamics}
    \tau_v \ddt \ve{v}_i(t) = -g_{\mathrm{lk}} \ve{v}_i(t) + g_{\mathrm{B}}\big(\vff_i(t) - \vv_i(t)\big) + g_{\mathrm{A}}\big(\vfb_i(t) - \vv_i(t)\big) + \sigma \ves{\xi}_i(t).
\end{align}
Like DFC, the network is structured in multiple layers, $0\leq i \leq L$, where each layer has its own dynamical equation as defined above. Basal and apical dendritic compartments ($\vff_i$ and $\vfb_i$ resp.) of pyramidal cells are coupled towards the somatic compartment ($\ve{v}_i$) with fixed conductances $g_{\mathrm{B}}$ and $g_{\mathrm{A}}$, and leakage $g_{\text{lk}}$. Background activity of all compartments is modeled by an independent white noise input $\ves{\xi}_i \sim \mathcal{N}(0,I)$. The dendritic compartment potentials are given in their instantaneous forms (c.f. equations (3) and (4) in \citet{sacramento2018dendritic})
%\footnote{For clarity, we used the notation of the Langevin equation for the white noise input. See Appendix \textbf{TODO} for a more formal notation corresponding to Itô calculus.}
\begin{align}
    \vff_i(t) &= W_i\phi(\ve{v}_{i-1}(t)) \label{eq:model_equation_B}\\
    \vfb_i(t) &= Q_i\ve{u}(t) \label{eq:model_equation_A}
\end{align}
with $W_i$ the synaptic weights of the basal dendrites, $Q_i$ the synaptic weights of the apical dendrites, $\phi$ a nonlinear activation function transforming the voltage levels to firing rates, and $\ve{u}$ a feedback input. 

Filling the instantaneous forms of $\vff$ and $\vfb$ into the dynamics of the somatic compartment \eqref{eq_app:pyramidal_dynamics}, and reworking the equation, we get:
\begin{align}
    \tilde{\tau}_v \ddt \vv_i(t) = -\vv_i + \tilde{g}_{\mathrm{B}}  W_i\phi(\ve{v}_{i-1}(t)) + \tilde{g}_{\mathrm{A}} Q_i\ve{u}(t) + \tilde{\sigma} \ves{\xi}_i(t),
\end{align}
with $\tilde{g}_{\mathrm{B}} = \frac{g_{\mathrm{B}}}{g_{\mathrm{lk}}+ g_{\mathrm{B}}+ g_{\mathrm{A}}}$, $\tilde{g}_{\mathrm{A}} = \frac{g_{\mathrm{A}}}{g_{\mathrm{lk}}+ g_{\mathrm{B}}+ g_{\mathrm{A}}}$, $\tilde{\sigma} = \frac{\sigma}{g_{\mathrm{lk}}+ g_{\mathrm{B}}+ g_{\mathrm{A}}}$ and $\tilde{\tau}_v = \frac{\tau_v}{g_{\mathrm{lk}}+ g_{\mathrm{B}}+ g_{\mathrm{A}}}$. When we absorb $\tilde{g}_{\mathrm{B}}$ and $\tilde{g}_{\mathrm{A}}$ into $W_i$ and $Q_i$, respectively, we recover the DFC network dynamics \eqref{eq:network_dynamics} with noise added. Hence, we see that not only the multi-compartment neuron model of DFC is closely related to dendritic compartment models of pyramidal neurons, but also the neuron dynamics used in DFC are intimately connected to models of cortical pyramidal neurons. What sets DFC apart from the cortical model of \citet{sacramento2018dendritic} is its unique feedback dynamics that make use of a feedback controller and lead to approximate GN optimization.

\section{Feedback pathway designs compatible with DFC} \label{app:fb_pathways}
To present DFC in its most simple form, we used direct linear feedback mappings from the output controller towards all hidden layers. However, DFC is also compatible with more general feedback pathways. 

Consider $\vfb_i = g_i(\ve{u})$ with $g_i$ a smooth mapping from the control signal $\ve{u}$ towards the feedback compartment of layer $i$, leading to the following network dynamics:
\begin{align}
    \tau_v \ddt \ve{v}_i(t) &= -\ve{v}_i(t) + W_i\phi\big(\ve{v}_{i-1}(t)\big) + g_i\big(\ve{u}(t)\big) \quad 1\leq i \leq L.
\end{align}
The feedback path $g_i$ could be for example a multilayer neural network (see Fig. \ref{fig:feedback_paths}A) and different $g_i$ could share layers (see Fig. \ref{fig:feedback_paths}B). As the output stepsize $\lambda$ is taken small in DFC, the control signal $\ve{u}$ will also remain small. Hence, we can take a first-order Taylor approximation of $g_i$ around $\ve{u}=0$:
\begin{align}
    g_i(\ve{u}) = J_{g_i}\ve{u} + \mathcal{O}(\lambda^2),
\end{align}
with $J_{g_i} = \frac{\partial g_i(\ve{u})}{\partial \ve{u}}\big\rvert_{\ve{u}=0}$. With this linear approximation and replacing $Q_i$ by $J_{g_i}$, all previous theoretical results from Section \ref{sec:learning_TPDI} hold, as they consider the limit of $\lambda \rightarrow 0$. Furthermore, the local stability results of Section \ref{sec:stability_TPDI} can be recovered by replacing $Q_i$ in Condition \ref{con:local_stability} with $J_{g_i}$ evaluated at $\ve{u}= \ve{u}_{\sss}$. Finally, the feedback learning results (Section \ref{sec:fb_learning}) can be extended to this setting, by learning the synaptic strengths connecting the feedback path $g_i$ to the network layers $\ve{v}_i$ according to the proposed feedback learning rule \eqref{eq:Q_dynamics}. For small $\sigma$, $\ve{u}$ will remain small and hence the feedback learning rule will align $J_{g_i}$, correctly evaluated around $\ve{u}=0$, with $J^T(JJ^T + \gamma I)^{-1}$.

\begin{figure}[h!]
    \centering
    \includegraphics[width=0.7\linewidth]{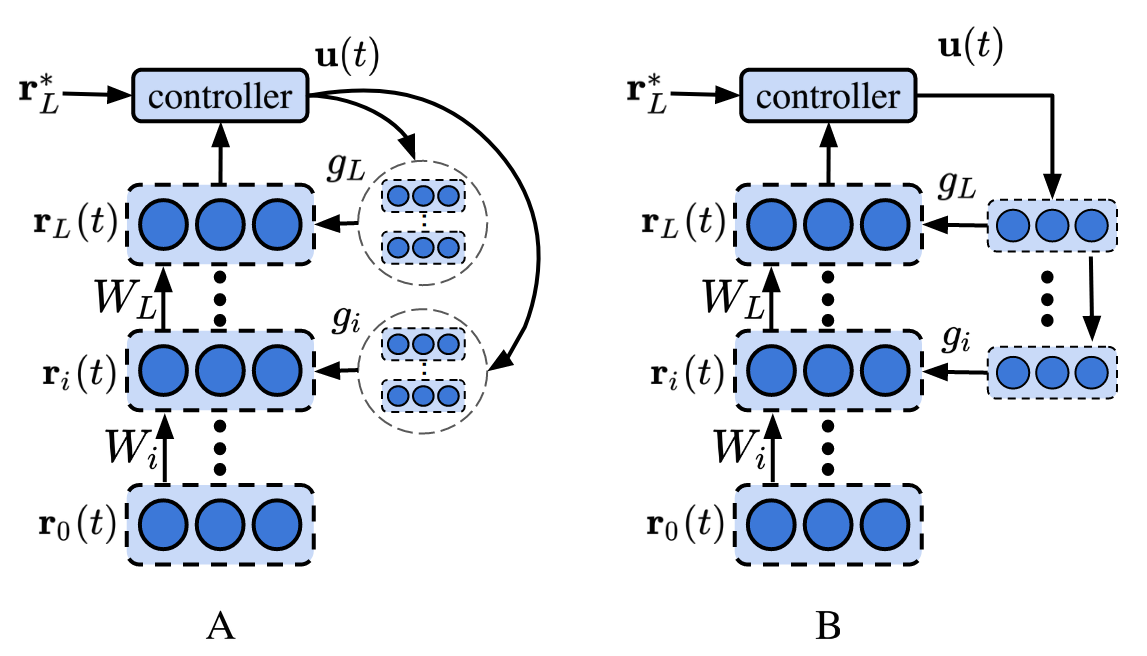}
    \caption{Schematic illustration of more general feedback paths compatible with DFC.}
    \label{fig:feedback_paths}
\end{figure}

Until now, we considered general feedback paths $g_i$ and linearized them around $\ve{u}=0$, thereby reducing their expressive power to linear mappings. As the forward Jacobian $J$ changes for each datasample in nonlinear networks, it can be helpful to have a feedback path for which $J_{g_i}$ also changes for each datasample. Then, each $J_{g_i}$ can specialize its mapping for a particular cluster of datasamples, thereby enabling a better compliance to Conditions \ref{con:Q_GN} and \ref{con:local_stability} for each datasample. To let $J_{g_i}$ change depending on the considered datasample and hence activations $\ve{v}_i$ of the network, the feedback path $g_i$ needs to be `influenced' by the network activations $\ve{v}_i$.

One interesting direction for future work is to have connections from the network layers $\ve{v}_i$ onto the layers of the feedback path $g_i$, that can modulate the nonlinear activation function $\phi_g$ of those layers. By modulating $\phi_g$, the feedback Jacobian $J_{g_i}$ will depend on the network activations $\ve{v}_i$ and, hence, will change for each datasample. Interestingly, there are many candidate mechanisms to implement such modulation in biological cortical neurons \citep{silver2010neuronal, ferguson2020mechanisms, larkum2004top}.

Another possible direction is to integrate the feedback path $g_i$ into the forward network $\eqref{eq:network_dynamics}$ and separate forward signals from feedback signals by using neural multiplexed codes \citep{payeur2020burst, naud2017burst}. As the feedback path $g_i$ is now integrated into the forward pathway, its Jacobian $J_{g_i}$ can be made dependent on the forward activations $\ve{v}_i$. While being a promising direction, merging the forward pathway with the feedback path is not trivial and significant future work would be needed to accomplish it.

\section{Discussion on the biological plausibility of the controller} \label{app:controller_microcircuit}
The feedback controller used by DFC (see Fig. \ref{fig:DFC_schematics}A and eq. \eqref{eq:controller_dynamics}) has three main components. First, it needs to have a way of computing the control error $\ve{e}(t)$. Second, it needs to perform a leaky integration ($\ve{u}^{\mathrm{int}}$) of the control error. Third, the controller needs to multiply the control error by $k_p$. 

Following the majority of biologically plausible learning methods \citep{lillicrap2016random, akrout2019deep, kunin2020two, lansdell2019learning,lee2015difference, meulemans2020theoretical, bengio2020deriving, payeur2020burst, nokland2016direct}, we assume to have access to an output error that the feedback controller can use. As the error is a simple difference between the network output and an output target $\ve{r}_L^*$, this  should be relatively easily computable. Another interesting aspect of computing the output error is the question of where the output target $\ve{r}_L^*$ could originate from in the brain. This is currently an open question in the field \citep{bengio2015towards} which we do not aim to address in this work.

Integrating neural signals over long time horizons is a well-studied subject concerning many application areas, ranging from oculomotor control to maintaining information in working memory \citep{seung1996brain, koulakov2002model, goldman2003robust, goldman2010neural, lim2013balanced}. To provide intuition, a straightforward approach to leaky integration is to use recurrent self-connections with strength $(1-\alpha)$. Then, the same neural dynamics used in \eqref{eq:network_dynamics} give rise to
\begin{align}
    \tau_u \ddt \ve{u}^{\mathrm{int}}(t) = -\ve{u}^{\mathrm{int}}(t) + W_{\mathrm{in}} \ve{e}(t) + (1-\alpha)\ve{u}^{\mathrm{int}}(t) = W_{\mathrm{in}} \ve{e}(t) - \alpha \ve{u}^{\mathrm{int}}(t).
\end{align}
When we take the input weights $W_{\mathrm{in}}$ equal to the identity matrix, we recover the dynamics for $\ve{u}^{\mathrm{int}}(t)$ described in \eqref{eq:controller_dynamics}. 

% We envision to possible ways in which a neural microcircuit could perform a leaky integration of the control error. First, if a physical neuron has a high time constant to integrate its inputs, it could directly be used to perform a leaky integration of the control error, when this error is taken as an input to the neuron. Hence, there would be a leaky integration neuron for each error neuron. Whether there exist neuron types with such properties is an interesting experimental question. Second, the leaky integration could be performed by a population of neurons $\ve{u}^{\mathrm{int}}$ that have the control error $\ve{e}$ as input and have self-connections with strength $(1-\alpha)$. Then, the same neural dynamics used in \eqref{eq:network_dynamics} give rise to
% \begin{align}
%     \tau_u \ddt \ve{u}^{\mathrm{int}}(t) = -\ve{u}^{\mathrm{int}}(t) + W_{\mathrm{in}} \ve{e}(t) + (1-\alpha)\ve{u}^{\mathrm{int}}(t) = W_{\mathrm{in}} \ve{e}(t) - \alpha \ve{u}^{\mathrm{int}}(t).
% \end{align}
% When we take the input weights $W_{\mathrm{in}}$ equal to the identity matrix, we recover the dynamics for $\ve{u}^{\mathrm{int}}(t)$ described in \eqref{eq:controller_dynamics}. 

Finally, a multiplication of the control error by $k_p$ can simply be done by having synaptic weights with strength $k_p$.

\end{document}